\documentclass[dvipsnames]{article}

\usepackage[accepted]{icml2022}
\usepackage{times}

\usepackage{amsmath,amsfonts,bm}

\def\1{\bm{1}}

\DeclareMathAlphabet{\mathsfit}{\encodingdefault}{\sfdefault}{m}{sl}
\SetMathAlphabet{\mathsfit}{bold}{\encodingdefault}{\sfdefault}{bx}{n}

\usepackage[utf8]{inputenc} 
\usepackage[T1]{fontenc}    
\usepackage{url}            
\usepackage{booktabs}       
\usepackage{amsfonts}       
\usepackage{nicefrac}       
\usepackage{microtype}      
\usepackage{scalefnt}

\usepackage{graphicx}
\usepackage{subfigure}
\usepackage{makecell,multirow} %

\usepackage{breakurl}
\usepackage[breaklinks]{hyperref}
\usepackage{amsmath,amssymb}

\usepackage{mwe}

\usepackage{algorithm}
\usepackage{algorithmicx}
\usepackage[noend]{algpseudocode}
\renewcommand{\algorithmicrequire}{\textbf{Input:}}
\renewcommand{\algorithmicensure}{\textbf{Output:}}
\usepackage{enumitem}

\usepackage[flushleft]{threeparttable} 
\usepackage{multirow}
\usepackage{colortbl}

\usepackage{xcolor}
\usepackage{pifont}
\newcommand{\cmark}{{\color{PineGreen}\ding{51}}}%
\newcommand{\xmark}{{\color{BrickRed}\ding{55}}}%

\definecolor{niceblue}{rgb}{0.0,0.19,0.56}

\usepackage{hyperref}
\hypersetup{colorlinks,linkcolor={blue},citecolor={niceblue},urlcolor={blue}}

\usepackage{xspace}
\newcommand{\algname}[1]{{\sc #1}\xspace}

\newcommand{\R}{\mathbb{R}}

\def\<#1,#2>{\left\langle #1,#2\right\rangle}

\usepackage{amsthm}
\usepackage{thmtools}

\usepackage{bbm}
\newcommand{\indicator}{\mathbbm{1}}

\newtheorem{lemma}{Lemma}[section]
\newtheorem{theorem}{Theorem}[section]
\newtheorem{definition}{Definition}[section]

\newtheorem{assumption}{Assumption}[section]

\newcommand{\argmin}{\mathop{\arg\!\min}}

\newcommand{\cB}{{\cal B}}
\newcommand{\cC}{{\cal C}}
\newcommand{\cD}{{\cal D}}

\newcommand{\cG}{{\cal G}}

\newcommand{\cO}{{\cal O}}

\newcommand{\cU}{{\cal U}}

\newcommand{\EE}{\mathbb{E}}
\newcommand{\PP}{\mathbb{P}}

\makeatletter
\newcommand*{\inlineequation}[2][]{%
  \begingroup
    \refstepcounter{equation}%
    \ifx\\#1\\%
    \else
      \label{#1}%
    \fi
    \relpenalty=10000 %
    \binoppenalty=10000 %
    \ensuremath{%
      #2%
    }%
    ~\@eqnnum
  \endgroup
}
\makeatother

\newcommand{\newstuff}[1]{#1}

\icmltitlerunning{Secure Distributed Training at Scale}

\author{%
    Anonymous
}

\begin{document}

\twocolumn[
\icmltitle{Secure Distributed Training at Scale}



\icmlsetsymbol{equal}{*}

\begin{icmlauthorlist}
\icmlauthor{Eduard Gorbunov}{equal,mipt,mila,yandex}
\icmlauthor{Alexander Borzunov}{equal,hse,yandex}
\icmlauthor{Michael Diskin}{hse,yandex}
\icmlauthor{Max Ryabinin}{hse,yandex}
\end{icmlauthorlist}

\icmlaffiliation{mipt}{MIPT}
\icmlaffiliation{mila}{Mila -- Quebec AI Institute}
\icmlaffiliation{hse}{HSE University}
\icmlaffiliation{yandex}{Yandex}

\icmlcorrespondingauthor{Eduard Gorbunov}{eduard.gorbunov@phystech.edu}
\icmlcorrespondingauthor{Alexander Borzunov}{borzunov.alexander@gmail.com}

\icmlkeywords{Machine Learning, ICML}

\vskip 0.3in
]

\printAffiliationsAndNotice{\icmlEqualContribution} 

\begin{abstract}
Many areas of deep learning benefit from using increasingly larger neural networks trained on public data, as is the case for pre-trained models for NLP and computer vision. Training such models requires a lot of computational resources (e.g., HPC clusters) that are not available to small research groups and independent researchers. One way to address it is for several smaller groups to pool their computational resources together and train a model that benefits all participants. Unfortunately, in this case, any participant can jeopardize the entire training run by sending incorrect updates, deliberately or by mistake. Training in presence of such peers requires specialized distributed training algorithms with Byzantine tolerance. These algorithms often sacrifice efficiency by introducing redundant communication or passing all updates through a trusted server, making it infeasible to apply them to large-scale deep learning, where models can have billions of parameters. In this work, we propose a novel protocol for secure (Byzantine-tolerant) decentralized training that emphasizes communication efficiency.

\end{abstract}
\section{Introduction}\label{sect:intro}

Many hard scientific problems were solved through collaboration between many nations, groups and individuals. This is especially evident in natural sciences, where researchers formed multinational collaborations to run large-scale experiments and share compute infrastructure~\citep{Collaboration2012ObservationOA,ISS,Collaboration2016ObservationOG}. Projects like Folding@home~\citep{foldingathome} and BOINC~\citep{anderson2004boinc} push this trend even further by recruiting volunteers that donate their compute to collectively run computational experiments at an unprecedented scale~\citep{folding_exaflop_2}.

Recently, similar techniques were proposed for deep learning. They aim to solve the challenges caused by the sheer computational complexity of many machine learning tasks, such as pretraining transformers for NLP~\citep{bert,gpt3,roberta} or learning on huge datasets in vision~\citep{jft300m,Kolesnikov2020BigT,vissl}. Recent works~\citep{volunteer_dl_async,hivemind_dmoe,atre2021distributed,dedloc} propose several systems sharing the computation across many volunteers that donate the idle time of their computers \newstuff{to train large models on public datasets.}



Despite their strengths, volunteer computing systems have so far seen limited practical applications~\citep{volunteer_dl_async}. A major roadblock towards the global adoption of these techniques is trust in reliability of each participant. For distributed training, all progress made by the collaboration can be undermined if a single peer sends incorrect outputs due to an error in computation~\citep{setihomeamdfail} or malicious intent~\citep{tolpegin2020data}.

Prior art in decentralized optimization proposed several optimization algorithms that are resistant to such ``Byzantine'' faults. However, most Byzantine-tolerant training protocols require either passing all updates through a trusted central server or exchanging additional messages that increase the network load by several times~\citep{chen2018draco,rajput2019detox}. This is a major problem for large-scale distributed deep learning, where hundreds of peers must exchange updates for millions of parameters at regular intervals~\citep{pytorch_distributed,horovod,shoeybi2019megatron}.
Thus, in many practical scenarios, the computation and communication overhead of Byzantine-tolerant algorithms outweighs the benefits of collaborating with others.

In this work, we set out to solve this problem by proposing a novel Byzantine-tolerant distributed training protocol designed for large-scale deep learning workloads.
Our approach combines the scalability and communication efficiency of modern distributed training techniques such as All-Reduce SGD~\citep{horovod} with resilience against Byzantine and Sybil attackers. To achieve this, we leverage cryptographic techniques to verify the integrity of training with minimal overhead that does not depend on the model size. Our protocol does not require any specific peers to be trusted. 

\pagebreak
Our contributions can be summarized as follows:
\begin{itemize}[leftmargin=*]
    \vspace{-0.5em}
    \vspace{-1px}\item We propose a novel protocol for decentralized Byzantine-tolerant training \newstuff{on data available to all participants,} where the extra communication cost does not depend on the number of parameters.
    \vspace{-1px}\item We rigorously analyze this protocol and prove convergence bounds for convex and non-convex losses with Byzantine attackers. Furthermore, we derive accelerated convergence rates for the same task under realistic assumptions about model gradients.
    \vspace{-1px}\item We propose a heuristic for resisting Sybil attacks from computationally constrained attackers, allowing to accept new untrusted peers joining midway through training.
    \vspace{-1px}\item We verify the effectiveness of our algorithm in controlled experiments\footnote{Source code for the experiments is available at \href{https://github.com/yandex-research/btard}{\texttt{https://github.com/yandex-research/btard}}} and actual large-scale training runs. Specifically, we start with ResNet-18 for CIFAR-10 classification and follow up with pretraining ALBERT-large in a setup where almost a half of all peers are malicious.
\end{itemize}

\section{Related work}\label{sect:related}

\subsection{Distributed deep learning}\label{sect:related_distributed}





Training modern neural networks often requires the amount of computation that is infeasible to achieve on any single machine. 
One has to train such models on multiple machines using methods for distributed training. Most of these methods fall into two groups: in \textit{data-parallel} training, each worker trains the entire model by sampling batches from the training data~\citep{horovod,goyal2017accurate}; in contrast, \textit{model-parallel} training allocates parts of the model on different workers~\citep{gpipe,pipedream,shoeybi2019megatron}. 
In this study, we consider only the first group; notably, most model-parallel systems still rely on data parallelism between nodes at the same stage~\citep{zero,megatron2}.


Usually, data-parallel training consists of two phases: first, each worker computes the gradients over its data; then, all workers aggregate the gradients and run an \algname{SGD} step.
The simplest aggregation strategy is known as Parameter Servers (PS)~\citep{parameter_server_first,sharded_ps_first,recht2011hogwild}:
one of the servers stores and updates the model parameters, while all others iteratively compute the gradients, send them to the PS, and download the updated parameters.
This strategy can be quite efficient with a small number of workers; as it increases, the parameter server eventually becomes unable to handle the load.
While gradient compression~\citep{seide20141,deepgradientcompression, mishchenko2019distributed,koloskova2020decentralized,gorbunov2020linearly,gorbunov2021marina} and local updates~\citep{localsgd_first} partially address this issue, it still remains a bottleneck of such methods.

\begin{figure}[tb]
    \vskip 0.05in
    \centering
    \includegraphics[width=1.0\linewidth]{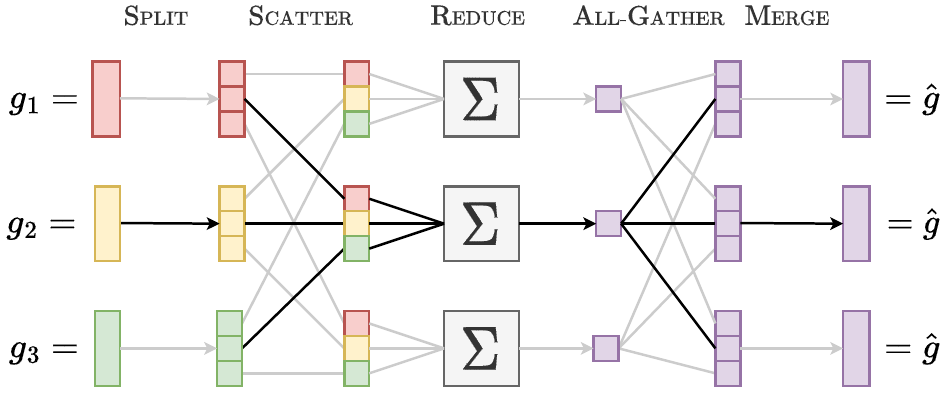}
    \vskip -0.05in
    \caption{A scheme of Butterfly All-Reduce~\citep{butterfly_arsgd}. Each peer transfers only $O(d)$ data when averaging a vector of size $d$.}
    \label{fig:butterfly}
    \vskip -0.15in
\end{figure}

In practice, most distributed training systems leverage All-Reduce (AR)~\citep{goyal2017accurate,mikami2019massively,lamb} --- a family of collective communication protocols that allow servers to average their data and receive the result on each machine. 
The resulting method, named All-Reduce SGD (\algname{AR-SGD}), runs AR on local gradients of each peer to compute the global average.
Usually, AR-SGD uses bandwidth-optimal versions of All-Reduce~\citep{horovod,bandwidth_optimal_allreduce}, such as Butterfly All-Reduce (see Figure~\ref{fig:butterfly}). Depending on the exact algorithm, they require \textit{each} peer to transfer only $O(d)$ or $O(d\log n)$ data when averaging a vector of size $d$ across $n$ peers (unlike the PS-based approaches, where PS transfers $O(d n)$ data).




\subsection{Byzantine-tolerant optimization}\label{sect:related_byzantine}

Standard distributed training methods are not robust against Byzantine attacks. In the vanilla parallel SGD, one malicious worker can break the convergence of the whole method by shifting the mean of the resulting vector in an arbitrary way. Therefore, the research community invented special algorithms that can train models even in this setup.

\textbf{Parameter-server (PS) based approaches.} Most of the algorithms designed to be Byzantine-resilient rely on the existence of a trusted parameter server. In such approaches, the standard mean estimator, e.g., the one used in parallel SGD, is typically replaced with a more robust aggregation rule \citep{blanchard2017machine, yin2018byzantine, damaskinos2019aggregathor, mhamdi2018hidden, pillutla2019robust}. However, recent works show that it is not enough by proposing special types of Byzantine attacks~\citep{baruch2019little,xie2020fall} and showing that permutation-invariant algorithms cannot converge to any predefined accuracy of the solution \citep{karimireddy2020learning}.

Although there are several approaches aiming to circumvent this issue, most of them have significant limitations such as no convergence analysis \citep{chen2018draco,rajput2019detox,rodriguez2020dynamic, xu2020towards}, too restrictive assumptions in the analysis \citep{alistarh2018byzantine, allen2020byzantine, regatti2020bygars}, or the usage of variance-reduced estimators \citep{wu2020federated}, which are known to converge slowly in deep learning applications~\citep{defazio2018ineffectiveness}. The only paper without such limitations is \citet{karimireddy2020learning}: it proposes a new aggregation rule called \algname{CenteredClip}, applies it to SGD with client momentum, and proves convergence results for the obtained method in the non-convex case under reasonable assumptions. We provide more details on Byzantine-tolerant PS-based approaches in Appendix~\ref{appendix:extra_related_ps}.

\textbf{Decentralized approaches} for Byzantine-tolerant optimization are studied only in a few papers. Unfortunately, the known approaches are not well-suited for distributed deep learning since they either rely on full gradient computations \citep{yang2019bridge,yang2019byrdie}, \newstuff{or use redundant communications with multiple servers \citep{el2020genuinely},} or require peer-to-peer communication of full vectors at each step \citep{gupta2021byzantine, gupta2021byzantine2}, which is not scalable, or provide the convergence guarantees that are inferior to non-parallel \algname{SGD} \citep{peng2021byzantine}, which has prohibitively slow convergence on modern deep learning tasks. We defer further details to Appendix~\ref{appendix:extra_related_decentralized}.

\subsection{Security in distributed systems}\label{sect:related_mpc}

\textbf{Message propagation protocols.} In this work, we consider distributed systems relying exclusively on peer-to-peer connections (e.g., the ones working over the Internet). Several key stages of our algorithm require peers to \textit{broadcast} small messages to all other peers. For the sake of consistency, if at least one honest peer receives a message, we expect all other honest peers to eventually receive it as well.

A naive solution would be for all peers to relay each previously unseen message to all other peers. In this case, for $n$ peers and a $b$-bit message, one all-to-all broadcast would require \textit{each} peer to transfer $O(n^2 b)$ data. To improve efficiency, we use GossipSub~\citep{vyzovitis2020gossipsub} that reduces this cost to $O(n b)$ data per peer by relaying each message to only $D$ carefully chosen neighbors, where $D$ is a constant chosen based on latency requirements.

\textbf{Digital signatures.} Our approach relies on the fact that an attacker cannot impersonate an honest peer or change messages an honest peer broadcasts. To achieve that, we require all peers to declare their public keys and sign all their messages with digital signatures~\citep{rivest1978method}.


\textbf{Multi-party random number generator.} To ensure that peers compute gradients honestly, our approach verifies a random subset of all computed gradients. Thus, we need to choose who is going to be checked in such a way that the attackers can neither predict nor influence the random draw. This can be done with a multi-party random number generator (MPRNG) based on the coin tossing protocol from \citet{blum1983coin}. We explain the full algorithm in Appendix~\ref{appendix:mprng}. The algorithm requires each peer to only broadcast 3 scalars, so its communication cost is $O(n)$ data per peer.

\section{Method}\label{sect:method}

We consider secure distributed training on public datasets, where each peer can access the entire training data \newstuff{and communicate with any other peer}. In this scenario, multiple parties cooperate by combining their computational resources for a single large-scale training run. Specifically, we consider a data-parallel training setup with All-Reduce SGD (as described in Section~\ref{sect:related_distributed}), where peers aggregate their gradient vectors of size $d$.

We describe our strategy in several stages:
\begin{itemize}[leftmargin=*]
 \vspace{-0.5em}
 \vspace{-1px}\item Section~\ref{sect:method_aggregation} outlines our approach for \textbf{B}yzantine-\textbf{T}olerant \vspace{-1px}\textbf{A}ll-\textbf{R}e\textbf{d}uce (\algname{BTARD}).
 \vspace{-1px}\item In Section~\ref{sect:method_analysis}, we formulate the underlying optimization problem and derive its convergence bounds.
 \vspace{-1px}\item In Section~\ref{sect:method_reputation}, we propose a heuristic for resisting Sybil attacks, allowing our system to accept new untrusted peers midway through training.
\end{itemize}

\subsection{Byzantine-Tolerant All-Reduce}\label{sect:method_aggregation}

\begin{figure*}[t]
    \vskip 0.05in
    \centering
    \includegraphics[width=0.95\textwidth]{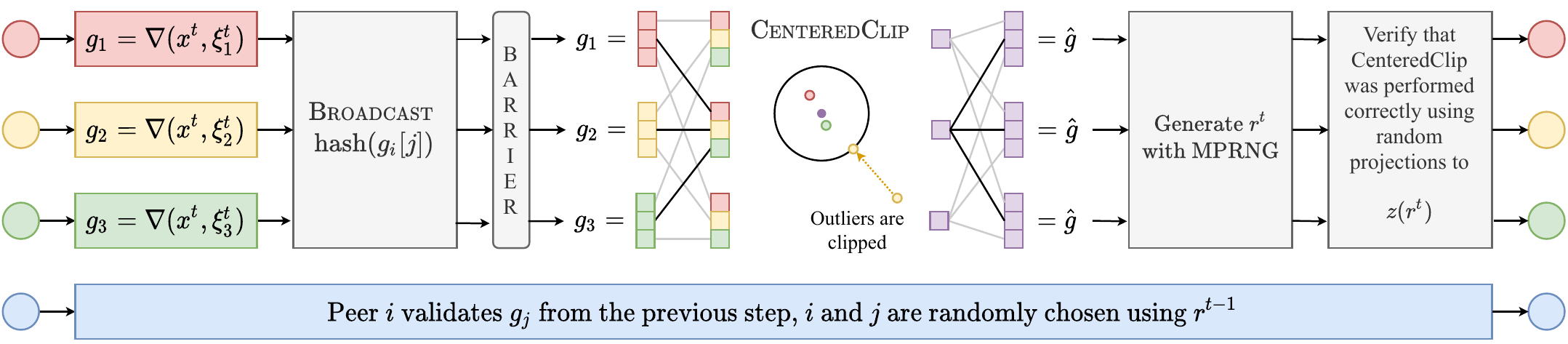}
    \vskip -0.05in
    \caption{A scheme illustrating one step of Byzantine-Tolerant All-Reduce~--- a part of Algorithm~\ref{alg:btarsgd_outline} executed between the consecutive SGD steps. Here, $t$ is the step number, $x^t$ is the model weights, and $\xi^t_i$ is a publicly known random seed for sampling a minibatch.}
    \label{fig:btard_scheme}
    \vskip -0.15in
\end{figure*}

We assume that some workers can be malicious, i.e., they can \newstuff{arbitrarily} deviate from our algorithm: \newstuff{for instance,} send arbitrary vectors instead of stochastic gradients or violate the communication  protocol. Such workers are called \textit{Byzantine nodes} or just \textit{Byzantines}. We assume them to be omniscient \citep{karimireddy2020learning} (except for the honest nodes' private keys and the internals of MPRNG) \newstuff{and able to collude with each other}. We denote the set of all ``good'' workers as $\cG$ and the set of Byzantine workers as $\cB$. We further assume that $\cB$ is fixed throughout the optimization process, and less than a half of the nodes are Byzantine: $|\cB| \le \delta n$, where $\delta \in [0, \nicefrac{1}{2})$. Finally, we assume that all workers have access to the data defining the objective function, sampling minibatches from the full dataset.\footnote{\citet{he2020byzantine} show that it is impossible to achieve any predefined accuracy of the solution without this assumption, i.e., in the heterogeneous case (see discussion in Appendix~\ref{sec:imposssibility_hetero}).}

\newstuff{We design our algorithm in such a way that all types of Byzantine faults have limited effect and chance of being discovered.} 
To limit the damage over a single \algname{SGD} step, we modify Butterfly All-Reduce\footnote{We choose Butterfly All-Reduce so that peers aggregate non-overlapping parts of the gradient vector. This helps to identify the attacker if the gradients are aggregated incorrectly. \citet{byteps} report that Butterfly All-Reduce is near-optimal for distributed training over high-latency networks such as the Internet.}~(see Figure~\ref{fig:butterfly}) with a robust aggregation technique known as \algname{CenteredClip} \citep{karimireddy2020learning}. We apply \algname{CenteredClip} to each partition of the gradient vector instead of naive averaging. We denote this procedure as \algname{ButterflyClip} (see Algorithm~\ref{alg:butterfly_clip} for its formal description).

However, Byzantine peers can circumvent this limit by attacking over many iterations. To protect against this, \algname{BTARD} periodically chooses random peers to serve as \textit{validators}. The validators must recalculate the gradients of other peers and report any discrepancies \newstuff{instead of computing their own gradients.} Since such tests are effective only if the attackers cannot predict whether they will be validated, we use a multi-party random number generator (as described in Section~\ref{sect:related_mpc}) to choose the validated peers.

After each training step, peers use MPRNG to choose \newstuff{$m$ validators and $m$ peers to validate (each validator checks one peer).} The Byzantines cannot predict ``safe'' iterations before they commit to an attack. Thus, more frequent attacks (with greater total damage) are more likely to be detected by an honest validator.

Since validators can also be malicious, \algname{BTARD} uses a special \textsc{Accuse} procedure to detect false reports. Before each averaging round, peers broadcast hashes of their gradients using GossipSub\footnote{We assume that peers declare their public key when joining and sign all broadcasted messages with the respective private key. Any peer broadcasting contradicting messages (e.g., different gradient hashes) should be banned, since it could break the eventual consistency (other peers may receive them in a different order).} (line~2 of Algorithm~\ref{alg:butterfly_clip}). Then, if validator $i$ accuses peer $j$ of modifying gradients, all other peers will be able to recalculate $j$'s gradients and compare their hash against the broadcasted one. If the peers find that $j$'s gradients are correct, peer $i$ is banned instead~\citep{hammurabi}. This procedure is described in Algorithm~\ref{alg:accuse}.

The resulting algorithm is resilient to attacks made through incorrect gradients. However, malicious peers may also harm training by violating the \algname{CenteredClip} procedure for the portion of gradients they are aggregating. Fortunately, we can design a test through which peers can verify that a vector they received is indeed the output of \algname{CenteredClip}. We need to view \algname{CenteredClip} as a fixed-point iteration for the equation (see details in Appendix~\ref{appendix:centered_clip}):
\begin{equation}
    \sum_{i=1}^n(\vec g_i - \vec x)\min\left\{1, \frac{\tau}{\|\vec g_i - \vec x\|}\right\} = 0
    \label{eq:fixed_point_iteration}
\end{equation}
The workers are not able to test whether (\ref{eq:fixed_point_iteration}) holds directly, since collecting $\vec g_i$ would require sending $O(dn)$ data per peer, defeating the purpose of our algorithm.

Instead, workers should use the MPRNG output to sample a random direction $\vec z$ in the space of model gradients. Then, each peer computes and broadcasts the inner product (\ref{eq:inner_product}):
\begin{equation}
    s_i {=} \left\langle \vec z, (\vec g_i - \vec x)\min\left\{1, \frac{\tau}{\|\vec g_i - \vec x\|}\right\}\right\rangle
    \label{eq:inner_product}
\end{equation}
Finally, all peers can verify that $\sum_{i=1}^n s_i = 0$. Similarly to our previous use of MPRNG, all aggregators must broadcast the hashes of their aggregation results (line~6 of Alg.~\ref{alg:butterfly_clip}) before they learn $\vec z$. This ensures that a malicious aggregator cannot modify the results in such a way that the difference would be orthogonal to $\vec z$ (this and more complex attack vectors are analyzed in Appendices~\ref{appendix:attack_types}~and~\ref{appendix:detecting_protocol_violations}).

We combine all these procedures in Algorithm~\ref{alg:btarsgd_outline} (see its scheme in Figure~\ref{fig:btard_scheme} and its detailed version in Alg.~\ref{alg:detailed_bytar}--\ref{alg:BTARD_SGD}). Crucially, one step of Algorithm~\ref{alg:btarsgd_outline} requires each peer to \textbf{(a)}~receive and send $n$ partitions of the gradient vector (exactly as in Butterfly All-Reduce), \textbf{(b)}~broadcast $O(n)$ scalar values (all hashes, the inner products $s_i^j$, and the necessary accusations), and \textbf{(c)}~run MPRNG once. According to Sections~\ref{sect:related_distributed} and \ref{sect:related_mpc}, the total communication cost of these procedures is $O(d + n^2)$ data per peer. This is close to $O(d)$ cost of the bandwidth-optimal versions of All-Reduce: the $O(n^2)$ extra cost is usually much smaller than $O(d)$ for models that benefit from distributed training.

The algorithm's synchronization points and computational overhead are reviewed in Appendix~\ref{appendix:complexity_overhead}.


\begin{figure}[tb]
\let\oldtextsc\textsc
\renewcommand{\textsc}[1]{\oldtextsc{\scalefont{0.9}#1}}

\vskip -0.15in

\begin{algorithm}[H]
  \caption{\algname{BTARD-SGD} for peer $i$ (informal)}
  \label{alg:btarsgd_outline}
\begin{algorithmic}[1]
  \Require rank $i$, model $x^0$, seed $\xi_i^0$, step~count~$T$, peer~count~$n$
  \For{$t \in 0, \dots, T {-} 1$}
    \State $g_i = \textsc{ComputeGradients}(x^t, \xi_i^t)$
    \State $\hat g = \textsc{ButterflyClip}(i, g_i)$
    \State $r^t = \textsc{MPRNG}()$ 
    \State $z = \textsc{GetRandomVector}(r^t)$ 
    \For{$j \in 1, \dots, n$}
     \State // $\hat g[j]$ is the aggregated part from peer $j$
     \State $\Delta_i^j {=} (g_i[j] - \hat g[j])\min\left\{1, \frac{\tau}{\|g_i[j] - \hat g[j] \|_2}\right\}$
     \State \textbf{broadcast} $s^j_i = \langle z[j], \Delta_i^j \rangle$
    \EndFor
    \For{$j \in 1, \dots, n$}
      \State // We know $\Delta_j^i$ from \textsc{CenteredClip}
      \If{$ s_j^i \neq \langle z[j], \Delta_j^i \rangle$}
        \State \textbf{broadcast} $s_j^i$ is wrong \hspace{0.5em} // Invokes Alg.~\ref{alg:accuse}
      \EndIf
      \If{$\sum_t^n s_t^j \neq 0$}
         \State // Peer $j$ lied that all $s_{\cdot}^j$ are correct
         \State \textbf{broadcast} $\hat g[j]$ is wrong \hspace{0.5em} // Invokes Alg.~\ref{alg:accuse}
      \EndIf
    \EndFor
    \State $x^{t+1} = \textsc{SGDStep}(x^t, \hat g)$
    \State $\xi_i^{t+1} = \text{hash}(r^t || i)$
    \If{$i \in \textsc{ChooseValidators}(r^t)$}
     \State $j = \textsc{ChooseTarget}(r^t, i)$
     \State $\textsc{ValidatePeer}(j, x^t, \xi_j^t, c_j, h_j^\text{*}, s_j^\text{*})$ 
     \State \newstuff{// ... instead of computing gradients for step $t {+} 1$}
    \EndIf
  \EndFor
  \State\Return $x^T$
\end{algorithmic}
\end{algorithm}
\vskip -0.35in
\end{figure}

\begin{figure}[tb]
\let\oldtextsc\textsc
\renewcommand{\textsc}[1]{\oldtextsc{\scalefont{0.9}#1}}

\vskip -0.15in

\begin{algorithm}[H]
  \caption{\algname{ButterflyClip} for peer $i$}
  \label{alg:butterfly_clip}
\begin{algorithmic}[1]
  \Require rank $i$, gradients $g_i \in \mathbb{R}^d$ 
  \State $g_{i}[1], ..., g_{i}[n] = \textsc{Split}(g_i, n)$
  \State \textbf{broadcast} $\forall j, \enskip h_i^j = \text{hash}(g_i[j])$
  \State \textbf{send} $\forall j, \enskip g_{i}[j] \rightarrow \text{peer}_j$
  \State \textbf{receive} $\forall j, \enskip g_{j}[i] \leftarrow \text{peer}_j$
    \hspace{0.5em} // and verify against $h_j^i$
  \State $\hat g_i = \textsc{CenteredClip}(g_{1}[i], ..., g_{n}[i])$
  \State \textbf{broadcast} $\enskip \hat h_i = \text{hash}(\hat g_i)$
  \State \textbf{send} $\forall j, \enskip \hat g_i \rightarrow \text{peer}_j$
  \State \textbf{receive} $\forall j, \enskip \hat g_j \leftarrow \text{peer}_j$
    \hspace{0.5em} // and verify against $\hat h_j$
  \State\Return $\textsc{Merge}(\hat g_{1}, ..., \hat g_{n})$
\end{algorithmic}
\end{algorithm}
\vskip -0.2in

\begin{algorithm}[H]
  \caption{\algname{Accuse}(i, j), invoked on all peers}
  \label{alg:accuse}
\begin{algorithmic}[1]
  \Require accuser $i$, target $j$ 
  \State $g_j = \textsc{ComputeGradients}(x^t, \xi_j^t)$
  \If{$\exists k : ( \text{hash}(g_j[k]) \neq h_j^k$ \\ $\hspace{12px} \textbf{ or } s_j^k {\neq}   \langle z[k], \Delta_j^k \rangle ) \textbf{ or } \sum_{k=1}^n s_k^j {\neq} 0$}
    \State \textsc{Ban}($\text{peer}_j$) \hspace{0.5em} // and everyone who covered it up
  \Else
    \State \textsc{Ban}($\text{peer}_i$)
  \EndIf
\end{algorithmic}
\end{algorithm}

\vskip -0.35in
\end{figure}

\subsection{Convergence analysis}\label{sect:method_analysis}

From the perspective of the optimization theory, our task is the expectation minimization problem:
\begin{equation}
    \min_{x\in Q\subseteq \R^d}\left\{f(x) := \EE_{\xi \sim \cD}\left[f(x,\xi)\right]\right\} \label{eq:main_problem}
\end{equation}
Here, the objective function $f$ is smooth and uniformly lower bounded, $Q\subseteq \R^d$ is a closed convex set of admissible parameters and $\xi$ is the source of stochasticity, such as minibatch indices. We assume that the problem~(\ref{eq:main_problem}) can be solved in a distributed manner, i.e., one can use $n$ workers calculating (mini-batched) stochastic gradients in parallel and communicating according to some protocol. We denote the set of workers as $[n] := \{1,2,\ldots,n\} = \cG \sqcup \cB$.

There are many ways for Byzantines to affect the training. We can classify all of them into four categories:
\textbf{(a)~gradient attacks}, where Byzantines modify their $g_i^k$, but otherwise behave normally; \textbf{(b)~aggregation attacks}, where a malicious aggregator returns wrong $\hat g_i$ and relies on others to cover it up by misreporting $s_i$; \textbf{(c)~reputation attacks}, such as slander via false $\textsc{Accuse}(i,j,\cdot)$; and \textbf{(d)~protocol violations}, that is, any other deviations from the steps of Algorithm~\ref{alg:btarsgd_outline} (e.g., refusing to send data within a predefined timeout). We elaborate on each attack type in Appendix~\ref{appendix:attack_types}.

For the purpose of this analysis, the latter two attacks can be repelled with an extra policy that allows an active worker to \textsc{Eliminate} any other worker at the cost of also being banned. If peer $i$ encounters a protocol violation from peer $j$, it broadcasts a message asking to remove both peers $i$ and $j$ from training. The design of this policy ensures that every such message, whether sent by honest or Byzantine peers, eliminates at least 1 Byzantine peer and at most 1 honest peer (see details in Appendix~\ref{appendix:accuse_and_eliminate}). Thus, if a Byzantine minority uses this against honest peers, it will only decrease their relative numbers: $(\delta n - 1) / (n - 2) < \delta $. This leaves us only with the attacks targeting the aggregated gradients. 

{\renewcommand{\arraystretch}{1.2}
\begin{table*}[tb]
    \vskip -0.1in
    \centering
    \caption{Summary of complexity bounds for \algname{BTARD-SGD} in different scenarios. By complexity we mean the number of iterations sufficient to find such point $\widehat{x}$ that $\EE[\|\nabla f(\widehat{x})\|^2] \le \varepsilon^2$ for non-convex problems and $\EE[f(\widehat{x}) - f(x^*)] \le \varepsilon$ for convex and $\mu$-strongly convex problems (see Def.~\ref{def:mu_strong_convexity}) with $x^*$ being the solution. Notation: ``known $|\cB_k^a|$'' = the exact number of attacking Byzantine workers at iteration $k$ is known to each participant, $L$ = smoothness constant (see Def.~\ref{def:L_smoothness}), $\Delta_0 = f(x^0) - f_*$, $f_*$ = uniform lower bound for $f$, $\sigma^2$ = variance parameter from As.~\ref{as:bounded_var}, $n$ = the initial number of peers, $b$ = the initial number of Byzantine workers, $\delta = \nicefrac{b}{n}$, $m$ = number of peers checked at each iteration, $R_0 = \|x^0 - x^*\|$.}
    \label{tab:results_summary}
    \vskip 0.1in
    \begin{tabular}{lccc}
    \toprule
    \multirow{2}{*}{Assumptions}& \multicolumn{3}{c}{Convexity of $f$}\\
    \cmidrule(lr){2-4}
    & Non-convex & Convex & Strongly convex \\
    \midrule
    
    As.~\ref{as:bounded_var}+ As.~\ref{as:quadratically_bounded_tails} & \multirow{2}{*}{
    $\frac{L\Delta_0}{\varepsilon^2}\! +\! \frac{L\Delta_0\sigma^2}{n\varepsilon^4}\! + \!\frac{n\delta\sigma^2 }{m\varepsilon^2}$
    }&
    \multirow{2}{*}{
    $\frac{LR_0^2}{\varepsilon}\! +\! \frac{\sigma^2R_0^2}{n\varepsilon^2}\! +\! \frac{n\sqrt{\delta}\sigma R_0}{m\varepsilon}$
    }& 
    \multirow{2}{*}{
    $
    \frac{L}{\mu}\log\frac{\mu R_0^2}{\varepsilon}\! +\! \frac{\sigma^2}{n\mu\varepsilon} \!+\! \frac{n\sqrt{\delta}\sigma}{m\sqrt{\mu\varepsilon}}
    $}\\
     + known $|\cB_k^{a}|$ & & & \\
    
    \midrule
    As.~\ref{as:bounded_var} + As.~\ref{as:quadratically_bounded_tails} & $\frac{L\Delta_0}{\varepsilon^2}\! +\! \frac{L\Delta_0\sigma^2}{n\varepsilon^4}\! +\! \frac{n^2\delta\sigma^2 }{m\varepsilon^2}$
    &
    $\frac{LR_0^2}{\varepsilon}\! +\! \frac{\sigma^2R_0^2}{n\varepsilon^2}\! +\! \frac{n^2\delta\sigma R_0}{m\varepsilon}$
    & 
    $\frac{L}{\mu}\log\frac{\mu R_0^2}{\varepsilon} \!+\! \frac{\sigma^2}{n\mu\varepsilon} \!+\! \frac{n^2\delta\sigma}{m\sqrt{\mu\varepsilon}}$\\
    \bottomrule

    \end{tabular}
    \vskip -0.15in
\end{table*}
}


We provide convergence guarantees for variants of \algname{BTARD-SGD} with $Q = \R^d$ under different sets of assumptions about the function $f$ and its stochastic gradients. Our first two setups assume that:

\begin{assumption}\label{as:bounded_var}
    There exist such constant $\sigma \ge 0$, $s_0 \in [d]$ that for any set of indices $S = (i_1,\ldots,i_s)$, $1\le i_1< i_2 <\ldots < i_s \le d$, $s \ge s_0$ stochastic gradient $\nabla f(x,\xi)$ satisfy
    \begin{equation*}
        \EE[\nabla f(x,\xi)] = \nabla f(x),
    \end{equation*}
    \begin{equation}
        \EE\left[\left\|\nabla_{[S]} f(x,\xi) - \nabla_{[S]} f(x)\right\|^2\right] \le \frac{s\sigma^2}{d}, \label{eq:uniformly_bounded_var}
    \end{equation}
where $\nabla_{[S]} f(x,\xi) = (\nabla_{i_1} f(x,\xi),\ldots,\nabla_{i_s} f(x,\xi))^\top$, $\nabla_{[S]} f(x) = (\nabla_{i_1} f(x),\ldots,\nabla_{i_s} f(x))^\top$, and $\nabla f_j(x,\xi), \nabla_j f(x)$ are $j$-th components of $\nabla f(x,\xi)$ and $f(x)$ respectively.
\end{assumption}
Here, \eqref{eq:uniformly_bounded_var} is an extension of the classical uniformly bounded variance (UBV) assumption \citep{nemirovski2009robust,ghadimi2012optimal,ghadimi2013stochastic} ensuring that the noise in all subvectors of large enough dimension has the variance dependent on the ratio between the dimension of the subvector $s$ and the dimension of the full vector $d$. For example, it holds when the noise is isotropic. Moreover, one can relax this assumption to the standard UBV assumption, if blocks for aggregation in \algname{BTARD} are chosen uniformly at random (see Appendix~\ref{sec:on_assumptions}). In order to further reduce overhead from \textbf{Verification 3} in the full Algorithm~\ref{alg:detailed_bytar}, we also assume that the stochastic gradient distributions have sub-quadratically decreasing tails (see Appendix~\ref{sec:on_assumptions}).

\begin{assumption}\label{as:quadratically_bounded_tails}
    There exist such constant $\sigma \ge 0$, $s_0 \in [d]$ that for any set of indices $S = (i_1,\ldots,i_s)$, $1\le i_1< i_2 <\ldots < i_s \le d$, $s \ge s_0$ and any $t > 0$ stochastic gradient $\nabla f(x,\xi)$ satisfy\vspace{-2px}
    \begin{equation*}
        \PP\left\{\left\|\frac{1}{k}\sum\limits_{i=1}^k\nabla_{[S]}f(x,\xi_i) - \nabla_{[S]} f(x)\right\|^2 > \frac{ts\sigma^2}{kd}\right\} < \frac{1}{t^2},
    \end{equation*}
    where $\xi_1,\ldots,\xi_k$ are i.i.d.\ samples from $\cD$, and $\nabla_{[S]} f(x,\xi)$, $\nabla_{[S]} f(x)$ are defined in As.~\ref{as:bounded_var}.
\end{assumption}

Under these assumptions, we derive the following convergence bounds for strongly convex, generally convex, and non-convex objectives (see Table~\ref{tab:results_summary}). The respective proofs and further details are deferred to Appendix~\ref{sec:BTARD_SGD_appendix}.

\textbf{Discussion of the convergence bounds.} Let us briefly discuss the main properties of the derived results. When $\delta = 0$ (there are no Byzantine peers), we recover the tightest known rates for parallel \algname{SGD} for strongly convex, generally convex, and non-convex objectives with both sets of assumptions. Next, we notice that in all complexity bounds in the known $|\cB_k^a|$ case, the term depending on the ratio of Byzantine workers $\delta$ (the third one in all bounds) has better dependence on the accuracy of the solution $\varepsilon$ than the classical variance term (the second one in all bounds). Therefore, for sufficiently small $\varepsilon$, the derived complexity bounds are the same as in the case when there are no Byzantine workers and parallel \algname{SGD} is used. However, these bounds are obtained under the assumption that all participants know the exact number of attacking Byzantine workers at each iteration, which is not realistic \newstuff{but helps to better adjust clipping parameter $\tau$ in \algname{CenteredClip}}.

As for the more general case, the third term is much worse than the corresponding term in the previous setup. Nevertheless, the term that depends on the ratio of Byzantine workers $\delta$ has the same dependence on $\varepsilon$ as in the known $|\cB_k^a|$ case. This implies that for sufficiently small $\varepsilon$ the derived complexity bounds are the same as in the case when there are no Byzantine workers and parallel \algname{SGD} is used. We provide the complete formulations and proofs in Appendix~\ref{sec:BTARD_SGD_appendix}.

Finally, the derived convergence results are superior to the previous state-of-the-art ones \emph{even in the PS setup} if $\varepsilon$ is sufficiently small. For example, in the non-convex case, \citet{karimireddy2020learning} show the $\cO(\nicefrac{1}{\varepsilon^2} + \nicefrac{\sigma^2}{n\varepsilon^4} + \nicefrac{\delta\sigma^2}{\varepsilon^4})$\footnote{For simplicity we omit numerical factors, logarithmic terms depending on the parameters of the problem, and factors, quantifying suboptimality of the starting point, i.e., $R_0 = \|x^0 -x^*\|$ and $f(x^0) - \inf_{x\in \R^d} f(x)$.} complexity bound for a version of \algname{Momentum-SGD} that uses \algname{CenteredClip} aggregation rule when PS is available. When there is at least one Byzantine peer, the leading term in the above bound is $\cO(\nicefrac{\delta\sigma^2}{\varepsilon^4})$ since $\delta \geq \nicefrac{1}{n}$. In contrast, when $\varepsilon$ is sufficiently small, i.e., $\varepsilon \leq \cO(\sqrt{\nicefrac{L\Delta_0 m}{n^3\delta \sigma^2}})$ (see Table~\ref{tab:results_summary}), the leading term in our bound is $\cO(\nicefrac{\sigma^2}{n\varepsilon^4})$, which is better than $\cO(\nicefrac{\delta\sigma^2}{\varepsilon^4})$. However, it is worth mentioning the differences between the setups. Although we do not assume the existence of PS, our algorithm and theoretical analysis rely on the usage of part of the workers to check the computations of some other workers and we allow bans of the participants. This is the key feature of our algorithm allowing us to obtain the improvement. See the detailed comparison with other works in Appendix~\ref{appendix:Byz_tol_opt_details}.

\textbf{Intuition behind the proofs.} First, we show that all possible violations of our protocol either lead to the instant ban of a Byzantine peer or (with some positive probability) to the ban during the checks of computations following each iteration of the algorithm (Appendix~\ref{appendix:detecting_protocol_violations}). Next, we upper-bound the (expected squared) shifts that Byzantines can create at each iteration (Lemmas~\ref{lem:quality_of_agg_known_num_byz} and \ref{lem:quality_of_agg_unknown_num_byz}). Therefore, in expectation, Byzantines can deviate from the protocol only a finite number of times, and ``the power'' of their attacks is limited at each particular iteration. Using these results, we analyze \algname{BTARD-SGD} as parallel \algname{SGD} with shifted updates, where the shifts are bounded and exist during the finite number of steps (see Appendices~\ref{appendix:non_convex_analysis_BTARD_SGD}--\ref{appendix:str_convex_analysis_BTARD_SGD}).

\textbf{Results for heavy-tailed problems.} So far, all our convergence results rely on As.~\ref{as:quadratically_bounded_tails}, i.e., that the stochastic gradients have not too heavy tails. This assumption holds for many real-world neural networks. However, there are important NLP tasks such as BERT training \citep{zhang2020why}, where the noise in the stochastic gradient has so heavy tails that As.~\ref{as:quadratically_bounded_tails} becomes unrealistic. The third and final setup in our analysis aims to address such heavy-tailed problems with \algname{BTARD-Clipped-SGD} (Algorithm~\ref{alg:BTARD_Clipped_SGD} in Appendix~\ref{appendix:BTARD_Clipped_SGD_appendix}). We analyze the method under the assumption that $\alpha$-th moments of the stochastic gradients are uniformly upper-bounded for some $\alpha\in(1, 2]$. We notice that for $\alpha < 2$ this assumption allows the variance of the stochastic gradient to be unbounded. In this setting, we prove that \algname{BTARD-Clipped-SGD} finds an $\varepsilon$-solution of the convex problem after $\cO\left(\varepsilon^{-\nicefrac{\alpha}{(\alpha-1)}}\left(1+\left(\nicefrac{n\sqrt{\delta}}{m}\right)^{\nicefrac{\alpha}{(\alpha-1)}}\right)\right)$ iterations when the number of attacking Byzantine peers is known at each iteration and $\cO\left(\varepsilon^{-\nicefrac{\alpha}{(\alpha-1)}}\left(1+\left(\nicefrac{n^2\delta^2}{m}\right)^{\nicefrac{\alpha}{(\alpha-1)}}\right)\right)$ iterations otherwise. One can find the full statements and complete proofs of our results in Appendix~\ref{appendix:extra_analysis}.

\subsection{Resisting Sybil attacks}\label{sect:method_reputation}

\begin{figure*}[tb]
    \vskip 0.05in
    \centering
    \includegraphics[width=0.925\textwidth]{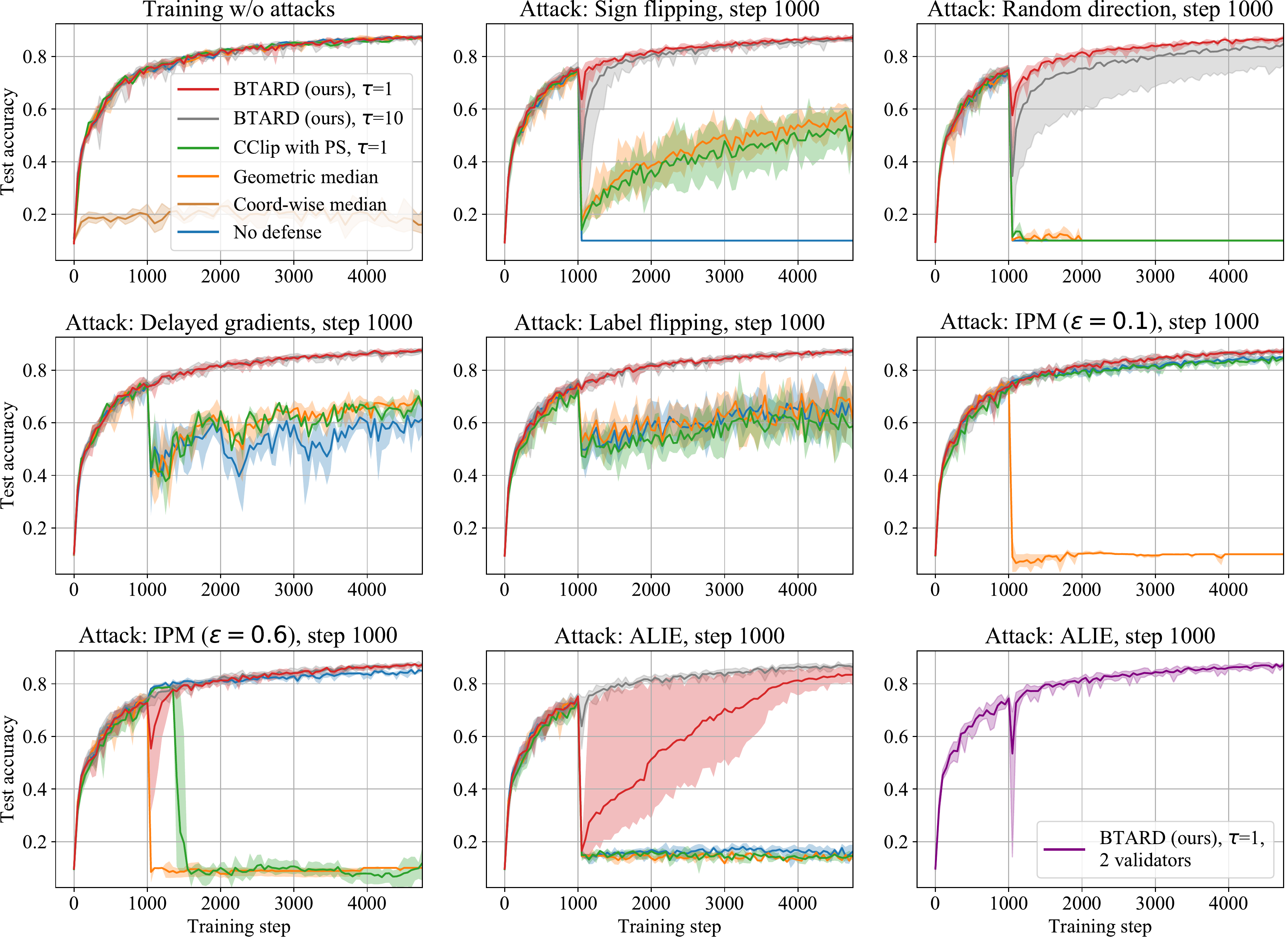}
    \vskip -0.1in
    \caption{The ResNet-18 test accuracy in the case of various attacks and robust aggregation techniques.}
    \label{fig:resnet}
    \vskip -0.15in
\end{figure*}

The algorithm described in Section~\ref{sect:method_aggregation} operates with a predefined list of peers that can only decrease in size. However, many real-world scenarios would benefit from new peers joining midway through training. Unfortunately, this exposes the system to Sybil attacks~\citep{sybil}, when a single computationally constrained attacker adopts multiple pseudonymous identities in order to establish a dishonest majority and break the algorithm.

To handle this, one may augment BTARD with a heuristic protocol that dictates how new peers can join the experiment. A new participant must prove that it has honestly computed enough gradients over multiple continuous iterations before it is allowed to actually contribute to the training. This ensures that the influence of Sybil attackers is proportional to their computing power (see details in Appendix~\ref{appendix:reputation}).




\section{Experiments}\label{sect:experiments}

\begin{figure*}[tb]
    \vskip 0.05in
    \centering
    \includegraphics[width=0.925\textwidth]{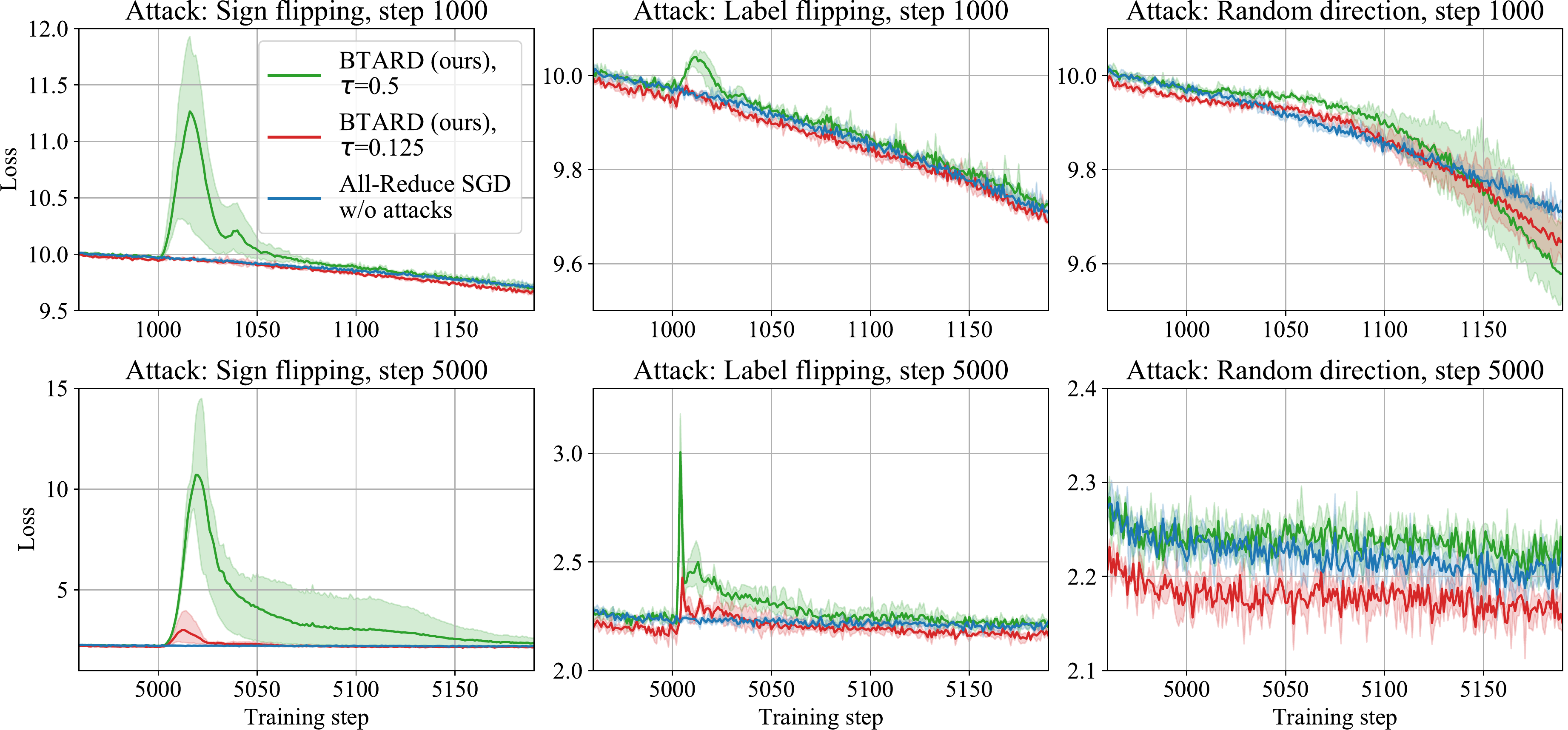}
    \vskip -0.1in
    \caption{The ALBERT-large training objective in the case of \algname{BTARD-Clipped-SGD} (in presence of various attacks) and the standard All-Reduce SGD (without attacks).}
    \label{fig:albert}
    \vskip -0.15in
\end{figure*}

\subsection{CIFAR10 classification}\label{sect:experiments_resnet}

First, we evaluate our approach in controlled conditions.
Our setup is a ResNet-18~\citep{resnet} model trained to solve the CIFAR10 classification task~\citep{cifar}. We train the model on 16 peers (each peer processes 8 samples per batch) using \algname{SGD} with \citet{nesterov} momentum and the cosine annealing learning rate~\citep{loshchilov2016sgdr}. We use a tuned setup achieving 93.5\% test accuracy.

Our method has a hyperparameter $\tau$ responsible for clipping strength in \algname{CenteredClip}. We experiment with $\tau = 10$ (weaker clipping) and $\tau = 1$ (stronger clipping). These values were chosen based on the maximal standard deviation of the gradient parts averaged by the workers during normal training, so that almost no vectors are clipped for the weaker clipping and almost half of the vectors are clipped for the stronger clipping scenario.
We begin with using only 1 validator on each step. If a validator happens to be Byzantine, it never accuses its peers.

We compare our method to the regular All-Reduce without clipping \newstuff{and the baselines that use a trusted parameter server: the original variant of \algname{CenteredClip}~\citep{karimireddy2020learning}, the coordinate-wise and geometric medians}. Some other popular robust aggregation techniques are omitted because they were shown to be inferior in ~\citet{karimireddy2020learning}. We run all iterative algorithms (such as \algname{CenteredClip}) to convergence with $\epsilon=10^{-6}$, as we have found that limiting the number of iterations can significantly decrease the final model quality (see Figure~\ref{fig:resnet_cclip_steps} in Appendix~\ref{appendix:resnet18_extra}).

In addition to measuring training convergence, we evaluate our setup in presence of malicious peers. To test pessimistic conditions, we pick a setting where 7 of 16 peers are Byzantine (see Appendix~\ref{appendix:resnet18_extra} for a setup with 3 Byzantines). We experiment with the following attack types:
\begin{itemize}[leftmargin=*]
    \vspace{-1px}\item
    \textsc{Sign flipping:} each attacker sends the opposite of its true gradient.
    \vspace{-1px}\item
    \textsc{Random direction:} all attackers send large vectors pointed in a common random direction.
    \vspace{-1px}\item
    \textsc{Label flipping:} each attacker computes its gradient based on the cross-entropy loss with flipped labels. For CIFAR-10, we replace label $l \in \{0, ..., 9\}$ with $9 - l$.
    \vspace{-1px}\item
    \textsc{Delayed gradient:} 
    attackers send their real gradients delayed by $1000$ steps.
    \vspace{-1px}\item
    \textsc{Inner product manipulation (IPM):} 
    attackers send the average of all honest peers' gradients multiplied by $- \epsilon$.
    We test $\epsilon = 0.1$ (\citet{xie2020fall} demonstrate its efficiency against the coordinate-wise median and Krum) and $\epsilon = 0.6$ (\citet{allen2020byzantine} report it as the most efficient attack against their \algname{SafeguardSGD}).
    \vspace{-1px}\item
    \textsc{``A little is enough'' (ALIE):} 
    attackers collude to move the coordinate-wise median while still sending values inside the population variance. \citet{baruch2019little} show that this attack is effective against TrimmedMean \citep{yin2018byzantine} and Krum \citep{blanchard2017machine}.
\end{itemize}

We further amplify the Byzantine gradients from the first two attacks by a large coefficient $\lambda = 1000$ so they would dominate the aggregated gradient if no clipping is used. While in practice such attacks can be identified right away by the large gradient norms, we deliberately avoid doing that to test our clipping approach.

Here, we make Byzantines behave honestly prior to step $s = 1000$, then simultaneously attack on each step until they are banned (i.e., attacks start at the early stage of training). In Appendix~\ref{appendix:resnet18_extra}, we have also evaluated our method in the case of $s = 10{,}000$ (i.e., attacks start closer to the convergence stage) and in the case of Byzantines attacking periodically, but found no significant differences in its behavior.

We repeat each experiment 5 times and report the mean and range (between the minimum and maximum) of the test accuracy during at least 2000 steps after all Byzantines are banned. In our experiments, this usually happened within 150 steps after $s$.

We observe that our method does not significantly worsen the speed of convergence compared to the All-Reduce baseline (see Figure~\ref{fig:resnet}, upper-left). On average, the final test accuracy after 25,000 steps is only 0.6\% worse for $\tau = 1$ and even 0.1\% better for $\tau = 10$. The other tested defenses have a similar effect, except for the coordinate-wise median, which does not converge even without attacks.

Next, we find that BTARD with the stronger clipping and only 1 validator protects from all tested attack types except the ALIE attack, where we need 2 validators to guarantee recovery (see Figure~\ref{fig:resnet}). In general, the weaker clipping is most sensitive to the attacks with large magnitudes (sign flipping and random direction), while the stronger clipping is sensitive to the low-magnitude ALIE attack\footnote{This observation coincides with \citet{baruch2019little} demonstrating that the \algname{ALIE} attack is more harmful against median-based and clipping approaches than to the mean aggregation without any defenses. Indeed, the stronger clipping makes \algname{CenteredClip} closer to the geometric median, and the weaker clipping makes it closer to the mean (as explained in Appendix~\ref{appendix:centered_clip}).}. The other defenses fail to protect training from most attack types.

We conclude that BTARD with the stronger clipping and 2 validators (i.e., $1/8$ of the compute dedicated for validation) allows to quickly recover the pre-attack accuracy after all tested attack types even in the extreme case with 7 out of 16 peers being Byzantine.


\subsection{Pre-training transformers}\label{sect:experiments_albert}

In this section, we choose a more compute-intensive and hyperparameter-sensitive model with an adaptive optimizer to demonstrate that our approach may be applied to the models commonly used in distributed training scenarios. Our setup is pre-training ALBERT-large~\citep{albert} on the WikiText-103 dataset~\citep{wikitext103} using the LAMB optimizer~\citep{lamb} (see details in Appendix~\ref{appendix:config_albert}). Since the original ALBERT setup uses gradient clipping, we use \algname{BTARD-Clipped-SGD} (see Algorithm~\ref{alg:BTARD_Clipped_SGD} in Appendix~\ref{appendix:BTARD_Clipped_SGD_appendix}). We train the model on 16 machines that jointly accumulate 4096 samples for every batch.

We evaluate the All-Reduce baseline without attacks, as well as BTARD with weaker and stronger clipping (larger and smaller $\tau$ respectively) in presence of attackers. In the last case, we make 7 workers malicious, use 1 validator, and test two attack regions: $s = 1000$ and $s = 5000$. We omit reporting of the delayed gradient attack (due to its inefficiency), as well as ALIE and IPM attacks (they require Byzantines to make an extra All-Reduce round on each step, which is hard to do efficiently in the real multi-host setup).

As in the previous section, we observe that without attacks both $\tau$ values have no significant effect on the training progress, \newstuff{reaching only 1.3\% larger loss in the worst case}. However, the stronger clipping shows faster recovery from the tested attacks (see Figure~\ref{fig:albert}). Crucially, while some attacks significantly increase the loss function, the model recovers much faster than it takes to reach the pre-attack loss when training from scratch.

We also report the computation overhead of \textsc{BTARD-SGD} in this setup in Appendix~\ref{appendix:compute_overhead_eval} and conduct experiments with 64 machines and most efficient attacks in Appendix~\ref{appendix:eval_at_scale}, confirming that \algname{BTARD} remains efficient at a larger scale.

\section{Conclusion}\label{sect:discussion}

In this work, we formulated \algname{BTARD-SGD}~--- a Byzantine-tolerant training strategy for large neural networks. We verified its robustness and efficiency through theoretical analysis and large-scale distributed training experiments.

Our research opens new opportunities in many deep learning applications, making it possible to train large neural networks in a cooperative manner.
Small research groups can host open cooperative training projects where the training hardware is crowdsourced by volunteers around the world, or a group of small companies can compete with larger corporations by combining their compute clusters.  While these applications also require engineering effort to become practical, our algorithm ensures that they can run securely without the need to screen every potential participant.


\vspace{-5px}
\paragraph{Acknowledgments.} This work was partially supported by a grant for research centers in the field of artificial intelligence, provided by the Analytical Center for the Government of the Russian Federation in accordance with the subsidy agreement (agreement identifier 000000D730321P5Q0002) and the agreement with the Moscow Institute of Physics and Technology dated November 1, 2021 No. 70-2021-00138.

We thank Sai Praneeth Karimireddy for useful discussions and suggestions, Lie He for providing the code with \algname{CenteredClip}, William Cappelletti for pointing out several relevant papers, Gennady Pekhimenko for his technical expertise and infrastructure for the distributed training experiments, Dmitrii Emelianenko for helpful discussions, and Nazarii Tupitsa for spotting inaccuracies in the proofs of Lemmas \ref{lem:BTARD_for_bad_peer_known_num_byz} and \ref{lem:quality_of_agg_unknown_num_byz}.


\nocite{paszke2019pytorch}
\bibliography{main}
\bibliographystyle{icml2022}

\clearpage

\appendix
\onecolumn
\section*{\scalefont{1.2}Appendix}

\vspace{0.5em}
\renewcommand{\contentsname}{Table of contents}
{\small\setlength{\parskip}{0.025em}\tableofcontents}

\section{Additional related work}

\subsection{Byzantine-tolerant optimization: additional details}\label{appendix:Byz_tol_opt_details}

In this section, we provide extra details on the related work discussed in Section~\ref{sect:related_byzantine}. The summary of complexity results is presented in Table~\ref{tab:comparison_of_rates}.

\subsubsection{Parameter-server (PS) based approaches}\label{appendix:extra_related_ps}

There is a quite large number of papers on Byzantine-tolerant optimization that aim to robustify parallel \algname{SGD} in the case when a trusted parameter-server (PS) is available. Since in the classical parallel \algname{SGD} even one Byzantine worker can break the convergence of the whole method by shifting the mean of the resulting vector in an arbitrary way, it is natural to substitute averaging of the vectors received from the workers by a more robust aggregation rule, e.g., Krum \citep{blanchard2017machine}, coordinate-wise median, trimmed median \citep{yin2018byzantine}, Multi-Krum \citep{damaskinos2019aggregathor}, Bulyan \citep{mhamdi2018hidden}, geometric median \citep{pillutla2019robust}. However, all these methods were shown to be brittle and not robust to special types of Byzantine attacks~\citep{baruch2019little,xie2020fall,karimireddy2020learning}. Moreover, \citet{karimireddy2020learning} show that all permutation-invariant algorithms cannot converge to any predefined accuracy of the solution, meaning that simple application of some aggregation rules on top of \algname{SGD} does not lead to Byzantine tolerance.

There are several approaches to circumvent this issue. \citet{alistarh2018byzantine} propose \algname{ByzantineSGD} and prove the convergence results for convex problems. \citet{allen2020byzantine} extend this approach to handle non-convex problems as well. In both papers, the key idea is based on applying the concentration properties of the sums depending on the stochastic gradients as well as iterative removing of Byzantine peers. However, theoretical guarantees from \citet{alistarh2018byzantine,allen2020byzantine} rely on the restrictive assumption that the noise in the stochastic gradients is uniformly bounded with probability $1$. \citet{bulusu2020distributed} propose similar approach to the one from \citep{allen2020byzantine} but analyze their method under more restrictive assumptions (boundedness of the gradient). Next, \citet{wu2020federated} propose a Byzantine-tolerant version of parallel \algname{SAGA} \citep{defazio2014saga}, i.e., variance-reduced version of \algname{SGD}, with geometric median as an aggregation rule --- \algname{Byrd-SAGA} -- and prove its convergence for strongly convex objectives. However, the authors do not establish the convergence of \algname{Byrd-SAGA} to any predefined accuracy of the solution. Moreover, variance-reduced methods are known to converge slowly in deep learning applications~\citep{defazio2018ineffectiveness}, which limits the practical utility of \algname{Byrd-SAGA}. Finally, \citet{karimireddy2020learning} propose a new aggregation rule called \algname{CenteredClip}, apply it to \algname{SGD} with client momentum, and prove convergence results for the obtained method in the non-convex case under reasonable assumptions. Alternative lines of work achieve Byzantine-tolerant optimization through redundant computations \citep{chen2018draco,rajput2019detox} or reputation-based approaches \citep{rodriguez2020dynamic, regatti2020bygars, xu2020towards}. Unfortunately, these papers either do not contain theoretical (non-asymptotic) convergence results for the proposed methods or rely on too restrictive assumptions in the analysis. See more references in the recent survey by \citet{lyu2020privacy}.

\subsubsection{Decentralized approaches}\label{appendix:extra_related_decentralized}

Byzantine-tolerant optimization methods for decentralized communication architectures are studied only in a couple of papers. \citet{yang2019bridge,yang2019byrdie} consider a specific scenario when workers compute full gradients, local loss functions on peers are heterogeneous, and the trimmed coordinate-wise median is used as an aggregation rule. In this setup, the authors prove convergence results in the strongly convex case to some accuracy depending on the heterogeneity level of local loss functions, which is natural in the presence of Byzantine peers. However, these results are not applicable to a wide range of practically important problems where stochastic gradients have to be used. This issue was partially resolved in \citet{peng2021byzantine}, where the authors propose a version of \algname{Gossip SGD} applied to the equivalent reformulation of the original problem based on TV-regularization \citep{ben2013robust}. However, the established convergence results in the strongly convex case do not show any benefits of using communications with other workers in the homogeneous data regime that appears in large-batch training of deep learning models. \citet{li2019rsa} use the same idea for a parameter-server architecture. \newstuff{Next}, there are approaches requiring peer-to-peer communications of full vectors at each step \citep{gupta2021byzantine, gupta2021byzantine2}, which is not scalable.

\begin{table}[H]
    \vskip -0.1in
    \centering
    \caption{\small Summary of the complexity results for Parameter-Server (PS) based and distributed Byzantine-tolerant optimization. The columns ``Non-convex'', ``Convex'', and ``Strongly convex'' contain the complexity bounds for $L$-smooth non-convex, convex, and $\mu$-strongly convex problems respectively. By complexity we mean the number of iterations sufficient to find such point $\widehat{x}$ that $\EE[\|\nabla f(\widehat{x})\|^2] \le \varepsilon^2$ for non-convex problems and $\EE[f(\widehat{x}) - f(x^*)] \le \varepsilon$ for convex and $\mu$-strongly convex problems (see Def.~\ref{def:mu_strong_convexity}) with $x^*$ being the solution. For simplicity, we omit numerical factors, logarithmic terms depending on the parameters of the problem, and factors, quantifying suboptimality of the starting point, i.e., $R_0 = \|x^0 -x^*\|$ and $f(x^0) - \inf_{x\in \R^d} f(x)$. Notation: $\delta = \nicefrac{|\cB|}{n}$, $m$ = number of peers checked at each iteration. The results from \citet{yang2019bridge,yang2019byrdie} are not included since they rely on full-gradient computations.}
    \label{tab:comparison_of_rates}
    \vskip 0.1in
    {\small
    \begin{threeparttable}
    \begin{tabular}{|c|c|c c c|}
        \hline
        Non-PS? & Work & Non-convex & Convex & Strongly convex \\
        \hline\hline
        \multirow{4.5}{*}{\xmark} & \citep{alistarh2018byzantine}\tnote{\color{red}(1),(2)}~~~~~  & \xmark & $\frac{1}{\varepsilon} + \frac{\sigma^2}{n\varepsilon^2} + \frac{\delta^2\sigma^2}{\varepsilon^2}$ & $\frac{1}{\mu} + \frac{\sigma^2}{n\mu\varepsilon} + \frac{\delta^2\sigma^2}{\mu\varepsilon}$ \\
        & \citep{allen2020byzantine}\tnote{\color{red}(1),(3)}~~~~~  & $\frac{1}{n\varepsilon^4} + \frac{\delta^2}{\varepsilon^4}$ & \xmark & \xmark \\
        & \citep{wu2020federated}\tnote{\color{red}(4)}  & \xmark & \xmark & $\frac{L^2}{\mu^2}$\tnote{\color{red}(5)} \\
        & \citep{karimireddy2020learning}\tnote{\color{red}(6)}~  & $\frac{1}{\varepsilon^2} + \frac{\sigma^2}{n\varepsilon^4} + \frac{\delta\sigma^2}{\varepsilon^4}$ & \xmark & \xmark \\
        \hline\hline
         \multirow{7}{*}{\cmark} & \citep{peng2021byzantine}\tnote{\color{red}(6),(7)} & \xmark & \xmark & $\frac{1}{\mu\varepsilon} + \frac{n\sigma^2}{\mu^2\varepsilon} + \frac{\lambda^2 d \overline{N}^2}{\mu^2\varepsilon}$\\
        & \textbf{This work}\tnote{\color{red}(8)} & $\frac{1}{\varepsilon^2} + \frac{\sigma^2}{n\varepsilon^4} + \frac{n\delta\sigma^2 }{m\varepsilon^2}$
    &    $\frac{1}{\varepsilon} + \frac{\sigma^2}{n\varepsilon^2} + \frac{n\sqrt{\delta}\sigma}{m\varepsilon}$ & $
    \frac{1}{\mu} + \frac{\sigma^2}{n\mu\varepsilon} + \frac{n\sqrt{\delta}\sigma}{m\sqrt{\mu\varepsilon}}
    $\\
        & \textbf{This work}\tnote{\color{red}(9)} & $\frac{1}{\varepsilon^2} + \frac{\sigma^2}{n\varepsilon^4} + \frac{n^2\delta\sigma^2 }{m\varepsilon^2}$& $\frac{1}{\varepsilon} + \frac{\sigma^2}{n\varepsilon^2} + \frac{n^2\delta\sigma}{m\varepsilon}$& $\frac{1}{\mu} + \frac{\sigma^2}{n\mu\varepsilon} + \frac{n^2\delta\sigma}{m\sqrt{\mu\varepsilon}}$\\
        & \textbf{This work}\tnote{\color{red}(10)} & \xmark & $\left(\frac{G\Lambda_1}{\varepsilon}\right)^{\frac{\alpha}{\alpha-1}}$ & $\left(\frac{G^2\Lambda_1}{\mu\varepsilon}\right)^{\frac{\alpha}{2(\alpha-1)}}$ \\
        & \textbf{This work}\tnote{\color{red}(11)} & \xmark & $\left(\frac{G\Lambda_2}{\varepsilon}\right)^{\frac{\alpha}{\alpha-1}}$ & $\left(\frac{G^2\Lambda_2}{\mu\varepsilon}\right)^{\frac{\alpha}{2(\alpha-1)}}$\\
        \hline
    \end{tabular}
    \vskip 0.1in
    \begin{tablenotes}
        {\small
        \item [{\color{red}(1)}] The results are proven under uniformly bounded noise assumption: $\|\nabla f(x,\xi) - \nabla f(x)\| \le \sigma$ for all $x$ and $\xi$. High-probability guarantees are established, i.e., it is proven that with probability at least $1 - \beta$ algorithms from \citep{alistarh2018byzantine} find $\hat x$ such that $f(\hat x) - f(x^*) \le \varepsilon$ and algorithms from \citep{allen2020byzantine} find $\hat x$ such that $\|\nabla f(\hat x)\| \le \varepsilon$.
        \item [{\color{red}(2)}] Dependencies on $\beta$ are logarithmic and, therefore, omitted. The optimization problems are assumed to be defined on a bounded set, the rates depend on the diameter of this set.
        \item [{\color{red}(3)}] The results are derived for the case $\sigma = 1$. \citet{allen2020byzantine} also derive convergence guarantees for finding second-order stationary points.
        \item [{\color{red}(4)}] \citet{wu2020federated} consider finite-sum case of \eqref{eq:main_problem}, i.e., $f(x) = \frac{1}{N}\sum_{j=1}^N f(x,j)$. The results are derived under the uniformly bounded variance (UBV) assumption: $\EE_j[\|\nabla f(x,j) - \nabla f(x)\|^2] \le \sigma^2$ for all $x\in \R^d$, where $j$ is sampled uniformly at random from $\{1,\ldots,N\}$. \citet{wu2020federated} also derive convergence guarantees under $\zeta$-bounded dissimilarity assumption, i.e., when $f(x) = \frac{1}{|\cG|}\sum_{i\in \cG} f_i(x)$, $f_i(x) = \frac{1}{N}\sum_{j=1}^N f_i(x,j)$ for all $i\in \cG$, and $\frac{1}{|\cG|}\sum_{i\in \cG}\|\nabla f_i(x) - \nabla  f(x)\|^2 \leq \zeta^2$.
        \item [{\color{red}(5)}] This result is obtained the main result of  \citep{wu2020federated} and states that the method from \citep{wu2020federated} finds $\hat x$ such that $f(\hat x) - f(x^*) \le \varepsilon$ only for $\varepsilon \ge \nicefrac{\sigma^2}{\mu^2\left(\frac{1}{2}-\delta\right)^2}$, which can be large.
        \item [{\color{red}(6)}] The result is derived under the UBV assumption, i.e., $\EE_{\xi\sim \cD}[\|\nabla f(x,\xi) - \nabla f(x)\|^2] \leq \sigma^2$ for all $x \in \R^d$.
        \item [{\color{red}(7)}]  \citet{peng2021byzantine} consider the case, when peers are allowed to communicate with their neighbors that are defined via some communication graph. The result establishes the total number of iterations/communication rounds needed to find $\hat x$ such that $\EE\|\hat x - x^*\|^2 \le \varepsilon$ for $\varepsilon \ge \frac{\lambda^2 d}{\mu^2}\sum_{i \in \cG} |\cB_i|^2$, where $\lambda \geq 0$ is any non-negative number and $\cB_i$ is the set of Byzantine peers neighboring with the $i$-th peer. In the complexity result, we use the notation $\overline{N}^2 = \sum_{i\in \cG}(|\cG_i|^2 + |\cB_i|^2)$, where $\cG_i$ is the set of good neighbors of the $i$-th peer. When $\lambda = 0$, the workers do not communicate at all. Moreover, \citet{peng2021byzantine} analyze the case of heterogeneous local functions, composite optimization problems and time-varying setup but in that case $\lambda$ is lower bounded by a strictly positive quantity depending on the heterogeneity level and minimal non-zero singular value of the node-edge incidence matrix, i.e., any predefined accuracy cannot be achieved.
        \item [{\color{red}(8)}] The results are derived for \algname{BTARD-SGD} (in the strongly convex case, for \algname{Restarted-BTARD-SGD}) under Assumptions~\ref{as:bounded_var}~and~\ref{as:quadratically_bounded_tails} in the case when the exact number of attacking Byzantine workers at iteration $k$ is known to each participant. See Theorems~\ref{thm:BTARD_SGD_known_byz_non_cvx},~\ref{thm:BTARD_SGD_known_byz_cvx},~and~\ref{thm:BTARD_SGD_known_byz_str_cvx}.
        \item [{\color{red}(9)}] The results are derived for \algname{BTARD-SGD} (in the strongly convex case, for \algname{Restarted-BTARD-SGD}) under Assumptions~\ref{as:bounded_var}~and~\ref{as:quadratically_bounded_tails}. See Theorems~\ref{thm:BTARD_SGD_unknown_byz_non_cvx},~\ref{thm:BTARD_SGD_unknown_byz_cvx},~and~\ref{thm:BTARD_SGD_unknown_byz_str_cvx}.
        \item [\color{red}(10)] The results are derived for \algname{BTARD-Clipped-SGD} (in the strongly convex case, for \algname{Restarted-BTARD-Clipped-SGD}) under Assumption~\ref{as:bounded_alpha_moment} without any additional assumptions on the tails of the distribution. Moreover, it is assumed that the exact number of attacking Byzantine workers at iteration $k$ is known to each participant. See Theorems~\ref{thm:BTARD_Clipped_SGD_known_byz_cvx}~and~\ref{thm:BTARD_Clipped_SGD_known_byz_str_cvx}. In the complexity results, we use the notation $\Lambda_1 = 1 + \frac{n\sqrt{\delta}}{m}$.
        \item [\color{red}(11)] The results are derived for \algname{BTARD-Clipped-SGD} (in the strongly convex case, for \algname{Restarted-BTARD-Clipped-SGD}) under Assumption~\ref{as:bounded_alpha_moment} without any additional assumptions on the tails of the distribution. See Theorems~\ref{thm:BTARD_Clipped_SGD_unknown_byz_cvx}~and~\ref{thm:BTARD_Clipped_SGD_unknown_byz_str_cvx}. In the complexity results, we use the notation $\Lambda_2 = 1 + \frac{n^2\delta}{m}$.
        }
    \end{tablenotes}
    \end{threeparttable}}
\end{table}

\newstuff{Finally, \citet{el2020genuinely} propose an algorithm based on the usage of multiple servers. The authors assume that both workers and servers can be Byzantines, which is a realistic scenario. However, their approach requires the workers to send their gradients to all servers at each iteration and receive parameters from all servers as well. This leads to a significant communication overhead in practice. Moreover, \citet{el2020genuinely} do not provide non-asymptotic convergence rates, making it problematic to provide an in-depth comparison with existing works and with our results as well. Therefore, it is unclear whether the usage of multiple servers speeds up training or it just leads to overhead in the communications and computations.}

In contrast, our results do benefit from the communications between workers. First of all, as one can see from Table~\ref{tab:comparison_of_rates}, the terms depending on the fraction $\delta$ of Byzantine peers in our complexity bounds for \algname{BTARD-SGD} and \algname{Restarted-BTARD-SGD} (the third terms) have better dependence on the target accuracy $\varepsilon$ than the corresponding terms in the complexity bounds from \textit{all} previous works (even from those relying on the existence of a PS). Moreover, for sufficiently small $\varepsilon$ these terms in our complexity results are smaller than the second terms, which correspond to the main term in the complexity of parallel \algname{SGD}. That is, \algname{BTARD-SGD}/\algname{Restarted-BTARD-SGD} applied to the problem with Byzantine peers has convergence guarantees that are not worse than the corresponding guarantees for parallel \algname{SGD} applied to the problem without any Byzantine workers. \newstuff{In such regimes, our theoretical convergence results outperform even ones derived for PS-based algorithms.}

We notice that Assumptions~\ref{as:bounded_var}~and~\ref{as:quadratically_bounded_tails} used in the analysis of \algname{BTARD-SGD}/\algname{Restarted-BTARD-SGD} are slightly stronger than uniformly bounded variance assumption used in \citep{wu2020federated, karimireddy2020learning, peng2021byzantine}. However, as we explain in Appendix~\ref{sec:on_assumptions}, our analysis allows to relax Assumptions~\ref{as:bounded_var} to uniformly bounded variance assumption, and Assumption~\ref{as:quadratically_bounded_tails} is reasonable for many practically important problems. Finally, we also propose and analyze \algname{BTARD-Clipped-SGD} and \algname{Restarted-BTARD-Clipped-SGD} under Assumption~\ref{as:bounded_alpha_moment} that may hold even in the case of \textit{unbounded} variance of the stochastic gradient. To the best of our knowledge, this is the first time in the literature on the Byzantine-tolerant optimization when the complexity results are obtained without assuming boundedness of the stochastic gradient's variance.

\subsection{Multi-party random number generators: additional details}
\label{appendix:mprng}

Many distributed systems may benefit from the multi-party random number generators (MPRNG) where a group of malicious peers would have little influence (bias) on the generator output. MPRNGs are usually based on multi-party coin tossing protocols, such as the protocol from \citet{blum1983coin}. As an example, MPRNG allows to choose a participant winning a lottery or choose a peer whose calculations are going to be validated by other peers to detect possible cheating.

\newstuff{While \citet{blum1983coin} formally introduces a protocol for one bit and two parties, its generalization to multiple bits and parties (as necessary for MPRNG) is trivial assuming the presence of the broadcast channel. This modification is widely known in literature, e.g., described in \citet{zhang2019efficient}. According to this generalization, peers should execute the following protocol to obtain $k$ random bits (see the intuitive scheme in Figure~\ref{fig:mprng}):
\begin{enumerate}[leftmargin=*]
    \item Each peer generates its own random string $x_i$ made of $k$ bits.
    \item Each peer broadcasts \textit{commitment} $h_i = h(i || x_i || s_i)$, where $||$ denotes concatenation, $h(x)$ is a common cryptographic hash function, $i$ is the peer's unique identifier (known by other peers), and $s_i$ is a large random string.
    \item Peers wait until all of them finish broadcasting the commitments. After that, no peer can alter its $x_i$ to influence the protocol output (otherwise, peers would notice that the new value $x_i'$ does not match the commitment).
    \item Each peer \textit{reveals} their random string by broadcasting its $x_i$ and $s_i$.
    \item Each peer verifies that all other peers revealed values $x_j$ and $s_j$ that match their commitments $h_j = h(j || x_j || s_j)$.
    \item If a peer detects that peer $j$ aborted the procedure or its commitment does not match its revealed values, it concludes that we cannot trust peer $j$. Since other peers read the same broadcast channel, all of them can make the same conclusion. In this case, the system repeats the protocol.
    \item If peers do not detect any mismatches, they calculate the protocol output $x = x_1 \oplus ... \oplus x_n$, where $\oplus$ denotes the bitwise XOR operation.
\end{enumerate}

In this protocol, the commitments include the peer identifier $i$ to protect from \textit{replay attacks} (when an attacker repeats someone else's message) and the large random string $s_i$ to resist \textit{dictionary attacks} (when an attacker reverses the hash function using a large dictionary of its values).}

While there are MPRNGs \citep{rabin1989verifiable} with a negligible bias for the case when more than a half parties are honest (assuming the presence of the broadcast channel), \citet{cleve1986limits} proves that it is impossible to reach the negligible bias for the case of dishonest majority, which may be reached in practice with the Sybil attacks.

However, we note that the bias in \citet{blum1983coin} \newstuff{(and its modification above)} appears only in the case when an attacker learns the result earlier than other peers and \newstuff{forces the protocol to be repeated}. If we are using MPRNG to choose a peer that to be checked for cheating, we may ban all peers that aborted the procedure and restart from scratch without them, therefore eliminating the bias problem.

\begin{figure}[tb]
    \vskip 0.1in
    \centering
    \includegraphics[width=0.9\textwidth]{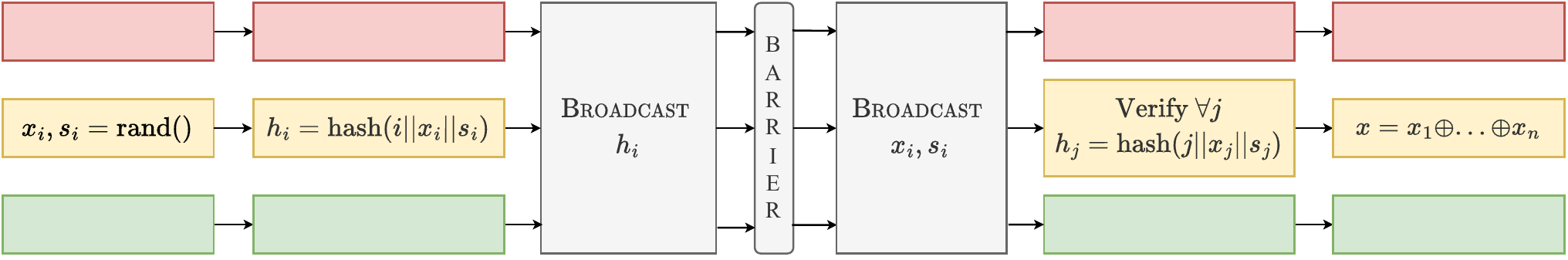}
    \caption{A scheme of MPRNG based on \newstuff{the generalization of} \citet{blum1983coin}. Here, $||$ denotes concatenation, \newstuff{$\oplus$ denotes bitwise XOR}, $h(x)$ is a common cryptographic hash function. The hashed values include the peer identifier $i$ to protect from the replay attacks and a large random string $s_i$ to resist the dictionary attacks.}
    \label{fig:mprng}
    \vskip -0.2in
\end{figure}

\section{Synchronization points and computation overhead of BTARD-SGD}\label{appendix:complexity_overhead}

\newstuff{\paragraph{Synchronization points.} An important aspect of BTARD performance is synchronization. The naive implementation of Algorithm~\ref{alg:btarsgd_outline} would have many global synchronization ``barriers'' per step: one for aggregating gradients, another for choosing a random direction $z$, yet another for electing validators, etc. These frequent synchronizations could undermine the practical training performance of BTARD in high-latency networks, such as when training over the Internet. 

Fortunately, it is possible to reduce the number of synchronizations by bundling them together. For instance, peers use a single MPRNG round for sampling $z$ and for electing validators. Furthermore, this MPRNG round and subsequent checks can be done in background, while a peer accumulates gradients for the next step. The only restriction is that this ``shared'' MPRNG round must be performed after all peers declare their checksums for that round.

With these optimizations, BTARD-SGD requires only two points of synchronization per round. The first one occurs right before gradient aggregation, and the second one is in a background task that performs verifications. Finally, there is a non-regular need for synchronization when one peer accuses another of being Byzantine. However, as we elaborated earlier, each accusation will result in at least one Byzantine being banned. Therefore, this additional cost will occur only a limited number of times over the training run.
}
\vspace{-4px}

\paragraph{Computation overhead.} In terms of computation, \algname{BTARD-SGD} introduces two main overheads: from validators and \algname{CenteredClip} respectively. As we have shown empirically in Section~\ref{sect:experiments} and Appendix~\ref{appendix:additional_experiments}, both \algname{BTARD-SGD} and \algname{BTARD-Clipped-SGD} can withstand even attacks with 7 out of 16 peers being Byzantine using only 1--2 validators randomly chosen from 16 peers. As such, the computation overhead for validation is no more than $1/8$ of the total compute.

As for the \algname{CenteredClip}, our algorithm executes the same amount of computation as the original \algname{CenteredClip}~\citep{karimireddy2020learning}, except that now the extra load is distributed evenly across all peers. We provide an empirical evaluation of such overhead in Appendix~\ref{appendix:compute_overhead_eval}.

\newstuff{Finally, we note that generating a shared vector $z$ from a scalar seed $r^t$ (as defined in Algorithm~\ref{alg:btarsgd_outline}) has a negligible cost and can be done with any standard pseudo-random number generator. For instance, generating $z$ for ALBERT-large (the setup from Section~\ref{sect:experiments_albert}) takes $30 \pm 1.2$ ms on the same T4 GPU that we use in our experiments.}

\section{Overview of attack vectors}\label{appendix:attack_types}

In Section~\ref{sect:method_analysis}, we have outlined the four types of Byzantine attacks that can affect \algname{BTARD-SGD}. Here, we analyze each of these types in detail and provide a list of attacks that fit these types.

\vspace{-4px}
\paragraph{Gradient attacks.} This attack vector encompasses all attacks where Byzantine peers replace their true gradients with something else, but otherwise act normally. With this attack, $b$ Byzantine peers can collectively shift the outputs of \algname{CenteredClip} by up to $\tau \cdot b / n$ in any chosen direction. However, since Byzantine peers will need to commit hash of their incorrect gradients, \textit{every honest validator} can accuse one of these peers with probability $b/n$ .

\vspace{-4px}
\paragraph{Aggregation attacks.} A similar, but opposite attack type can be attempted when a Byzantine peer performs gradient aggregation. Instead of honestly computing \algname{CenteredClip}, an attacker may modify the returned vector to incorporate the same kinds of changes as in gradient attacks (see above). This time, the maximum difference that can be applied through such attacks is larger, but it only affects $b/n$ of vector coordinates that are aggregated by Byzantines.

Done naively, such attacks can be detected and banned by the gradient checksum (see L15-17 in Algorithm~\ref{alg:btarsgd_outline}).
In order to ensure that the above check passes, Byzantines can misreport their $s^j_i$ in such a way that $\sum_i s^j_i {=} 0$. However, since actual $s^j_i$ depend only on $g_i^k$ and $\hat g^k$, these values can be verified by the chosen validators, and, in case of mismatch, reported via \textsc{Accuse}. We rigorously prove this in Appendix~\ref{appendix:detecting_protocol_violations}.

Furthermore, if an honest validator finds that a certain peer has broadcast incorrect $s^j_i$, the validator can simultaneously accuse the corresponding Byzantine aggregator $j$ that \textit{should have} notified about the incorrect $s^j_i$ (see L12-14 in Algorithm~\ref{alg:btarsgd_outline}).

\vspace{-4px}
\paragraph{Reputation abuse.} Since \algname{BTARD-SGD} provides means by which benign participants can ban Byzantine attackers, it is important to ensure that the same means cannot be exploited by Byzantine peers to eliminate benign ones or otherwise abuse the system. There are three potential attack vectors that fit this description:

\begin{itemize}[leftmargin=*]
    \item Falsely accusing a benign peer,
    \item Persistently calling the \textsc{Accuse} procedure to slow down training,
    \item Automatically approving gradients without actual validation,
\end{itemize}

In \algname{BTARD-SGD}, we protect against slander (issues 1. and 2.) by the design of \textsc{Accuse} protocol, by which a peer that initiates false allegations will itself be banned. As such, Byzantines can only invoke \textsc{Accuse} protocol a limited number of times before they are all permanently banned.

In turn, the attack vector (3.) is more effective: if one Byzantine was chosen as validator for another Byzantine, they can automatically report successful validation without negative consequences for either of them. However, since all validators are chosen through MPRNG, an attacker has no way of predicting whether its validator will be benign or Byzantine. Thus, any malicious activity will always have a chance of being caught by an honest validator.

\vspace{-4px}
\paragraph{Protocol violations.} Finally, a Byzantine attacker can deviate from the protocol prescribed by \algname{BTARD-SGD} in simpler ways ways, for instance:
\begin{enumerate}[leftmargin=*]
    \item Not committing the hash of its gradient when required by~\ref{alg:detailed_bytar},
    \item Not sending data to a particular peer when required (or sending data twice),
    \item Deliberately broadcasting a hash that mismatches the subsequently sent data,
    \item Sending metadata (e.g. gradient norm) that is inconsistent with previously sent gradient part,
    \item Sending $s_i$ that is inconsistent with previously sent gradient,
    \item Not validating when chosen as validator, validating when \textbf{not} chosen, or validating a different peer than was chosen by \algname{BTARD-SGD}.

\end{enumerate}

For protocol deviations that are visible to all benign participants, such as in (1.) or (6.), benign peers can ban the offender instantaneously. However, this is not the case for attacks such as (2.), where the deviation is only visible to one or few peers.

As described earlier in Section~\ref{sect:method_analysis}, we address this issue with a special procedure that allows any peer to ban any other peer at the cost of also being banned. Thus, if an attacker sends inconsistent gradients, norms or inner products to only one benign peer, that peer can still get the attacker banned even though it wouldn't be able to call \textsc{Accuse}.

Protecting from attacks 3, 4 and 5 from the above list also relies on this mutual elimination procedure. Specifically, if an attacker sends provably incorrect data to a benign peer, that peer will immediately trigger the mutual elimination procedure. The only exception to this rule is if one Byzantine peer sends incorrect data to another Byzantine peer: this behavior is neither punishable nor, in itself, harmful. In turn, the mutuality of this elimination procedure prevents potential misuse by Byzantines: if an attacker decides to ban someone through this procedure, that attacker will also be banned.
 
\section{Detailed algorithm description}\label{appendix:detailed_algs}
In this section, we provide more formal versions of the \algname{BTARD} (Alg.~\ref{alg:detailed_bytar}) and \algname{BTARD-SGD} (Alg.~\ref{alg:BTARD_SGD}) algorithms, as well as auxiliary subroutines and further details. We describe our approach in a bottom-up manner.

\subsection{Basic building blocks}

We begin with a glossary of basic functions used in the algorithms:
\begin{itemize}[leftmargin=*]
    \item $\textbf{broadcast } m$ --- broadcast the message $m$ to all other peers using GossipSub~\citep{vyzovitis2020gossipsub} and receive for the respective messages of other peers. $m$ should be signed by the sender's private key~\citep{rivest1978method} before sending. A receiver should ignore messages with an invalid signature and ban a peer in case of receiving two contradicting messages signed by it (e.g., two different hashes for the same iteration and the same stage of the algorithm).
    \item $\textsc{Split}(v, n)$ --- split vector $v$ of size $d$ into $n$ parts. The first $d \mod n$ parts are of size $\lceil d / n \rceil$ and the remaining parts have size $\lfloor d / n \rfloor$.
    \item $\textsc{Merge}(v_1, \dots, v_n)$ --- concatenate vectors $v_1, \dots, v_n$ into one.
    \item $\textsc{Ban}(\text{peer}_j)$ --- add peer $j$ to a local blocklist, ignore any subsequent messages from that peer, and continue training without it. Note that the honest peers do not need to explicitly coordinate on their decisions to ban someone, because these decisions are made using the broadcasted data only.
    \item $\textsc{CheckComputations}(j)$ or $\textsc{ValidatePeer}$ --- run $\textsc{ComputeGradients}(x^t, \xi_j^t)$ and compare against the $c_j, h_j^\text{*}, s_j^\text{*}$ broadcasted by that peer. If there is mismatch, $\textsc{Accuse}$.
\end{itemize}

\subsection{CenteredClip and verification of its results}\label{appendix:centered_clip}

An important building block of \algname{BTARD} is \algname{CenteredClip} -- a robust aggregation rule proposed in \citet{karimireddy2020learning}. Unlike a number of other aggregation rules as coordinate-wise median, Krum, geometric median, \algname{CenteredClip} is provably robust against Byzantine attacks (see Theorem~III from \citet{karimireddy2020learning} and Lemma~\ref{lem:centered_clip_guarantee_fixed}).

Let $\cG$ be the set of good peers, $\cB$ be the set of Byzantine workers, and, for simplicity, let $[n] = \cG \sqcup \cB$, $|\cB| = \delta n \le \delta_0 n < \nicefrac{n}{2}$. Assume that we have $n$ random vectors $x_1,\ldots,x_n$, such that $\forall i,j \in \cG$
\begin{equation*}
    \EE[x_i] = \EE[x_j] = x,\quad \EE[\|x_i - x_j\|^2] \le \sigma^2,
\end{equation*}
and for all $i \in \cB$ vectors $x_i$ can be arbitrary. \algname{CenteredClip} works as follows: it is an iterative procedure generating a sequence $\{v_l\}_{l\ge 0}$ satisfying
\begin{equation}
    v^{l+1} = v^l + \frac{1}{n}\sum\limits_{i=1}^n(x_i - v^l)\min\left\{1, \frac{\tau_l}{\|x_i - v^l\|}\right\}, \tag{CenteredClip}\label{eq:CenteredClip}
\end{equation}
where
\begin{equation}
    \tau_l = 4\sqrt{\frac{(1-\delta)\left(\nicefrac{B_l^2}{3} + \sigma^2\right)}{\sqrt{3}\delta}},\quad B_{l+1}^2 = 6.45\delta B_l^2 + 5\sigma^2.  \label{eq:tau_CenteredClip}
\end{equation}

\newstuff{Intuitively, \algname{CenteredClip} behaves like the mean for all points within the sphere of radius $\tau$ and like the median for ``outliers''. In turn, choosing different values of $\tau$ allows one to smoothly interpolate between the mean ($\tau \rightarrow \inf$) and the geometric median ($\tau\rightarrow 0$) aggregation rules.
}

The goal of this procedure is natural: find good enough approximation $\widehat{x}$ of $\overline{x} = \frac{1}{|\cG|}\sum_{i\in \cG} x_i$. \citet{karimireddy2020learning} show\footnote{In fact, \citet{karimireddy2020learning} derive this result for two-staged version of \algname{CenteredClip}. One can derive similar result for the original \algname{CenteredClip} under the assumption that for all $i,j\in\cG$ we have $\EE[\|x_i - x_j\|^4] \le \sigma^4$.} that, for $\delta \le 0.1$, the sequence $\{v_l\}_{l\ge 0}$ generated by \algname{CenteredClip} satisfies
\begin{equation}
    \EE[\|v^l - \overline{x}\|^2] \le (9.7\delta)^l3\EE[\|v_0 - \overline{x}\|^2] + 4000\delta\sigma^2.\label{eq:kndskcnkdnsvnsbdc}
\end{equation}
Moreover, \citet{karimireddy2020learning} prove that for all possible aggregation rules producing $\widehat{x}$ and given $\delta_0$, $\sigma$ there exists such set of vectors $x_1,\ldots,x_n$ and such a partition $[n] = \cG \sqcup \cB$ that
\begin{equation*}
    \EE[\|\widehat{x} - \overline{x}\|^2] = \Omega(\delta \sigma^2).
\end{equation*}
Therefore, \algname{CenteredClip} can be seen as an optimal aggregation rule neglecting numerical constants. The usage of \algname{CenteredClip} helps the good peer $i$ to produce a good enough approximation of the ideal average of the $i$-th parts of stochastic gradients among good peers in \algname{BTARD}.

Moreover, since $\delta \le 0.1$ we have that $6.45\delta \le 0.645$ implying that $B_l^2 \to B^2 \sim \sigma^2$ when $l \to \infty$, and $\tau_l \to \tau \sim \sqrt{\nicefrac{\sigma^2}{\delta}}$. These limits can be easily computed from \eqref{eq:tau_CenteredClip}. Next, for $l \to \infty$ \ref{eq:CenteredClip} converges to the solution of the following equation:
\begin{equation}
    \sum\limits_{i=1}^n (x_i - v)\min\left\{1, \frac{\tau}{\|x_i - v\|}\right\} = 0. \label{eq:CenteredClip_equation}
\end{equation}
In other words, \ref{eq:CenteredClip} for large enough $l$ approximates the fixed-point iteration process of solving \eqref{eq:CenteredClip_equation}. This property plays a key role in \textbf{Verification 2} of \algname{BTARD}. 


\subsection{Protocols for banning Byzantine peers}\label{appendix:accuse_and_eliminate}

\textsc{Accuse} and \textsc{Eliminate} are the two protocols by which peers ban Byzantine attackers from training. The \textsc{Accuse} protocol is only invoked if there the malicious activity of the target peer can be proven to others. We detail the exact mechanism in Algorithm~\ref{alg:detailed_accuse}, which is a formal version of Algorithm~\ref{alg:accuse} from Section~\ref{sect:method_aggregation}.

In contrast, \textsc{Eliminate} is a mechanism that allows any peer $i$ to ban any other peer $j$ from training without proof --- but at the cost of peer $i$ also being banned. We have described this protocol earlier as a countermeasure for protocol violations (see Appendix~\ref{appendix:attack_types}).

\newstuff{Both $\textsc{Accuse}(i, j)$ and $\textsc{Eliminate}(i, j)$ imply that peer $i$ uses the broadcast channel to declare its intent to ban peer $j$. Since the broadcast channel does not guarantee the order of receiving these messages, peers should collect all of them during a training step and process them at the end of the step in some specific order (e.g. sorted by $(\text{type}, \text{public\_key}_i, \text{public\_key}_j)$, where $\text{type} \in \{\textsc{Accuse}, \textsc{Eliminate}\}$ and $\textsc{Accuse} < \textsc{Eliminate}$). If processing one of the messages results in banning peer $p$, further messages involving $p$ are \textit{ignored} regardless of the $p$'s role.

This way, it is impossible for a Byzantine to eliminate more than one honest peer along with itself. Peers reach consensus since their decisions on banning someone are based solely on the messages from the broadcast channel (sorted in the common order) and the calculations with identical results.}

\begin{algorithm}[H]
  \caption{\algname{Accuse}(i, j), the formal version of Algorithm~\ref{alg:accuse}}
  \label{alg:detailed_accuse}
\let\oldtextsc\textsc
\renewcommand{\textsc}[1]{\oldtextsc{\scalefont{0.9}#1}}
\begin{algorithmic}[1]
  \Require accuser $i$, target $j$, peer count $n$, all values exchanged in Algorithm~\ref{alg:detailed_bytar}
  \State Recalculate $g_j^k = \textsc{ComputeGradients}(x^k, \xi_j^k)$
  \State Split $g_i$ into $n$ parts: $g_i = (g_i(1)^\top,\ldots, g_i(n)^\top)^\top$, $g_i(j) \in \R^{d_j}$ for all $j\in [n]$
  \State
  
  \For{$l = 1\dots n$}
  
  \If{$\text{hash}(g_j^k) \neq c_j^k \textbf{ or } \text{hash}(g_j^k(l)) \neq h_j^l$}
    \State $\textsc{Ban}(\text{peer}_j)$ \quad\text{// For gradient attack}
  \EndIf
  \State
  
  \State $\Delta_l^j {=} ( g_l(j) - \widehat g(j))\cdot\min\left\{1, \frac{\tau}{\|g_l(j) - \widehat g(j) \|_2}\right\}$

  \If{ $\|g_j(l) - \widehat g(l)\|_2 \neq \text{norm}_{jl} \textbf{ or } \langle \Delta_l^j, z_j \rangle \neq s_l^j \textbf{ or } \sum_{l=1}^n s_l^j \neq 0$}
    \State $\textsc{Ban}(\text{peer}_j)$ \quad\text{// For aggregation attack}
    \For{$o = 1, \dots, n$}
       \If{peer $o$ approved $\text{norm}_{jo}$ or $s_j^o$}
         \State $\textsc{Ban}(\text{peer}_o)$ \quad\text{// for covering up the $j$-th peer's aggregation attack}
       \EndIf
    \EndFor
  \EndIf
  \EndFor

\end{algorithmic}
\end{algorithm}

\subsection{ButterflyClip}

Algorithm~\ref{alg:detailed_butterflyclip} provides details on peer-to-peer communication conducted during a BTARD aggregation step. It was outlined earlier in Algorithm~\ref{alg:butterfly_clip} from Section~\ref{sect:method_aggregation}. For simplicity, we assume (here and below) that workers run each line in a synchronous manner (e.g. wait for all peers to broadcast $\text{hash}(g_i)$ before communicating the actual gradients). In practice, this restriction can be lifted in favor of asynchronous steps with several explicit synchronization barriers, but that would further complicate the pseudo-code.

\begin{algorithm}[h]
    \caption{\textsc{ButterflyClip} for peer $i$, the formal version of Algorithm~\ref{alg:butterfly_clip}}
   \label{alg:detailed_butterflyclip}
\begin{algorithmic}[1]
  \Require rank $i$, gradients $g_i \in \mathbb{R}^d$ 
   \State Split $g_i$ into $n$ parts: $g_i = (g_i(1)^\top,\ldots, g_i(n)^\top)^\top$, $g_i(j) \in \R^{d_j}$ for all $j\in [n]$
   \State
   \For{j = 1, \dots, n}
     \State \textbf{broadcast} $c_i(j) = \text{hash}(g_i(j))$
   \EndFor
   \State Send $g_i(j)$ to peer $j$ for all $j \neq i$
   \State Receive $g_j(i)$ from peer $j$ for all $j\neq i$
   \For{$j=1,\dots,n$}
     \If{$\text{hash}(g_j(i)) \neq c_j(i)$} 
       \State $\textsc{Eliminate}(i, j) \quad \text{// Signed with $\text{peer}_i$ private key}$
   \EndIf\EndFor
   \State
   \State $\widehat g(i) = \text{\algname{CenteredClip}}(g_1(i), g_2(i), \ldots, g_n(i))$
   \State
   \State \textbf{broadcast} $\widehat c(i) = \text{hash}(\widehat g(i))$
   \State Send $\widehat g(i)$ to each worker
   \State Receive $\widehat g(j)$ for all $j\neq i$ from other workers
   \For{$j=1,\dots,n$}
     \If{$\text{hash}(\widehat g(j)) \neq \widehat c(j)$} 
       \State $\textsc{Eliminate}(i, j) \quad \text{// Signed with $\text{peer}_i$ private key}$
   \EndIf\EndFor
   \State
   \State\Return $\textsc{Merge}(\widehat g(1), \dots, \widehat g(n))$
     


\end{algorithmic}
\end{algorithm}

\subsection{Byzantine-tolerant All-Reduce and its verification procedures}\label{appendix:detecting_protocol_violations}

Algorithm~\ref{alg:detailed_bytar} defines a single gradient aggregation step with additional verification procedures needed to reduce the negative influence of Byzantine peers. We explain the motivation for each of these procedures below.

\begin{algorithm}[h]
      \caption{\textbf{B}yzantine-\textbf{T}olerant \textbf{A}ll-\textbf{R}e\textbf{d}uce (\algname{BTARD})}
  \label{alg:detailed_bytar}
\begin{algorithmic}[1]
   \Require number of workers $n$, gradient vectors on the workers $g_1, g_2,\ldots,g_n \in\R^d$, $d> n$, $\Delta_{\max} > 0$ -- parameter for Verification 3
   \For{workers $i = 1,\ldots,n$ in parallel}
   \State $\widehat g = \textsc{ButterflyClip}(i, g_i)$ \quad // Described in Algorithm~\ref{alg:detailed_butterflyclip}
   \State
   \State \textbf{Send metadata for verification:}
   \State Generate $r$ via \textsc{MPRNG}
   \State $z = \textsc{GetRandomVector}(r)$
   \For{$j \in 1, ..., n$}
     \State $\Delta_i^j {=} ( g_i(j) - \widehat g(j))\cdot\min\left\{1, \frac{\tau}{\|g_i(j) - \widehat g(j) \|_2}\right\}$
     \State \textbf{broadcast} $s^j_i = \langle z[j], \Delta_i^j \rangle$
     \State \textbf{broadcast} $\text{norm}_{ij} = \|g_i(j) - \widehat g(j)\|_2$
     \For{$l=1,\dots,n$}
        \State $w_{lj} = \min\left\{1, \frac{\tau}{\text{norm}_{lj}}\right\} $
     \EndFor
   \EndFor
   \State
     
   \For{$j=1,\dots,n$}
     \State \textbf{Verification 1:} 
     \If{$\text{norm}_{ji} \neq \|g_j(i) - \widehat g(i)\|_2$}
       \State \textbf{broadcast} $\text{norm}_{ji}$ does not mach $c_j(i)$
       \quad // All recipients should run $\textsc{Accuse}(i, j)$ (Algorithm~\ref{alg:detailed_accuse})
     \EndIf
     \State
     \State \textbf{Verification 2:} 
          \State $\text{// Peer } i \text{ knows }\Delta_j^i\text{ from CenteredClip}$
     \If{$ s_j^i \neq \langle z^k[j], \Delta_j^i \rangle$}
        \State \textbf{broadcast} $s_i^j$ does not match $c_j(i)$
       \quad\quad // All recipients should run $\textsc{Accuse}(i, j)$ (Algorithm~\ref{alg:detailed_accuse})
     \EndIf
     \If{$\sum_i^n s_i^j \neq 0$}
         \State // Peer $j$ lied that all $s_{\cdot}^j$ are correct
         \State \textbf{broadcast} $\widehat g(j)$ is wrong
       \qquad\qquad\qquad // All recipients should run $\textsc{Accuse}(i, j)$ (Algorithm~\ref{alg:detailed_accuse})
     \EndIf
     \State
     \State \textbf{Verification 3:} 
     \State \textbf{broadcast} $\text{check}_{ij} = \left[ \|g_i(j) - \widehat g(j)\|_2 > \Delta_{\max} \right]$
     \If{ $\sum_l \text{check}_{lj} > \frac{n}{2}$}
       \State $\textsc{CheckAveraging}(j)$
     \EndIf
   \EndFor
   \State\Return $\widehat g$
\EndFor

\end{algorithmic}
\end{algorithm}

\paragraph{Verifications 1 and 2.} While good peers always run \algname{CenteredClip}, Byzantine peers can arbitrary violate the protocol meaning that they can send an arbitrary vector instead of sending the result of \algname{CenteredClip}. \textbf{Verification 1} and \textbf{2} are needed to prevent such violations and make it possible to identify them during the check of computations.

First of all, both verifications are split into $2$ rounds in order to let the aggregators of the corresponding part accuse those peers who send inconsistent norms or inner products. Next, in theory, we assume that all good peers find exactly the solution of \algname{CenteredClip} equaition \eqref{eq:CenteredClip_equation}. Therefore, it is possible to compute the weights from \eqref{eq:CenteredClip_equation} for each worker $i$ and each component $j$ knowing only a norm of the difference of corresponding vectors, i.e., one can compute $\min\{1,\frac{\tau}{\|g_i(j) - \widehat g(i)\|}\}$ by $\|g_i(j) - \widehat g(i)\|$. That is, if Byzantine peer $i$ sends $\texttt{norm}_{ij} \neq \|g_i(j) - \widehat{g}(j)\|$, it will be either revealed by $j$-th worker if $j \in \cG$ or it will be revealed with some probability during the subsequent checks of computations.

However, \textbf{Verification 1} is insufficient to prevent malicious behavior: at iteration $k$ Byzantine peer can send $g_i^k(j)$ such that $\|g_i^k(j) - \widehat{g}^k(j)\| = \|\nabla_{(j)}f(x^k,\xi_{i,k}) - \widehat{g}^k(j)\|$. If $j \in \cB$, then it can be the case that $i$-th worker commits the hash of $\nabla_{(j)}f(x^k,\xi_{i,k})$ and the check of gradient computation will not identify the violation of the protocol. That is why, \textbf{Verification 2} is required.

\textsc{GetRandomVector} is a function that generates a random unit vector $z$ in the space of model parameters. This vector is based on a random seed $r$ obtained from MPRNG.

The goal of \textbf{Verification 2}, is to check that \algname{CenteredClip} equation \eqref{eq:CenteredClip_equation} holds for the received vector. The idea is simple: if
\begin{equation}
    \sum\limits_{l=1}^n(g_l(i) - \widehat g(i))\min\left\{1,\frac{\tau}{\|g_l(i) - \widehat g(i)\|}\right\} = 0, \label{eq:CC_equation_i}
\end{equation}
then for any $z_i$ of an appropriate dimension
\begin{equation}
    \sum\limits_{l=1}^n\langle g_l(i) - \widehat g(i), z_i\rangle\min\left\{1,\frac{\tau}{\|g_l(i) - \widehat g(i)\|}\right\} = 0. \label{eq:CC_equation_i_z}
\end{equation}
Since $z_i$ in \algname{BTARD} is generated from the uniform distribution on the unit Euclidean sphere, we have
\begin{equation}
    \PP\left\{\eqref{eq:CC_equation_i} \text{ does not hold}\quad \&\quad \eqref{eq:CC_equation_i_z} \text{ holds}\right\} = 0. \label{eq:prob_of_wrong_inner_prod}
\end{equation}
However, it is impossible to verify \eqref{eq:CC_equation_i_z} explicitly for workers $j\neq i$. Therefore, in the algorithm, good workers check
\begin{equation}
    \sum\limits_{l=1}^n s_l^i = 0,\quad \text{where } s_l^i = \begin{cases}\langle g_l(i) - \widehat g(i), z_i\rangle\min\left\{1,\frac{\tau}{\|g_l(i) - \widehat g(i)\|}\right\}, & \text{if } l\in \cG,\\
    *,& \text{if } l\in\cB.\end{cases} \label{eq:sum_s_l_i_check}
\end{equation}
Unfortunately, Byzantine peers can send arbitrary $s_l^i$. This can lead to the situations when \eqref{eq:sum_s_l_i_check} holds while \eqref{eq:CC_equation_i_z} and, as a consequence, \eqref{eq:CC_equation_i} do not. Below, we rigorously show that all possible violations of the protocol that are not detected by verifications of \algname{BTARD} can be detected by the auxiliary check of computations with some probability.

\paragraph{Verification 3.} This is an additional verification that serves to limit the potential scope of \textit{aggregation attacks} (as described in Appendix~\ref{appendix:attack_types}). If the result of CenteredClip landed far from too many benign participants, BTARD will verify it by re-running the same aggregation across all peers. While this procedure is costly, our analysis proves that it is has a very small probability of triggering unless some of the peers perform aggregation attacks. In the latter case, verifying the gradient accumulation will root out such attacks and ban the corresponding peers.

\paragraph{Check of computations.} As we mentioned earlier, it is possible to violate the protocol without being detected by the verifications of \algname{BTARD}. Therefore, extra checks of computations are required. In particular, after each aggregation in \algname{BTARD-SGD} $2m$ workers are selected uniformly at random: $m$ workers check the computations at the previous step of other $m$ workers. That is, each Byzantine peer is checked at iteration $k$ with probability $\sim \nicefrac{m}{n}$ by some good worker (see the proof of Thm.~\ref{thm:BTARD_SGD_known_byz_non_cvx}). Consider an arbitrary Byzantine peer $j$ and all possible violations of the protocol at iteration $k$ that are not detected by verifications of \algname{BTARD}.

First of all, we notice that if $c_j(i) \neq \text{hash}(\nabla_{(i)} f(x^k, \xi_{j,k}))$, then it will be detected during the check of computations with some probability\footnote{Here and below, this means that the attack/violation will be detected iff a non-Byzantine peer is chosen to validate the perpetrator.}. Moreover, if $i\in \cG$, then $j$-th worker has to send $c_j(i) = \text{hash}(g_j(i))$ to avoid ban.

Therefore, the only non-trivial case is when $i \in \cB$ as well. In this case, $j$-th worker can commit $c_j(i) = \text{hash}(\nabla_{(i)} f(x^k, \xi_{j,k}))$ since it is meaningless for $i$-th worker to accuse $j$-th one. Since $\texttt{norm}_{ij}, s_i^j$ and $\widehat{g}(i)$ are known for all $i$ and $j$, $j$-th worker has to broadcast $\texttt{norm}_{ji} = \|\nabla_{(i)} f(x^k,\xi_{j,k}) - \widehat{g}(i)\|$ and $s_j^i = \langle \nabla_{(i)} f(x^k,\xi_{j,k}) - \widehat g(i), z_i\rangle\min\left\{1, \frac{\tau}{\|\nabla_{(i)} f(x^k,\xi_{j,k}) - \widehat g(i) \|}\right\}$ to avoid the ban during the check of the computations. Therefore, regardless to the choice $g_{j}(i)$, to pass \textbf{Verification 2} $i$-th worker should send such $\widehat{g}(i)$ that
\begin{equation*}
    \sum\limits_{l\in \cG \cup \{j\}}\langle \nabla_{(i)} f(x^k, \xi_{l,k}) - \widehat g(i), z_i\rangle\min\left\{1,\frac{\tau}{\|\nabla_{(i)} f(x^k, \xi_{l,k}) - \widehat g(i)\|}\right\} + \sum\limits_{l \in \cB\setminus\{j\}} s_l^i = 0.
\end{equation*}
In this case, the behavior of the $j$-th worker along $i$-th component is equivalent to the behavior of the good one. It means, that to avoid ban during the check of computations, each Byzantine worker $l$ should broadcast $\texttt{norm}_{li} = \|\nabla_{(i)} f(x^k,\xi_{l,k}) - \widehat{g}(i)\|$ and $s_l^i = \langle \nabla_{(i)} f(x^k,\xi_{l,k}) - \widehat g(i), z_i\rangle\min\left\{1, \frac{\tau}{\|\nabla_{(i)} f(x^k,\xi_{l,k}) - \widehat g(i) \|}\right\}$ implying that $i$-th worker should send such $\widehat{g}(i)$ that
\begin{equation*}
    \sum\limits_{l=1}^n\langle \nabla_{(i)} f(x^k, \xi_{l,k}) - \widehat g(i), z_i\rangle\min\left\{1,\frac{\tau}{\|\nabla_{(i)} f(x^k, \xi_{l,k}) - \widehat g(i)\|}\right\} = 0.
\end{equation*}
In view of \eqref{eq:prob_of_wrong_inner_prod}, it implies that
\begin{equation*}
    \widehat{g}(i) = \text{\algname{CenteredClip}}(\nabla_{(i)} f(x^k, \xi_{1,k}), \nabla_{(i)} f(x^k, \xi_{2,k}), \ldots, \nabla_{(i)} f(x^k, \xi_{2,k})),
\end{equation*}
i.e., there are no violations of the protocol along the $i$-th component.

\subsection{BTARD-SGD training loop}

Finally, Algorithm~\ref{alg:BTARD_SGD} combines all procedures above into a training loop for secure decentralized SGD. Algorithms~\ref{alg:detailed_bytar}--\ref{alg:BTARD_SGD} represent a formal version of Algorithm~\ref{alg:btarsgd_outline} from Section~\ref{sect:method_aggregation}.

\begin{algorithm}[h]
   \caption{\algname{BTARD-SGD}, the formal version of Algorithm~\ref{alg:btarsgd_outline}}
   \label{alg:BTARD_SGD}
\begin{algorithmic}[1]
   \Require $x^0$ -- starting point, $\gamma$ -- stepsize, $K$ -- number of iterations, $\{s_{i,k}\}_{i,k=0,0}^{n,K-1}$ -- seeds for batches computations
   \State $C_{0} = \text{Banned}_{-1} = \varnothing$
   \For{$k = 0,1,\ldots,K-1$}
    \State Worker $i$ computes $g_i^k = \begin{cases}\nabla f(x^k,\xi_{i,k}),& \text{if } i\in \cG_k\setminus \cC_k,\\ *,& \text{if } i\in \cB_k\setminus \cC_k,\end{cases}$, where $\xi_{i,k}$ is generated via seed $s_{i,k}$ available to every worker
    \State
    \State $\left( \widehat{g}^k, \text{public{\_}info}_k \right) = \textsc{BTARD}(g_{i_1^k}^k,g_{i_1^k}^k,\ldots,g_{i_{a_k}^k}^k)$, where $\{i_1^k,\ldots, i_{a_k}^k\} = (\cG_k\cup \cB_k)\setminus \cC_k$
    \State \text{// \textsc{BTARD} is described in Algorithm~\ref{alg:detailed_bytar}}
    \State
    \State Choose $2m$ workers $c_1^{k+1},\ldots,c_m^{k+1}, u_1^{k+1},\ldots, u_m^{k+1}$ uniformly at random without replacement, $\cC_{k+1} = \{c_1^{k+1},\ldots,c_m^{k+1}\}$, $\cU_{k+1} = \{u_1^{k+1},\ldots, u_m^{k+1}\}$
    \State $\text{Banned}_k = \text{\algname{CheckComputations}}(\cC_{k+1}, \cU_{k+1}, \text{public{\_}info}_k)$
    \State $x^{k+1} = \text{proj}_Q(x^k - \gamma \widehat g^k) := \argmin_{x\in Q}\|x - (x^k - \gamma \widehat g^k)\|$
    \State $\cG_{k+1} = \cG_k\setminus \text{Banned}_{k-1}$
    \State $\cB_{k+1} = \cB_k\setminus \text{Banned}_{k-1}$
   \EndFor
\end{algorithmic}
\end{algorithm}

\clearpage

\section{Convergence analysis: missing proofs and extra details}\label{appendix:extra_analysis}

\subsection{Preliminaries}
For convenience, we provide the classical definitions and facts on smooth and strongly convex functions below.
\begin{definition}[$L$-smoothness]\label{def:L_smoothness}
    We say that function $f: Q \to \R$, $Q\subseteq \R^d$ is $L$-smooth if it is differentiable and
    \begin{equation}
        \forall x,y\in Q\quad \|\nabla f(x) - \nabla f(y)\| \le L\|x-y\|.\label{eq:L_smoothness}
    \end{equation}
\end{definition}

One can show \citep{nesterov2003introductory} that $L$-smoothness implies
\begin{equation}
    \forall x,y\in Q\quad f(y) \le f(x) + \langle \nabla f(x), y-x \rangle + \frac{L}{2}\|y-x\|^2, \label{eq:L_smoothness_cor}
\end{equation}
\begin{equation}
    \forall x\in Q\quad \|\nabla f(x)\|^2 \le 2L\left(f(x) -f_*\right), \label{eq:L_smoothness_cor_2}
\end{equation}
where $f_*$ is a uniform lower bound for $f$.

\begin{definition}[$\mu$-strong convexity]\label{def:mu_strong_convexity}
    Differentiable function $f: Q \to \R$, $Q\subseteq \R^d$ is called $\mu$-strongly convex if
    \begin{equation}
        \forall x,y\in Q\quad f(y) \ge f(x) + \langle\nabla f(x), y-x\rangle + \frac{\mu}{2}\|y-x\|^2.\label{eq:mu_strong_convexity}
    \end{equation}
\end{definition}

\subsection{Impossibility of Byzantine-tolerant learning in heterogeneous case}\label{sec:imposssibility_hetero}

Several papers on Byzantine-tolerant optimization consider non-homogeneous setup, when good workers have different local functions \citep{wu2020federated, he2020byzantine}. Formally, it means that instead of solving
\begin{equation}
    \min_{x\in Q\subseteq \R^d}\left\{f(x) := \EE_{\xi \sim \cD}\left[f(x,\xi)\right]\right\}, \label{eq:main_problem_appendix}
\end{equation}
where good peers sample stochastic gradients from the full dataset (i.e., they can sample $\xi$ from $\cD$), the following problem is considered:
\begin{equation}
    \min_{x\in Q\subseteq \R^d}\left\{f(x) := \frac{1}{|\cG|}\sum\limits_{i\in \cG}f_i(x)\right\}, \label{eq:hetero_problem_appendix}
\end{equation}
where $f_i(x) = \EE_{\xi_i \sim \cD_i}\left[f(x,\xi_i)\right]$ and there exists $\zeta \ge 0$ such that for all $x\in Q$
\begin{equation}
    \frac{1}{|\cG|}\sum\limits_{i\in \cG}\|\nabla f_i(x) - \nabla f(x)\|^2 \le \zeta^2. \label{eq:bounded_heterogeneity}
\end{equation}

However, under $\zeta$-bounded heterogeneity assumption \eqref{eq:bounded_heterogeneity} it is impossible in general to solve \eqref{eq:hetero_problem_appendix} with any predefined accuracy in the presence of Byzantine peers \citep{he2020byzantine}. Moreover, this is true even when trusted Parameter-Server is available.

\begin{theorem}[Theorem III from \citep{he2020byzantine}]
    For any optimization method $\texttt{Alg}$ there exist $n$ functions $f_1(x),\ldots, f_n(x)$ such that at least $(1-\delta)n$ of them are good (corresponding workers belong to $\cG$), $1$-smooth, $\mu$-strongly convex and satisfy \eqref{eq:bounded_heterogeneity} such that the output $\widehat x$ of $\texttt{Alg}$ given the access to these $n$ functions has an error at least
    \begin{equation}
        \EE\left[f(\widehat{x}) - \min\limits_{x\in \R^d} f(x)\right] \ge \Omega\left(\frac{\delta\zeta^2}{\mu}\right)\quad \text{and} \quad \EE\left[\|\nabla f(\widehat{x})\|^2\right] \geq \Omega\left(\delta\zeta^2\right), \label{eq:hetero_case_lower_bound}
    \end{equation}
    where the expectation is taken w.r.t.\ the randomness of $\texttt{Alg}$.
\end{theorem}

The intuition behind this negative result is as following: since the only assumption on the similarity of ``good'' functions is \eqref{eq:bounded_heterogeneity}, Byzantine peers can shift the gradients by a vector with a norm $\sim \zeta$ without being detected. In this case, it is impossible to distinguish good peers from Byzantines but the solution of \eqref{eq:hetero_problem_appendix} depends on which workers are good and which are bad. Therefore, the best one can hope for is the convergence to some neighborhood of the solution.

The lower bounds from \eqref{eq:hetero_case_lower_bound} are proportional to $\delta \zeta^2$ and cannot be made arbitrary small for given $\delta$ and $\zeta^2$. It means that the convergence to any predefined accuracy of the solution is impossible to achieve when local loss functions are $\zeta$-heterogeneous. In this sense, Byzantine-tolerant learning is impossible in the heterogeneous case. Moreover, in some practical applications (e.g., in Federated Learning \citep{FedLearningOriginal}), $\zeta$ from \eqref{eq:bounded_heterogeneity} can be large implying that one cannot achieve reasonable accuracy of the solution when $\delta$ is not too small (e.g., $\delta \ge 0.01$). Finally, strong convexity parameter $\mu$ is typically much smaller than $1$ (assuming that the smoothness parameter is $1$). In these cases, $\nicefrac{\delta\zeta^2}{\mu}$ can be too large and, as a result, all methods are not converging at all.



\subsection{Convergence guarantees for BTARD-SGD}\label{sec:BTARD_SGD_appendix}

\subsubsection{On Assumptions~\ref{as:bounded_var}~and~\ref{as:quadratically_bounded_tails}}\label{sec:on_assumptions}

First of all, Assumption~\ref{as:bounded_var} holds whenever standard uniformly bounded variance (UBV) assumption is satisfied. Indeed, if $\EE_{\xi\sim \cD}[\|\nabla f(x,\xi) - \nabla f(x)\|^2] \le \widehat{\sigma}^2$, then $\EE_{\xi\sim \cD}[(\nabla_i f(x,\xi) - \nabla_i f(x))^2] \le \widehat{\sigma}^2$ for all $i = 1,\ldots, d$, since $\|\nabla f(x,\xi) - \nabla f(x)\|^2 = \sum_{i=1}^d (\nabla_i f(x,\xi) - \nabla_i f(x))^2$. This implies that Assumption~\ref{as:bounded_var} holds with $\sigma^2 \leq d\widehat{\sigma}^2$. However, $\sigma^2$ can be significantly smaller than $d\widehat{\sigma}^2$. For example, if the noise in stochastic gradients is isotropic, e.g., Gaussian, then
\begin{equation*}
    \EE_{\xi\sim \cD}[(\nabla_1 f(x,\xi) - \nabla_1 f(x))^2] = \ldots = \EE_{\xi\sim \cD}[(\nabla_d f(x,\xi) - \nabla_d f(x))^2],
\end{equation*}
implying that
\begin{equation*}
    \EE_{\xi\sim \cD}[(\nabla_i f(x,\xi) - \nabla_i f(x))^2] = \frac{1}{d}\EE_{\xi\sim \cD}[(\nabla f(x,\xi) - \nabla f(x))^2] \le \frac{\widehat{\sigma}^2}{d}
\end{equation*}
for all $i=1,\ldots, d$. Therefore, in this case, Asssumption~\ref{as:bounded_var} holds with $\sigma^2 = \widehat{\sigma}^2$.

Next, it is possible to relax Assumption~\ref{as:bounded_var} to the classical UBV assumption. Indeed, in our proofs, we use Assumption~\ref{as:bounded_var} to bound the variance in the blocks of the stochastic gradients, where the blocks of components are chosen for workers to execute \algname{BTARD}. If these blocks are chosen uniformly at random, i.e., the vector is split into several parts of the given sizes uniformly at random, then it is enough to have
\begin{equation}
    \EE\left[\|\nabla f_{[S]}(x,\xi) - \nabla_{[S]} f(x)\|^2\right] \leq \frac{s \sigma^2}{d} \label{eq:sbhcjdscbsjdbjs}
\end{equation}
for a random subset $S$ of $\{1,\ldots,d\}$ such that $|S| = s$, where expectation is taken w.r.t.\ $\xi$ and $S$. To derive inequality \eqref{eq:sbhcjdscbsjdbjs} from UBV assumption  $\EE_{\xi\sim \cD}[\|\nabla f(x,\xi) - \nabla f(x)\|^2] \le \widehat{\sigma}^2$ we use tower property of the expectation:
\begin{eqnarray*}
   \EE\left[\|\nabla f_{[S]}(x,\xi) - \nabla_{[S]} f(x)\|^2\right] &=& \EE_{\xi\sim \cD}\left[\EE_{S}\left[\|\nabla f_{[S]}(x,\xi) - \nabla_{[S]} f(x)\|^2\right]\right]\\
   &=& \EE_{\xi\sim \cD}\left[\sum\limits_{i=1}^d\PP\{i\in S\} (\nabla_i f(x,\xi) - \nabla_i f(x))^2\right]\\
   &=& \frac{s}{d}\EE_{\xi\sim \cD}\left[\sum\limits_{i=1}^d (\nabla_i f(x,\xi) - \nabla_i f(x))^2\right]\\
   &=& \frac{s}{d}\EE_{\xi\sim \cD}\left[\|\nabla f(x,\xi) - \nabla f(x)\|^2\right] \le \frac{s\widehat{\sigma}^2}{d},
\end{eqnarray*}
i.e., \eqref{eq:sbhcjdscbsjdbjs} holds for $\sigma^2 = \widehat{\sigma}^2$.

Finally, as we show in Lemmas~\ref{lem:BTARD_for_bad_peer_known_num_byz}~and~\ref{lem:quality_of_agg_unknown_num_byz}, under As.~\ref{as:quadratically_bounded_tails} \textbf{Verification 3} at \algname{BTARD} leads to extra checking of computations with probability $\sim \nicefrac{1}{n}$ at each iteration when all workers honestly follow the protocol and under a proper choice of $\Delta_{\max}$. Therefore, extra computations either appear due to malicious manipulations of Byzantine peers, and lead eventually to the ban for the Byzantine peers who deviate from the protocol, or, when all workers honestly follow the protocol, only once per $n$ iterations on average. There are a number of important machine learning tasks, such as training ResNet-50 on Imagenet \citep{zhang2020why} and many others image classification problems, where the noise in the stochastic gradient has much ``lighter'' (sub-Gaussian) tails. That is, As.~\ref{as:quadratically_bounded_tails} is reasonable for a large class of practically important problems. Moreover, in Appendix~\ref{appendix:BTARD_Clipped_SGD_appendix}, we also provide an analysis of \algname{BTARD-Clipped-SGD} and \algname{Restarted-BTARD-Clipped-SGD} without any assumptions on the tails of the stochastic gradients distribution.

\subsubsection{Quality of the aggregation}
The quality of the aggregation at each iteration of \algname{BTARD-SGD} significantly affects the rate of the method. That is, properties of $\widetilde{g}^k$ are highly important for the convergence of \algname{BTARD-SGD}. This aggregator is obtained via \algname{BTARD} that requires to know a tight estimate of the total number of Byzantine workers violating the protocol at iteration $k$ -- clipping parameter $\tau$ depends on this quantity. Therefore, it is natural to start with relatively simple setup when the number of Byzantine workers violating the protocol is known at each iteration.

Before we formulate the first result we introduce some useful notations. Let $n_k$ be the total number of peers at iteration $k$, $b_k$ be the total number of Byzantine peers at iteration $k$, $\widehat{b}^k$ be the total number of Byzantine peers violating the protocol at iteration $k$, and $\delta_k = \frac{b_k}{n_k}$, $\widehat{\delta}_k = \frac{\widehat{b}_k}{n_k-m}$. In view of new notation, we start with the ideal situation when $\widehat{b}_k$ is known for each worker at each iteration $k$. First of all, it is needed to to estimate the quality of the aggregation for good workers. 

\begin{lemma}[Theorem~IV from \citet{karimireddy2020learning}]\label{lem:centered_clip_guarantee_fixed}
    Let  As.~\ref{as:bounded_var} hold, $\delta \le 0.1(n-m)$, and $i\in \cG_k\setminus \cC_k$. Assume that $\widehat{b}_k$ is known for each worker at iteration $k$ and $\delta = \widehat{\delta}_k$ is used to compute clipping parameter $\tau_l$ for \ref{eq:CenteredClip}. If the total number of iterations $T$ of \ref{eq:CenteredClip} satisfies $T \ge \log_{9.7\delta}\frac{\delta\sigma^2}{3\EE[\|v^0 - \overline{g}^k\|^2]}$, then
    \begin{equation}
        \EE\left[\|\widehat{g}^k(i) - \overline{g}^k(i)\|^2\mid x^k\right] \le 4001\widehat{\delta}_k\frac{\sigma^2}{n_k-m}, \label{eq:BTARD_for_good_peer_known_num_byz}
    \end{equation}
    where $\overline{g}^k(i) = \frac{1}{|\cG_k\setminus \cC_k|}\sum\limits_{j\in \cG_k\setminus \cC_k} g_j^k(i)$.
\end{lemma}
\begin{proof}
    The proof follows directly from \eqref{eq:kndskcnkdnsvnsbdc}.
\end{proof}

Unlike the good peers, Byzantine workers can cooperate and shift the result of \algname{CenteredClip} in the components they aggregate without being revealed at \textbf{Verification 2} of \algname{BTARD}. However, they cannot produce an arbitrary large shifts due to \textbf{Verification 3}. The next lemma estimates the maximal possible magnitude of a shift together with probability of triggering \algname{CheckAveraging} at iteration $k$ for at least one worker.

\begin{lemma}\label{lem:BTARD_for_bad_peer_known_num_byz}
    Let As.~\ref{as:bounded_var}~and~\ref{as:quadratically_bounded_tails} hold, $b \le 0.1(n-m)$, and $i\in \cB_k\setminus \cC_k$. Assume that $\widehat{b}_k$ is known for each worker at iteration $k$, $\Delta_{\max}^k = \frac{(1 + \sqrt{3})\sqrt{2}\sigma}{\sqrt{n_k-m}}$ and $\delta = \widehat{\delta}_k$ is used to compute clipping parameter $\tau_l$ for \ref{eq:CenteredClip}. If the total number of iterations $T$ of \ref{eq:CenteredClip} satisfies $T \ge \log_{9.7\delta}\frac{\delta\sigma^2}{3\EE[\|v^0 - \overline{g}^k\|^2]}$ and $\text{\algname{CheckAveraging}}(i)$ is not triggered, then
     \begin{equation}
        \EE\left[\|\widehat{g}^k(i) - \overline{g}^k(i)\|^2\mid x^k\right] \le \frac{4\left((1 + \sqrt{3})^2 + 3\right)\sigma^2}{n_k-m}, \label{eq:BTARD_for_bad_peer_known_num_byz}
    \end{equation}
    where $\overline{g}^k(i) = \frac{1}{|\cG_k\setminus \cC_k|}\sum\limits_{j\in \cG_k\setminus \cC_k} g_j^k(i)$. Moreover, if $\widehat{b}_k = 0$ and $n_k - m \ge 170$, then $\widehat{g}^k(i) = \overline{g}^k(i)$ and 
    \begin{equation}
        \PP\left\{\text{\algname{CheckAveraging} } \text{is triggered for } \ge 1 \text{ peer}\mid x^k\right\} \le \frac{149}{49(n_k-m)}. \label{eq:probability_of_verif_3_triggering}
    \end{equation}
\end{lemma}
\begin{proof}
    If $\text{\algname{CheckAveraging}}(i)$ is not triggered at iteration $k$, then for $r_k \ge \frac{n_k-m}{2}$ good workers $i_1,i_2,\ldots,i_{r_k} \in \cG_k\setminus \cC_k$ we have $\|g_{i_{j}}^k(i) - \widehat{g}^k(i)\| \le \Delta_{\max}^k$ (otherwise there exist at least $\frac{n_k - m}{2}$ good workers reporting that the norm is larger than $\Delta_{\max}^k$). Due to this and $|\cG_k \setminus \cC_k| \leq n_k - m$ we have
    \begin{eqnarray}
        \EE\left[\left\|\widehat{g}^k(i) - \overline{g}^k(i)\right\|^2\mid x^k\right] &\le& 2\EE\left[\left\|\widehat{g}^k(i) - \frac{1}{r_k}\sum\limits_{j=1}^{r_k}g_{i_j}^k(i)\right\|^2\mid x^k\right] + 2\EE\left[\left\|\frac{1}{r_k}\sum\limits_{j=1}^{r_k}g_{i_j}^k(i) - \overline{g}^k(i)\right\|^2\mid x^k\right]\notag \\
        &\le& 2\EE\left[\frac{1}{r_k}\sum\limits_{j=1}^{r_k}\|\widehat{g}^k(i) - g_{i_j}^k(i)\|^2\right] + 2\EE\left[\frac{1}{r_k}\sum\limits_{j=1}^{r_k}\left\|g_{i_j}^k(i) - \overline{g}^k(i)\right\|^2\mid x^k\right] \notag\\
        &\le& 2(\Delta_{\max}^k)^2 + 4\EE\left[\frac{1}{r_k}\sum\limits_{j=1}^{r_k}\left\|g_{i_j}^k(i) - \nabla_{(i)} f(x^k)\right\|^2\mid x^k\right]  + 4\EE\left[\left\|\overline{g}^k(i) - \nabla_{(i)} f(x^k)\right\|^2\mid x^k\right] \notag\\
        &\le& \frac{4\left((1 + \sqrt{3})^2 + 1\right)\sigma^2}{n_k-m} + 4\EE\left[\frac{2}{n_k - m}\sum\limits_{j\in \cG_k\setminus \cC_k}\left\|g_{j}^k(i) - \nabla_{(i)} f(x^k)\right\|^2\mid x^k\right] \notag\\
        &\leq& \frac{4\left((1 + \sqrt{3})^2 + 1\right)\sigma^2}{n_k-m} + 8\EE\left[\frac{1}{|\cG_k\setminus \cC_k|}\sum\limits_{j\in \cG_k\setminus \cC_k}\left\|g_{j}^k(i) - \nabla_{(i)} f(x^k)\right\|^2\mid x^k\right] \notag\\
        &\leq& \frac{4\left((1 + \sqrt{3})^2 + 3\right)\sigma^2}{n_k-m}, \label{eq:nvjsdfniviirnei}
    \end{eqnarray}
    where we use $\nabla_{(i)} f(x^k) = \EE[g_{i_j}^k\mid x^k]$. Finally, let us estimate the probability of triggering \algname{CheckAveraging} when all workers follow the protocol. In this case, $\widehat{g}^k(i) = \overline{g}^k(i)$. Next, due to As.~\ref{as:quadratically_bounded_tails} and $b \le 0.1(n-m)$ we have
    \begin{equation*}
        \PP\left\{\|\overline{g}^k(i) - \nabla_{(i)} f(x^k)\|_2 > \sqrt{\frac{\sigma^2}{n_k-m}}\mid x^k\right\} \le \frac{1}{|\cG_k\setminus \cC_k|^2} \le \frac{100}{49(n_k-m)^2}
    \end{equation*}
    and for all $j \in \cG_k\setminus \cC_k$
    \begin{equation*}
        \PP\left\{\|g_j^k(i) - \nabla_{(i)} f(x^k)\|_2 > \sqrt{\frac{3\sigma^2}{n_k-m}}\mid x^k\right\} \le \frac{1}{9}.
    \end{equation*}
    Consider the independent random variables $\eta_j$, $j\in \cG_k\setminus\cC_k$, where
    \begin{equation*}
        \eta_j = \begin{cases}1,& \text{if }  \|g_j^k(i) - \nabla_{(i)} f(x^k)\|_2 \le \sqrt{\frac{3\sigma^2}{n_k-m}},\\ 0,& \text{otherwise,}\end{cases}
    \end{equation*}
    where $x^k$ is fixed. Then, $\eta_j$ is a Bernoulli random variable with parameter of ``success'' $q \ge \nicefrac{8}{9}$. Applying Hoeffding's inequality we get that
    \begin{eqnarray*}
        \PP\left\{\sum\limits_{j\in \cG_k\setminus \cC_k} \eta_j \le \frac{n_k-m}{2}\mid x^k\right\} &\le& \exp\left(-2(n_k-m)\left(q - \frac{n_k-m}{2|\cG_k\setminus\cC_k|}\right)^2\right)\\
        &\le& \exp\left(-2(n_k-m)\left(\frac{8}{9} - \frac{n-m}{1.4(n-m)}\right)^2\right)\\
        &=& \exp\left(-\frac{242(n_k-m)}{3969}\right).
    \end{eqnarray*}
    Since for all $j\in \cG_k\setminus \cC_k$ we have $\|\overline{g}^k(i) - g_j^k(i)\|_2 \le \|\overline{g}^k(i) - \nabla_{(i)}f(x^k)\|_2 + \|\nabla_{(i)}f(x^k) - g_j^k(i)\|_2$ the obtained bounds imply that \algname{CheckAveraging} is triggered for at least one worker at iteration $k$ with probability not greater than
    \begin{equation*}
        \frac{100}{49(n_k-m)} + (n_k-m)\exp\left(-\frac{242(n_k-m)}{3969}\right) \le \frac{149}{49(n_k-m)},
    \end{equation*}
    where we use that $\exp\left(-\frac{242x}{3969}\right) \le \frac{1}{x^2}$ for all $x \ge 170$.
\end{proof}

We notice that Byzantine peers can trigger \algname{CheckAveraging} by violating the protocol. However, each Byzantine is checked at iteration $k$ with probability $p \sim \nicefrac{m}{n}$ (see Thm.~\ref{thm:BTARD_SGD_known_byz_non_cvx}). Therefore, Byzantine workers can trigger only $\cO\left(\nicefrac{bn}{m}\right)$ extra rounds of communications and computations on average via triggering \algname{CheckAveraging}. In contrast, when there are no Byzantine workers or all workers follow the protocol \algname{CheckAveraging} is triggered only once per $\cO(n-m)$ iterations that is a negligible communication an computation overhead when $n$ is large.

Combining two previous lemmas we get the following result.
\begin{lemma}\label{lem:quality_of_agg_known_num_byz}
    Let As.~\ref{as:bounded_var} hold and $b \le 0.1(n-m)$. Assume that $\widehat{b}_k$ is known for each worker at iteration $k$, $\Delta_{\max}^k = \frac{(1 + \sqrt{3})\sqrt{2}\sigma}{\sqrt{n_k-m}}$ and $\delta = \widehat{\delta}_k$ is used to compute clipping parameter $\tau_l$ for \ref{eq:CenteredClip}. If the total number of iterations $T$ of \ref{eq:CenteredClip} satisfies $T \ge \log_{9.7\delta}\frac{\delta\sigma^2}{3\EE[\|v^0 - \overline{g}^k\|^2]}$ and $\text{\algname{CheckAveraging}}$ is not triggered for any worker, then
     \begin{equation}
        \EE\left[\|\widehat{g}^k - \overline{g}^k\|^2\mid x^k\right] \le C\widehat{\delta}_k\sigma^2, \label{eq:BTARD_overall_peer_known_num_byz}
    \end{equation}
    \begin{equation}
        \EE\left[\|\widehat{g}^k\|^2\mid x^k\right] \le  2C\widehat{\delta}_k\sigma^2 + 2\|\nabla f(x^k)\|^2 + \frac{2\sigma^2}{n-2b-m}, \label{eq:BTARD_second_moment_known_num_byz}
    \end{equation}
    where $\overline{g}^k = \frac{1}{|\cG_k\setminus \cC_k|}\sum\limits_{j\in \cG_k\setminus \cC_k} g_j^k$ and $C = 4001 + 4\left((1 + \sqrt{3})^2 + 3\right)$. 
\end{lemma}
\begin{proof}
    We have
    \begin{eqnarray*}
        \EE\left[\|\widehat{g}^k - \overline{g}^k\|^2\mid x^k\right] &=& \sum\limits_{i\in \cG_k\setminus \cC_k}\EE\left[\|\widehat{g}^k(i) - \overline{g}^k(i)\|^2\mid x^k\right] + \sum\limits_{i\in \cB_k\setminus \cC_k}\EE\left[\|\widehat{g}^k(i) - \overline{g}^k(i)\|^2\mid x^k\right]\\
        &\overset{\eqref{eq:BTARD_for_good_peer_known_num_byz},\eqref{eq:BTARD_for_bad_peer_known_num_byz}}{\le}& (1-\widehat{\delta}_k)(n_k-m)\cdot4001\widehat{\delta}_k\frac{\sigma^2}{n_k-m} + \widehat{\delta}_k(n_k - m)\cdot \frac{4\left((1 + \sqrt{3})^2 + 3\right)\sigma^2}{n_k-m}\\
        &=& C\widehat{\delta}_k\sigma^2.
    \end{eqnarray*}
    Next, using the independence of $g_j^k$ for $j\in \cG_k \setminus \cC_k$ and fixed $x^k$ we derive
    \begin{eqnarray*}
        \EE\left[\|\widehat{g}^k\|^2\mid x^k\right] &\le& 2\EE\left[\|\widehat{g}^k - \overline{g}^k\|^2\mid x^k\right] + 2\EE\left[\|\overline{g}^k\|^2\mid x^k\right]\\
        &\overset{\eqref{eq:BTARD_overall_peer_known_num_byz}}{\le}& 2C\widehat{\delta}_k\sigma^2 + 2\|\nabla f(x^k)\|^2 + 2\EE\left[\|\overline{g}^k - \nabla f(x^k)\|^2\mid x^k\right]\\
        &\le& 2C\widehat{\delta}_k\sigma^2 + 2\|\nabla f(x^k)\|^2 + \frac{2\sigma^2}{|\cG_k\setminus \cC_k|}\\
        &\le&  2C\widehat{\delta}_k\sigma^2 + 2\|\nabla f(x^k)\|^2 + \frac{2\sigma^2}{n-2b-m}.
    \end{eqnarray*}
\end{proof}
In view of the definition of $(\delta,c)$-robust aggregator from \citet{karimireddy2020learning}, the result of \algname{BTARD} at iteration $k$ is $(\widehat{\delta_k},C)$-robust. However, we derive this property under assumption that $\widehat{b}_k$ is known to all workers at each iteration $k$, which is impractical.

When $\widehat{b}_k$ is unknown the situation changes dramatically: in general, good peers can only know some upper bound for the fraction of Byzantine peers at iteration $k$. Unfortunately, if used without bans, this is not enough to converge to any accuracy of the solution since \algname{BTARD-SGD} is a permutation-invariant algorithm in terms of \citet{karimireddy2020learning}. Therefore, in this case, we always use \algname{CenteredClip} with $\tau_l = \infty$ for all $l\ge 0$, i.e., good peers compute an exact average. In this settings, even $1$ Byzantine worker can significantly shift the average in all parts of the vector. The next lemma quantifies the negative effect of Byzantine workers in this case.

\begin{lemma}\label{lem:quality_of_agg_unknown_num_byz}
    Let As.~\ref{as:bounded_var}~and~\ref{as:quadratically_bounded_tails} hold, $b \le 0.1(n-m)$, $m \le \nicefrac{(n-2b)}{2}$. Assume that $\Delta_{\max}^k = \frac{(1 + \sqrt{3})\sqrt{2}\sigma}{\sqrt{n_k-m}}$ and $\delta = 0$ is used to compute clipping parameter $\tau_l$ for \ref{eq:CenteredClip}. If $\text{\algname{CheckAveraging}}$ is not triggered for any worker, then
     \begin{equation}
        \EE\left[\|\widehat{g}^k - \overline{g}^k\|^2\mid x^k\right] \le C\sigma^2\indicator_{k, v}, \label{eq:BTARD_overall_peer_unknown_num_byz}
    \end{equation}
    \begin{equation}
        \EE\left[\|\widehat{g}^k\|^2\mid x^k\right] \le  2C\sigma^2\indicator_{k, v} + 2\|\nabla f(x^k)\|^2 + \frac{2\sigma^2}{n-2b-m}, \label{eq:BTARD_second_moment_unknown_num_byz}
    \end{equation}
    where $\overline{g}^k = \frac{1}{|\cG_k\setminus \cC_k|}\sum\limits_{j\in \cG_k\setminus \cC_k} g_j^k$, $C = 4\left((1 + \sqrt{3})^2 + 3\right)$, and $\indicator_{k, v}$ is an indicator function of the event that at least $1$ Byzantine peer violates the protocol at iteration $k$. Moreover, if $\widehat{b}_k = 0$ and $n_k - m \ge 170$, then $\widehat{g}^k(i) = \overline{g}^k(i)$ and 
    \begin{equation}
        \PP\left\{\text{\algname{CheckAveraging} } \text{is triggered for } \ge 1 \text{ peer}\mid x^k\right\} \le \frac{149}{49(n_k-m)}. \label{eq:probability_of_verif_3_triggering_1}
    \end{equation}
\end{lemma}
\begin{proof}
    If $\text{\algname{CheckAveraging}}$ is not triggered for any worker, then similarly to \eqref{eq:nvjsdfniviirnei} one can derive $\EE\left[\|\widehat{g}^k(i) - \overline{g}^k(i)\|^2\mid x^k \right] \le \frac{C\sigma^2}{n_k - m}\indicator_{k, v}$ for all $i\in (\cG_k \cup \cB_k) \setminus \cC_k$, since when $\indicator_{k, v} = 0$ we have  $\widehat{g}^k(i) = \overline{g}^k(i)$ for all $i\in (\cG_k \cup \cB_k) \setminus \cC_k$. This implies
    \begin{eqnarray*}
        \EE\left[\|\widehat{g}^k - \overline{g}^k\|^2\mid x^k\right] &=& \sum\limits_{i\in (\cG_k \cup \cB_k) \setminus \cC_k}\EE\left[\|\widehat{g}^k(i) - \overline{g}^k(i)\|^2\mid x^k\right]\\
        &\le& (n_k - m)\cdot \frac{4\left((1 + \sqrt{3})^2 + 3\right)\sigma^2}{n_k-m}\indicator_{k, v} \le C\sigma^2\indicator_{k, v}.
    \end{eqnarray*}
     Next, using the independence of $g_j^k$ for $j\in \cG_k \setminus \cC_k$ and fixed $x^k$ we derive
    \begin{eqnarray*}
        \EE\left[\|\widehat{g}^k\|^2\mid x^k\right] &\le& 2\EE\left[\|\widehat{g}^k - \overline{g}^k\|^2\mid x^k\right] + 2\EE\left[\|\overline{g}^k\|^2\mid x^k\right]\\
        &\overset{\eqref{eq:BTARD_overall_peer_unknown_num_byz}}{\le}& 2C\sigma^2\indicator_{k, v} + 2\|\nabla f(x^k)\|^2 + 2\EE\left[\|\overline{g}^k - \nabla f(x^k)\|^2\mid x^k\right]\\
        &\le& 2C\sigma^2\indicator_{k, v} + 2\|\nabla f(x^k)\|^2 + \frac{2\sigma^2}{|\cG_k\setminus \cC_k|}\\
        &\le&  2C\sigma^2\indicator_{k, v} + 2\|\nabla f(x^k)\|^2 + \frac{2\sigma^2}{n-2b-m}.
    \end{eqnarray*}
    The proof of the final part of the lemma is identical to the proof of the same result from Lemma~\ref{lem:BTARD_for_bad_peer_known_num_byz}.
\end{proof}

\subsubsection{Non-convex case}\label{appendix:non_convex_analysis_BTARD_SGD}
In this section, we provide the complete statements and the full proofs of the convergence results for \algname{BTARD-SGD} when the objective function $f$ is smooth, but can be non-convex. We start with the case when the number of attacking Byzantine workers is known at each iteration.
\begin{theorem}\label{thm:BTARD_SGD_known_byz_non_cvx}
    Let As.~\ref{as:bounded_var}~and~As.~\ref{as:quadratically_bounded_tails} hold, $Q = \R^d$, and $f$ be $L$-smooth (see Def.~\ref{def:L_smoothness}) and uniformly lower bounded by $f_*$. Moreover, assume that $b \le 0.1(n-m)$, $m \le \nicefrac{(n-2b)}{2}$, and the exact number of attacking Byzantine peers is known to all good peers at each iteration. Next, assume that 
    \begin{equation}
        \gamma = \min\left\{\frac{1}{4L}, \sqrt{\frac{\Delta_0n}{L\sigma^2 K}}\right\},\quad \Delta_{\max}^k = \frac{(1 + \sqrt{3})\sqrt{2}\sigma}{\sqrt{n_k-m}}, \label{eq:choice_of_parameters_BTARD_SGD_non_cvx}
    \end{equation}
    where $\Delta_0 = f(x^0) - f_*$ and $\Delta_{\max}^k$ is the parameter for verification $3$ at iteration $k$ of \algname{BTARD-SGD}. Then, we have $\EE[\|\nabla f(\overline{x}^{K})\|^2] \le \varepsilon^2$ after $K$ iterations of \algname{BTARD-SGD}, where
    \begin{equation}
        K = \cO\left(\frac{L\Delta_0}{\varepsilon^2} + \frac{L\Delta_0\sigma^2}{n\varepsilon^4} + \frac{n\delta\sigma^2 }{m\varepsilon^2}\right) \label{eq:BTARD_SGD_known_byz_non_cvx}
    \end{equation}
    and $\overline{x}^K$ is picked uniformly at random from $\{x^0,x^1,\ldots,x^{K-1}\}$.
\end{theorem}
\begin{proof}
    From $L$-smoothness of $f$ we have
    \begin{eqnarray*}
        f(x^{k+1}) &\overset{\eqref{eq:L_smoothness_cor}}{\le}& f(x^k) + \langle \nabla f(x^k), x^{k+1} - x^k \rangle + \frac{L}{2}\|x^{k+1} - x^k\|^2\\
        &=& f(x^k) - \gamma \langle\nabla f(x^k), \widehat{g}^k \rangle + \frac{L\gamma^2}{2}\|\widehat{g}^k\|^2.
    \end{eqnarray*}
    Taking the conditional expectation $\EE[\cdot\mid x^k]$ from the both sides of the previous inequality we obtain
    \begin{eqnarray*}
        \EE\left[f(x^{k+1})\mid x^k\right] &\le& f(x^k) -\gamma\|\nabla f(x^k)\|^2 - \gamma\left\langle \nabla f(x^k), \EE\left[\widehat{g}^k - \overline{g}^k\mid x^k\right] \right\rangle\\
        &&\quad + \frac{L\gamma^2}{2}\EE\left[\|\widehat{g}^k\|^2\mid x^k\right]\\
        &\overset{\eqref{eq:BTARD_second_moment_known_num_byz}}{\le}& f(x^k) - \frac{\gamma}{2}\|\nabla f(x^k)\|^2 + \frac{\gamma}{2}\left\|\EE\left[\widehat{g}^k - \overline{g}^k\mid x^k\right]\right\|^2\\
        &&+ CL\gamma^2 \widehat{\delta}_k\sigma^2 + L\gamma^2\|\nabla f(x^k)\| + \frac{L\gamma^2\sigma^2}{n-2b-m}\\
        &\le& f(x^k) - \frac{\gamma}{2}\left(1 - 2L\gamma\right)\|\nabla f(x^k)\|^2 +\frac{\gamma}{2}\EE\left[\|\widehat{g}^k - \overline{g}^k\|^2\mid x^k\right]\\
        &&+ CL\gamma^2 \widehat{\delta}_k\sigma^2 + \frac{L\gamma^2\sigma^2}{n-2b-m}.
    \end{eqnarray*}
    Since $\gamma \le \frac{1}{4L}$ we continue our derivations as
    \begin{eqnarray*}
        \EE\left[f(x^{k+1})\mid x^k\right] &\le& f(x^k) - \frac{\gamma}{4}\|\nabla f(x^k)\|^2 + \gamma C\sigma^2(1+L\gamma)\widehat{\delta}_k + \frac{L\gamma^2\sigma^2}{n-2b-m}\\
        &\le& f(x^k) - \frac{\gamma}{4}\|\nabla f(x^k)\|^2 + 2\gamma C\sigma^2\widehat{\delta}_k + \frac{L\gamma^2\sigma^2}{n-2b-m}.
    \end{eqnarray*}
    Taking the full expectation from the both sides of the obtained inequality and summing up the results for $k=0,1,\ldots,K-1$ we get
    \begin{eqnarray*}
        \frac{1}{K}\sum\limits_{k=0}^{K-1}\EE\left[\|\nabla f(x^k)\|^2\right] &\le& \frac{4}{\gamma K}\sum\limits_{k=0}^{K-1}\EE\left[f(x^k) - f(x^{k+1})\right] + \frac{8C\sigma^2}{K} \EE\left[\sum\limits_{k=0}^{K-1}\widehat{\delta}_k\right] + \frac{4L\gamma\sigma^2}{n-2b-m}\\
        &=& \frac{4\left(f(x^0)-\EE[f(x^{K})]\right)}{\gamma K} + \frac{8C\sigma^2}{K} \EE\left[\sum\limits_{k=0}^{K-1}\frac{\widehat{b}_k}{n_k-m}\right] + \frac{4L\gamma\sigma^2}{n-2b-m} \\
        &\le& \frac{4(f(x^0)-f_*)}{\gamma K} + \frac{8C\sigma^2}{K(n-2b-m)} \EE\left[\sum\limits_{k=0}^{K-1}\widehat{b}_k\right] + \frac{4L\gamma\sigma^2}{n-2b-m}.
    \end{eqnarray*}
    If a Byzantine peer deviates from the protocol at iteration $k$, it will be detected with some probability $p_k$ during the next iteration. One can lower bound this probability as
    \begin{equation*}
        p_k \ge m\cdot \frac{|\cG_k|}{n_k}\cdot \frac{1}{n_k} = \frac{m(1-\delta_k)}{n_k} \ge \frac{m}{n}.
    \end{equation*}
    Therefore, each individual Byzantine worker can violate the protocol no more than $\nicefrac{1}{p}$ times on average implying that
    \begin{eqnarray*}
        \frac{1}{K}\sum\limits_{k=0}^{K-1}\EE\left[\|\nabla f(x^k)\|^2\right] &\le& \frac{4(f(x^0)-f_*)}{\gamma K} + \frac{8Cnb\sigma^2}{Km(n-2b-m)} + \frac{4L\gamma\sigma^2}{n-2b-m}\\
        &\le& \frac{4(f(x^0)-f_*)}{\gamma K} + \frac{16Cnb\sigma^2}{Km(n-2b)} + \frac{8L\gamma\sigma^2}{n-2b}\\
        &\le& \frac{4(f(x^0)-f_*)}{\gamma K} + \frac{160Cn\delta\sigma^2}{7Km} + \frac{80L\gamma\sigma^2}{7n}.
    \end{eqnarray*}
    Since  $\overline{x}^K$ is picked uniformly at random from $\{x^0,x^1,\ldots,x^{K-1}\}$ we have
    \begin{equation*}
        \EE\left[\|\nabla f(\overline{x}^K)\|^2\right] \le  \frac{4(f(x^0)-f_*)}{\gamma K} + \frac{160Cn\delta\sigma^2}{7Km} + \frac{80L\gamma\sigma^2}{7n}.
    \end{equation*}
    Using the stepsize rule
    \begin{equation*}
        \gamma = \min\left\{\frac{1}{4L}, \sqrt{\frac{\Delta_0n}{L\sigma^2 K}}\right\}
    \end{equation*}
    we derive
    \begin{equation*}
        \EE\left[\|\nabla f(\overline{x}^K)\|^2\right] = \cO\left(\frac{L\Delta_0}{K} + \frac{\sqrt{L\Delta_0}\sigma}{\sqrt{n K}} + \frac{n\delta\sigma^2}{mK}\right)
    \end{equation*}
    meaning that after
    \begin{equation*}
        K = \cO\left(\frac{L\Delta_0}{\varepsilon^2} + \frac{L\Delta_0\sigma^2}{n\varepsilon^4} + \frac{n\delta\sigma^2 }{m\varepsilon^2}\right)
    \end{equation*}
    iterations \algname{BTARD-SGD} guarantees $\EE\left[\|\nabla f(\overline{x}^K)\|^2\right] \le \varepsilon^2$.
\end{proof}

In the main part of the paper, we notice that the rate of \algname{BTARD-SGD} in the presence of bad workers is asymptotically the same as for \algname{SGD} without Byzantine peers when $\varepsilon$ is sufficiently small\footnote{This is true for convex and strongly convex cases as well.}. This phenomenon has a clear intuition. When the target accuracy $\varepsilon$ is small, the stepsize $\gamma$ is also needed to be small enough. However, as we show in Lemmas~\ref{lem:quality_of_agg_known_num_byz}~and~\ref{lem:quality_of_agg_unknown_num_byz}, Byzantine workers can produce only a bounded shift independent of the stepsize. Moreover, they can violate the protocol at only $\sim \nicefrac{n}{m}$ iterations on average. Therefore, the overall impact of Byzantine workers on the convergence of \algname{BTARD-SGD} decreases when the stepsize $\gamma$ decreases.

Next, we derive the result without assuming that $\widehat b^k$ is known to all peers at each iteration.
\begin{theorem}\label{thm:BTARD_SGD_unknown_byz_non_cvx}
    Let As.~\ref{as:bounded_var}~and~\ref{as:quadratically_bounded_tails} hold, $Q = \R^d$, and $f$ be $L$-smooth (see Def.~\ref{def:L_smoothness}) and uniformly lower bounded by $f_*$. Moreover, assume that $b \le 0.1(n-m)$, $m \le \nicefrac{(n-2b)}{2}$, and $\delta = 0$ is used to compute clipping parameter $\tau_l$ for \ref{eq:CenteredClip}. Next, assume that 
    \begin{equation}
        \gamma = \min\left\{\frac{1}{4L}, \sqrt{\frac{\Delta_0n}{L\sigma^2 K}}\right\},\quad \Delta_{\max}^k = \frac{(1 + \sqrt{3})\sqrt{2}\sigma}{\sqrt{n_k-m}}, \label{eq:choice_of_parameters_BTARD_SGD_non_cvx_unknown_byz}
    \end{equation}
    where $\Delta_0 = f(x^0) - f_*$ and $\Delta_{\max}^k$ is the parameter for verification $3$ at iteration $k$ of \algname{BTARD-SGD}. Then, we have $\EE[\|\nabla f(\overline{x}^{K})\|^2] \le \varepsilon^2$ after $K$ iterations of \algname{BTARD-SGD}, where
    \begin{equation}
        K = \cO\left(\frac{L\Delta_0}{\varepsilon^2} + \frac{L\Delta_0\sigma^2}{n\varepsilon^4} + \frac{nb\sigma^2 }{m\varepsilon^2}\right) \label{eq:BTARD_SGD_unknown_byz_non_cvx}
    \end{equation}
    and $\overline{x}^K$ is picked uniformly at random from $\{x^0,x^1,\ldots,x^{K-1}\}$.
\end{theorem}
\begin{proof}
    The proof is almost identical to the proof of Theorem~\ref{thm:BTARD_SGD_known_byz_non_cvx}. Following the same steps and using \eqref{eq:BTARD_overall_peer_unknown_num_byz} and \eqref{eq:BTARD_second_moment_unknown_num_byz} instead of \eqref{eq:BTARD_overall_peer_known_num_byz} and \eqref{eq:BTARD_second_moment_known_num_byz} respectively we obtain the same sequence of inequalities up to the following change: instead of $\widehat{\delta}_k$ we should use $\indicator_{k,v}$. Therefore, we have
    \begin{eqnarray*}
        \frac{1}{K}\sum\limits_{k=0}^{K-1}\EE\left[\|\nabla f(x^k)\|^2\right] &\le& \frac{4(f(x^0)-f_*)}{\gamma K} + \frac{8C\sigma^2}{K} \EE\left[\sum\limits_{k=0}^{K-1}\indicator_{k,v}\right] + \frac{4L\gamma\sigma^2}{n-2b-m}.
    \end{eqnarray*}
     If a Byzantine peer deviates from the protocol at iteration $k$, it will be detected with some probability $p_k$ during the next iteration. One can lower bound this probability as
    \begin{equation*}
        p_k \ge m\cdot \frac{|\cG_k|}{n_k}\cdot \frac{1}{n_k} = \frac{m(1-\delta_k)}{n_k} \ge \frac{m}{n}.
    \end{equation*}
    That is, each individual Byzantine worker can violate the protocol no more than $\nicefrac{1}{p}$ times on average. However, even one Byzantine peer can create a shift of the order $\Delta_{\max}^k$ at each part of the resulting vector. Therefore, all Byzantine peers can violate the protocol no more than $\nicefrac{b}{p}$ times on average implying that
    \begin{eqnarray*}
        \frac{1}{K}\sum\limits_{k=0}^{K-1}\EE\left[\|\nabla f(x^k)\|^2\right] &\le& \frac{4(f(x^0)-f_*)}{\gamma K} + \frac{8Cnb\sigma^2}{Km} + \frac{4L\gamma\sigma^2}{n-2b-m}\\
        &\le& \frac{4(f(x^0)-f_*)}{\gamma K} + \frac{8Cnb\sigma^2}{Km} + \frac{8L\gamma\sigma^2}{n-2b}\\
        &\le& \frac{4(f(x^0)-f_*)}{\gamma K} + \frac{8Cnb\sigma^2}{Km} + \frac{80L\gamma\sigma^2}{7n}.
    \end{eqnarray*}
    Since  $\overline{x}^K$ is picked uniformly at random from $\{x^0,x^1,\ldots,x^{K-1}\}$ we have
    \begin{equation*}
        \EE\left[\|\nabla f(\overline{x}^K)\|^2\right] \le  \frac{4(f(x^0)-f_*)}{\gamma K} + \frac{8Cnb\sigma^2}{Km} + \frac{80L\gamma\sigma^2}{7n}.
    \end{equation*}
    Using the stepsize rule
    \begin{equation*}
        \gamma = \min\left\{\frac{1}{4L}, \sqrt{\frac{\Delta_0n}{L\sigma^2 K}}\right\}
    \end{equation*}
    we derive
    \begin{equation*}
        \EE\left[\|\nabla f(\overline{x}^K)\|^2\right] = \cO\left(\frac{L\Delta_0}{K} + \frac{\sqrt{L\Delta_0}\sigma}{\sqrt{n K}} + \frac{nb\sigma^2}{mK}\right)
    \end{equation*}
    meaning that after
    \begin{equation*}
        K = \cO\left(\frac{L\Delta_0}{\varepsilon^2} + \frac{L\Delta_0\sigma^2}{n\varepsilon^4} + \frac{nb\sigma^2 }{m\varepsilon^2}\right)
    \end{equation*}
    iterations \algname{BTARD-SGD} guarantees $\EE\left[\|\nabla f(\overline{x}^K)\|^2\right] \le \varepsilon^2$.
\end{proof}

As we notice in the main part of the paper, the third term of the obtained complexity result is significantly worse than in \eqref{eq:BTARD_SGD_known_byz_non_cvx}: it is proportional to $b$ instead of $\delta = \nicefrac{b}{n}$. However, \eqref{eq:BTARD_SGD_unknown_byz_non_cvx} is derived without assuming that $\widehat{b}_k$ is known for all workers at each iteration. Moreover, as in \eqref{eq:BTARD_SGD_known_byz_non_cvx}, the third term in \eqref{eq:BTARD_SGD_unknown_byz_non_cvx} has better dependence on $\varepsilon$ than the second term implying that for small enough $\varepsilon$ the rate of \algname{BTARD-SGD} in the presence of bad workers without assuming that $\widehat{b}_k$ is known at each iteration is asymptotically the same as for \algname{SGD} without Byzantine peers\footnote{This is true for convex and strongly convex cases as well.}.

\subsubsection{Convex case}\label{appendix:convex_analysis_BTARD_SGD}
In this section, we provide the complete statements and the full proofs of the convergence results for \algname{BTARD-SGD} when the objective function $f$ is smooth and convex. We start with the case when the number of attacking Byzantine workers is known at each iteration.
\begin{theorem}\label{thm:BTARD_SGD_known_byz_cvx}
    Let As.~\ref{as:bounded_var}~and~\ref{as:quadratically_bounded_tails} hold, $Q = \R^d$, $f$ be $L$-smooth (see Def.~\ref{def:L_smoothness}), convex, and $x^*$ be some optimum of $f$. Moreover, assume that $b \le 0.1(n-m)$, $m \le \nicefrac{(n-2b)}{2}$, and the exact number of attacking Byzantine peers is known to all good peers at each iteration. Next, assume that 
    \begin{equation}
        \gamma = \min\left\{\frac{1}{4L}, \sqrt{\frac{7nR_0^2}{120\sigma^2 K}}, \sqrt{\frac{m^2 R_0^2}{1440 C\sigma^2 n^2\delta}}\right\},\quad \Delta_{\max}^k = \frac{(1 + \sqrt{3})\sqrt{2}\sigma}{\sqrt{n_k-m}}, \label{eq:choice_of_parameters_BTARD_SGD_cvx}
    \end{equation}
    where $R_0 \ge \|x^0 - x^*\|$ and $\Delta_{\max}^k$ is the parameter for verification $3$ at iteration $k$ of \algname{BTARD-SGD}. Then, we have $\EE[f(\overline{x}^{K}) -f(x^*)]\le \varepsilon$ after $K$ iterations of \algname{BTARD-SGD}, where
    \begin{equation}
        K = \cO\left(\frac{LR_0^2}{\varepsilon} + \frac{\sigma^2 R_0^2}{n\varepsilon^2} + \frac{n\sqrt{\delta}\sigma R_0 }{m\varepsilon}\right) \label{eq:BTARD_SGD_known_byz_cvx}
    \end{equation}
    and $\overline{x}^K = \frac{1}{K}\sum_{k=0}^{K-1}$.
\end{theorem}
\begin{proof}
    Lemma~\ref{lem:quality_of_agg_known_num_byz} implies
    \begin{eqnarray*}
        \EE\left[\|x^{k+1} - x^*\|^2\mid x^k\right] &=& \EE\left[\|x^{k} - x^* - \gamma \widehat{g}^k\|^2\mid x^k\right]\\
        &=& \|x^k-x^*\|^2 -2\gamma\EE\left[\langle x^k-x^*, \widehat{g}^k \rangle\mid x^k\right] + \gamma^2\EE\left[\|\widehat{g}^k\|^2\mid x^k\right]\\
        &\overset{\eqref{eq:BTARD_second_moment_known_num_byz}}{\le}& \|x^k - x^*\|^2 -2\gamma\langle x^k-x^*, \nabla f(x^k) \rangle + 2\gamma^2\|\nabla f(x^k)\|^2\\
        &&\quad - 2\gamma\EE\left[\langle x^k-x^*, \widehat{g}^k - \overline{g}^k \rangle\mid x^k\right] + 2\gamma^2 C\widehat{\delta}_k\sigma^2 + \frac{2\gamma^2\sigma^2}{n-2b-m}.
    \end{eqnarray*}
    Next, we use convexity (see \eqref{eq:mu_strong_convexity} with $\mu = 0$) and $L$-smoothness of $f$:
    \begin{eqnarray*}
        \EE\left[\|x^{k+1} - x^*\|^2\mid x^k\right] &\overset{\eqref{eq:L_smoothness_cor_2},\eqref{eq:mu_strong_convexity}}{\le}& \|x^k - x^*\|^2 -2\gamma\left(1 - 2L\gamma\right)\left(f(x^k) - f(x^*)\right)\\
        &&\quad - 2\gamma\EE\left[\langle x^k-x^*, \widehat{g}^k - \overline{g}^k \rangle\mid x^k\right] + 2\gamma^2 C\sigma^2\frac{\widehat b_k}{n_k-m} + \frac{2\gamma^2\sigma^2}{n-2b-m}.
    \end{eqnarray*}
    To estimate the inner product in the right-hand side we apply Cauchy-Schwarz inequality:
    \begin{eqnarray*}
        - 2\gamma\EE\left[\langle x^k-x^*, \widehat{g}^k - \overline{g}^k \rangle\mid x^k\right] &\le& 2\gamma\|x^k - x^*\|\EE\left[\|\widehat{g}^k - \overline{g}^k\|\mid x^k\right]\\
        &\le& 2\gamma\|x^k - x^*\| \sqrt{\EE\left[\|\widehat{g}^k - \overline{g}^k\|^2\mid x^k\right]}\\
        &\overset{\eqref{eq:BTARD_overall_peer_known_num_byz}}{\le}& 2\gamma\sqrt{C}\sigma \|x^k - x^*\|\sqrt{\widehat{\delta}_k} \le \frac{2\gamma\sqrt{C}\sigma}{\sqrt{n_k-m}} \|x^k - x^*\|\sqrt{\widehat{b}_k}\\
        &\le& \frac{2\gamma\sqrt{C}\sigma}{\sqrt{n-2b-m}} \|x^k - x^*\|\sqrt{\widehat{b}_k}.
    \end{eqnarray*}
    Putting all together and using $b \le 0.1(n-m)$, $m \le \nicefrac{(n-2b)}{2}$, $\gamma \le \nicefrac{1}{4L}$, $n_k-m \ge n-2b-m$, we obtain
    \begin{eqnarray*}
        \EE\left[\|x^{k+1} - x^*\|^2\mid x^k\right] &\le& \|x^k - x^*\|^2 -\gamma\left(f(x^k) - f(x^*)\right)\\
        &&\quad + \frac{4\gamma\sqrt{5C}\sigma}{\sqrt{n}} \|x^k - x^*\|\sqrt{\widehat{b}_k} + \frac{40\gamma^2 C\sigma^2}{7n}\widehat b_k + \frac{40\gamma^2\sigma^2}{7n}.
    \end{eqnarray*}
    Taking the full expectation from the both sides of the above inequality and summing up the results for $k = 0,1,\ldots,K-1$ we derive
    \begin{eqnarray*}
        \frac{\gamma}{K}\sum\limits_{k=0}^{K-1}\EE[f(x^k)-f(x^*)] &\le& \frac{1}{ K}\sum\limits_{k=0}^{K-1}\left(\EE\left[\|x^k - x^*\|^2\right] - \EE\left[\|x^{k+1} - x^*\|^2\right]\right) + \frac{40\gamma^2\sigma^2}{7n}\\
        &&\quad +\frac{4\gamma\sqrt{5C}\sigma}{\sqrt{n}K} \sum\limits_{k=0}^{K-1}\EE\left[\|x^k - x^*\|\sqrt{\widehat{b}_k}\right] + \frac{40\gamma^2 C\sigma^2}{7nK}\sum\limits_{k=0}^{K-1}\EE[\widehat b_k]\\
        &\le& \frac{\|x^0-x^*\|^2 - \EE[\|x^{K}-x^*\|^2]}{K} + \frac{40\gamma^2\sigma^2}{7n}\\
        &&\quad +\frac{4\gamma\sqrt{5C}\sigma}{\sqrt{n}K} \sum\limits_{k=0}^{K-1}\sqrt{\EE\left[\|x^k - x^*\|^2\right]\EE[\widehat{b}_k]} + \frac{40\gamma^2 C\sigma^2}{7nK}\sum\limits_{k=0}^{K-1}\EE[\widehat b_k].
    \end{eqnarray*}
    From Jensen's inequality we have $f(\overline{x}^K) \le \frac{1}{K}\sum_{k=0}^{K-1}f(x^k)$, where $\overline{x}^K = \frac{1}{K}\sum_{k=0}^{K-1} x^k$. Using this and new notation $R_k = \|x^k - x^*\|$, $k> 0$, $R_0 \ge \|x^0 - x^*\|$ we get
    \begin{eqnarray}
        0 \le \gamma\EE\left[f(\overline{x}^K) - f(x^*)\right] &\le& \frac{R_0^2 - \EE[R_K^2]}{K} + \frac{40\gamma^2\sigma^2}{7n}\notag\\
        &&\quad +\frac{4\gamma\sqrt{5C}\sigma}{\sqrt{n}K} \sum\limits_{k=0}^{K-1}\sqrt{\EE\left[R_k^2\right]\EE[\widehat{b}_k]} + \frac{40\gamma^2 C\sigma^2}{7nK}\sum\limits_{k=0}^{K-1}\EE[\widehat b_k]\label{eq:technical_bound_cvx_known_BTARD_SGD_0}
    \end{eqnarray}
    implying (after changing the indices) that
    \begin{equation}
        \EE[R_k^2] \le R_0^2 + \frac{40\gamma^2\sigma^2 k}{7n} +\frac{4\gamma\sqrt{5C}\sigma}{\sqrt{n}} \sum\limits_{l=0}^{k-1}\sqrt{\EE\left[R_l^2\right]\EE[\widehat{b}_l]} + \frac{40\gamma^2 C\sigma^2}{7n}\sum\limits_{l=0}^{k-1}\EE[\widehat b_l] \label{eq:technical_bound_cvx_known_BTARD_SGD_1}
    \end{equation}
    holds for all $k\ge 0$. In the remaining part of the proof we derive by induction that \begin{equation}
        R_0^2 + \frac{40\gamma^2\sigma^2 k}{7n} +\frac{4\gamma\sqrt{5C}\sigma}{\sqrt{n}} \sum\limits_{l=0}^{k-1}\sqrt{\EE\left[R_l^2\right]\EE[\widehat{b}_l]} + \frac{40\gamma^2 C\sigma^2}{7n}\sum\limits_{l=0}^{k-1}\EE[\widehat b_l] \le 2R_0^2 \label{eq:technical_bound_cvx_known_BTARD_SGD_2}
    \end{equation}
    for all $k=0,\ldots,K$. For $k=0$ this inequality trivially holds. Next, assume that it holds for all $k=0,1,\ldots, T-1$, $T \le K-1$. Let us show that it holds for $k = T$ as well. From \eqref{eq:technical_bound_cvx_known_BTARD_SGD_1} and \eqref{eq:technical_bound_cvx_known_BTARD_SGD_2} we have that $\EE[R_k^2]\le 2R_0^2$ for all $k=0,1,\ldots, T-1$. Therefore,
    \begin{eqnarray*}
        \EE[R_T^2] &\le& R_0^2 + \frac{40\gamma^2\sigma^2 T}{7n} +\frac{4\gamma\sqrt{5C}\sigma}{\sqrt{n}} \sum\limits_{l=0}^{T-1}\sqrt{\EE\left[R_l^2\right]\EE[\widehat{b}_l]} + \frac{40\gamma^2 C\sigma^2}{7n}\sum\limits_{l=0}^{T-1}\EE[\widehat b_l]\\
        &\le& R_0^2 + \frac{40\gamma^2\sigma^2 T}{7n} +\frac{4\gamma\sqrt{10C}\sigma R_0}{\sqrt{n}} \sum\limits_{l=0}^{T-1}\sqrt{\EE[\widehat{b}_l]} + \frac{40\gamma^2 C\sigma^2}{7n}\sum\limits_{l=0}^{T-1}\EE[\widehat b_l].
    \end{eqnarray*}
    If a Byzantine peer deviates from the protocol at iteration $k$, it will be detected with some probability $p_k$ during the next iteration. One can lower bound this probability as
    \begin{equation*}
        p_k \ge m\cdot \frac{|\cG_k|}{n_k}\cdot \frac{1}{n_k} = \frac{m(1-\delta_k)}{n_k} \ge \frac{m}{n}.
    \end{equation*}
    Therefore, each individual Byzantine worker can violate the protocol no more than $\nicefrac{1}{p}$ times on average implying that
    \begin{eqnarray*}
        \EE[R_T^2] &\le& R_0^2 + \frac{40\gamma^2\sigma^2 T}{7n} +\frac{4n\gamma\sqrt{10Cb}\sigma R_0}{m\sqrt{n}} + \frac{40\gamma^2 C\sigma^2nb}{7nm}\\
        &=& R_0^2 + \frac{40\gamma^2\sigma^2 T}{7n} +\frac{4n\gamma\sqrt{10C\delta}\sigma R_0}{m} + \frac{40\gamma^2 C\sigma^2n\delta}{7m}.
    \end{eqnarray*}
    Taking
    \begin{equation*}
        \gamma = \min\left\{\frac{1}{4L},\sqrt{\frac{7nR_0^2}{120\sigma^2 K}}, \sqrt{\frac{m^2 R_0^2}{1440 C\sigma^2 n^2\delta}}\right\}
    \end{equation*}
    we ensure that
    \begin{equation*}
        \frac{40\gamma^2\sigma^2 T}{7n} +\frac{4n\gamma\sqrt{10C\delta}\sigma R_0}{m} + \frac{40\gamma^2 C\sigma^2n\delta}{7m} \le \frac{R_0^2}{3} + \frac{R_0^2}{3} + \frac{R_0^2}{3} = R_0^2, 
    \end{equation*}
    and, as a result, we get $\EE[R_T^2] \le 2R_0^2$. Therefore, \eqref{eq:technical_bound_cvx_known_BTARD_SGD_2} holds for all $k=0,1,\ldots,K$. Together with \eqref{eq:technical_bound_cvx_known_BTARD_SGD_0} it implies
    \begin{equation*}
        \EE\left[f(\overline{x}^K) - f(x^*)\right] \le \frac{2R_0^2}{\gamma K}.
    \end{equation*}
    Next, from our stepsize rule \eqref{eq:choice_of_parameters_BTARD_SGD_cvx} it follows that
    \begin{equation*}
        \EE\left[f(\overline{x}^K) - f(x^*)\right] = \cO\left(\frac{LR_0^2}{K} + \frac{\sigma R_0}{\sqrt{n K}} + \frac{n\sqrt{\delta}\sigma R_0}{mK}\right)
    \end{equation*}
    meaning that after
    \begin{equation*}
        K = \cO\left(\frac{LR_0^2}{\varepsilon} + \frac{\sigma^2 R_0^2}{n\varepsilon^2} + \frac{n\sqrt{\delta}\sigma R_0 }{m\varepsilon}\right)
    \end{equation*}
    iterations \algname{BTARD-SGD} guarantees $\EE[f(\overline{x}^{K}) -f(x^*)]\le \varepsilon$.
\end{proof}

In the convex case, similar observations hold as in the non-convex case. Next, we derive the result without assuming that $\widehat b^k$ is known to all peers at each iteration.

\begin{theorem}\label{thm:BTARD_SGD_unknown_byz_cvx}
    Let As.~\ref{as:bounded_var}~and~\ref{as:quadratically_bounded_tails} hold, $Q = \R^d$, $f$ be $L$-smooth (see Def.~\ref{def:L_smoothness}), convex, and $x^*$ be some optimum of $f$. Moreover, assume that $b \le 0.1(n-m)$, $m \le \nicefrac{(n-2b)}{2}$, and $\delta = 0$ is used to compute clipping parameter $\tau_l$ for \ref{eq:CenteredClip}. Next, assume that 
    \begin{equation}
        \gamma = \min\left\{\frac{1}{4L},\sqrt{\frac{7nR_0^2}{120\sigma^2 K}}, \sqrt{\frac{m^2 R_0^2}{72 C\sigma^2 n^2b^2}}\right\},\quad \Delta_{\max}^k = \frac{(1 + \sqrt{3})\sqrt{2}\sigma}{\sqrt{n_k-m}}, \label{eq:choice_of_parameters_BTARD_SGD_cvx_unknown}
    \end{equation}
    where $R_0 \ge \|x^0 - x^*\|$ and $\Delta_{\max}^k$ is the parameter for verification $3$ at iteration $k$ of \algname{BTARD-SGD}. Then, we have $\EE[f(\overline{x}^{K}) -f(x^*)]\le \varepsilon$ after $K$ iterations of \algname{BTARD-SGD}, where
    \begin{equation}
        K = \cO\left(\frac{LR_0^2}{\varepsilon} + \frac{\sigma^2 R_0^2}{n\varepsilon^2} + \frac{nb\sigma R_0 }{m\varepsilon}\right) \label{eq:BTARD_SGD_unknown_byz_cvx}
    \end{equation}
    and $\overline{x}^K = \frac{1}{K}\sum_{k=0}^{K-1}$.
\end{theorem}
\begin{proof}
    The proof is almost identical to the proof of Theorem~\ref{thm:BTARD_SGD_known_byz_cvx}. Following the same steps and using \eqref{eq:BTARD_overall_peer_unknown_num_byz} and \eqref{eq:BTARD_second_moment_unknown_num_byz} instead of \eqref{eq:BTARD_overall_peer_known_num_byz} and \eqref{eq:BTARD_second_moment_known_num_byz} respectively we obtain the same sequence of inequalities up to the following change: instead of $\widehat{\delta}_k$ we should use $\indicator_{k,v}$. Therefore, we have
    \begin{eqnarray*}
        \EE\left[\|x^{k+1} - x^*\|^2\mid x^k\right] &\le& \|x^k - x^*\|^2 -2\gamma\left(1 - 2L\gamma\right)\left(f(x^k) - f(x^*)\right)\\
        &&\quad - 2\gamma\EE\left[\langle x^k-x^*, \widehat{g}^k - \overline{g}^k \rangle\mid x^k\right] + 2\gamma^2 C\sigma^2\indicator_{k,v} + \frac{2\gamma^2\sigma^2}{n-2b-m},
    \end{eqnarray*}
    \begin{eqnarray*}
        - 2\gamma\EE\left[\langle x^k-x^*, \widehat{g}^k - \overline{g}^k \rangle\mid x^k\right] &\le& 2\gamma\sqrt{C}\sigma \|x^k - x^*\|\indicator_{k,v},
    \end{eqnarray*}
    that result in
    \begin{eqnarray*}
        \EE\left[\|x^{k+1} - x^*\|^2\mid x^k\right] &\le& \|x^k - x^*\|^2 -\gamma\left(f(x^k) - f(x^*)\right)\\
        &&\quad + 2\gamma\sqrt{C}\sigma \|x^k - x^*\|\indicator_{k,v} + 2\gamma^2 C\sigma^2\indicator_{k,v} + \frac{40\gamma^2\sigma^2}{7n}.
    \end{eqnarray*}
    Taking the full expectation from the both sides of the above inequality and summing up the results for $k = 0,1,\ldots,K-1$ we derive
    \begin{eqnarray*}
        \frac{\gamma}{K}\sum\limits_{k=0}^{K-1}\EE[f(x^k)-f(x^*)] &\le& \frac{1}{ K}\sum\limits_{k=0}^{K-1}\left(\EE\left[\|x^k - x^*\|^2\right] - \EE\left[\|x^{k+1} - x^*\|^2\right]\right) + \frac{40\gamma^2\sigma^2}{7n}\\
        &&\quad +\frac{2\gamma\sqrt{C}\sigma}{K} \sum\limits_{k=0}^{K-1}\EE\left[\|x^k - x^*\|\indicator_{k,v}\right] + \frac{2\gamma^2 C\sigma^2}{K}\sum\limits_{k=0}^{K-1}\EE[\indicator_{k,v}]\\
        &\le& \frac{\|x^0-x^*\|^2 - \EE[\|x^{K}-x^*\|^2]}{K} + \frac{40\gamma^2\sigma^2}{7n}\\
        &&\quad +\frac{2\gamma\sqrt{C}\sigma}{K} \sum\limits_{k=0}^{K-1}\sqrt{\EE\left[\|x^k - x^*\|^2\right]\EE[\indicator_{k,v}]} + \frac{2\gamma^2 C\sigma^2}{K}\sum\limits_{k=0}^{K-1}\EE[\indicator_{k,v}].
    \end{eqnarray*}
    From Jensen's inequality we have $f(\overline{x}^K) \le \frac{1}{K}\sum_{k=0}^{K-1}f(x^k)$, where $\overline{x}^K = \frac{1}{K}\sum_{k=0}^{K-1} x^k$. Using this and new notation $R_k = \|x^k - x^*\|$, $k\ge 0$ we get
    \begin{eqnarray}
        0 \le \gamma\EE\left[f(\overline{x}^K) - f(x^*)\right] &\le& \frac{R_0^2 - \EE[R_K^2]}{K} + \frac{40\gamma^2\sigma^2}{7n}\notag\\
        &&\quad +\frac{2\gamma\sqrt{C}\sigma}{K} \sum\limits_{k=0}^{K-1}\sqrt{\EE\left[R_k^2\right]\EE[\indicator_{k,v}]} + \frac{2\gamma^2 C\sigma^2}{K}\sum\limits_{k=0}^{K-1}\EE[\indicator_{k,v}]\label{eq:technical_bound_cvx_unknown_BTARD_SGD_0}
    \end{eqnarray}
    implying (after changing the indices) that
    \begin{equation}
        \EE[R_k^2] \le R_0^2 + \frac{40\gamma^2\sigma^2 k}{7n} +2\gamma\sqrt{C}\sigma \sum\limits_{l=0}^{k-1}\sqrt{\EE\left[R_l^2\right]\EE[\indicator_{l,v}]} + 2\gamma^2 C\sigma^2\sum\limits_{l=0}^{k-1}\EE[\indicator_{l,v}] \label{eq:technical_bound_cvx_unknown_BTARD_SGD_1}
    \end{equation}
    holds for all $k\ge 0$. In the remaining part of the proof we derive by induction that \begin{equation}
        R_0^2 + \frac{40\gamma^2\sigma^2 k}{7n} +2\gamma\sqrt{C}\sigma \sum\limits_{l=0}^{k-1}\sqrt{\EE\left[R_l^2\right]\EE[\indicator_{l,v}]} + 2\gamma^2 C\sigma^2\sum\limits_{l=0}^{k-1}\EE[\indicator_{l,v}] \le 2R_0^2 \label{eq:technical_bound_cvx_unknown_BTARD_SGD_2}
    \end{equation}
    for all $k=0,\ldots,K$. For $k=0$ this inequality trivially holds. Next, assume that it holds for all $k=0,1,\ldots, T-1$, $T \le K-1$. Let us show that it holds for $k = T$ as well. From \eqref{eq:technical_bound_cvx_unknown_BTARD_SGD_1} and \eqref{eq:technical_bound_cvx_unknown_BTARD_SGD_2} we have that $\EE[R_k^2]\le 2R_0^2$ for all $k=0,1,\ldots, T-1$. Therefore,
    \begin{eqnarray*}
        \EE[R_T^2] &\le& R_0^2 + \frac{40\gamma^2\sigma^2 k}{7n} +2\gamma\sqrt{C}\sigma \sum\limits_{l=0}^{T-1}\sqrt{\EE\left[R_l^2\right]\EE[\indicator_{l,v}]} + 2\gamma^2 C\sigma^2\sum\limits_{l=0}^{T-1}\EE[\indicator_{l,v}]\\
        &\le& R_0^2 + \frac{40\gamma^2\sigma^2 k}{7n} +2\gamma\sqrt{2C}\sigma R_0 \sum\limits_{l=0}^{T-1}\sqrt{\EE[\indicator_{l,v}]} + 2\gamma^2 C\sigma^2\sum\limits_{l=0}^{T-1}\EE[\indicator_{l,v}].
    \end{eqnarray*}
    If a Byzantine peer deviates from the protocol at iteration $k$, it will be detected with some probability $p_k$ during the next iteration. One can lower bound this probability as
    \begin{equation*}
        p_k \ge m\cdot \frac{|\cG_k|}{n_k}\cdot \frac{1}{n_k} = \frac{m(1-\delta_k)}{n_k} \ge \frac{m}{n}.
    \end{equation*}
     That is, each individual Byzantine worker can violate the protocol no more than $\nicefrac{1}{p}$ times on average. However, even one Byzantine peer can create a shift of the order $\Delta_{\max}^k$ at each part of the resulting vector. Therefore, all Byzantine peers can violate the protocol no more than $\nicefrac{b}{p}$ times on average implying that
    \begin{eqnarray*}
        \EE[R_T^2] &\le& R_0^2 + \frac{40\gamma^2\sigma^2 T}{7n} +\frac{2\gamma nb\sqrt{2C}\sigma R_0}{m} + \frac{2\gamma^2 nb C\sigma^2}{m}.
    \end{eqnarray*}
    Taking
    \begin{equation*}
        \gamma = \min\left\{\frac{1}{4L},\sqrt{\frac{7nR_0^2}{120\sigma^2 K}}, \sqrt{\frac{m^2 R_0^2}{72 C\sigma^2 n^2b^2}}\right\}
    \end{equation*}
    we ensure that
    \begin{equation*}
        \frac{40\gamma^2\sigma^2 T}{7n} +\frac{2\gamma nb\sqrt{2C}\sigma R_0}{m} + \frac{2\gamma^2 nb C\sigma^2}{m} \le \frac{R_0^2}{3} + \frac{R_0^2}{3} + \frac{R_0^2}{3} = R_0^2, 
    \end{equation*}
    and, as a result, we get $\EE[R_T^2] \le 2R_0^2$. Therefore, \eqref{eq:technical_bound_cvx_unknown_BTARD_SGD_2} holds for all $k=0,1,\ldots,K$. Together with \eqref{eq:technical_bound_cvx_unknown_BTARD_SGD_0} it implies
    \begin{equation*}
        \EE\left[f(\overline{x}^K) - f(x^*)\right] \le \frac{2R_0^2}{\gamma K}.
    \end{equation*}
    Next, from our stepsize rule \eqref{eq:choice_of_parameters_BTARD_SGD_cvx_unknown} it follows that
    \begin{equation*}
        \EE\left[f(\overline{x}^K) - f(x^*)\right] = \cO\left(\frac{LR_0^2}{K} + \frac{\sigma R_0}{\sqrt{n K}} + \frac{nb\sigma R_0}{mK}\right)
    \end{equation*}
    meaning that after
    \begin{equation*}
        K = \cO\left(\frac{LR_0^2}{\varepsilon} + \frac{\sigma^2 R_0^2}{n\varepsilon^2} + \frac{nb\sigma R_0 }{m\varepsilon}\right)
    \end{equation*}
    iterations \algname{BTARD-SGD} guarantees $\EE[f(\overline{x}^{K}) -f(x^*)]\le \varepsilon$.
\end{proof}

\subsubsection{Strongly convex case: Restarted-BTARD-SGD}\label{appendix:str_convex_analysis_BTARD_SGD}
In this section, we provide the complete statements and the full proofs of the convergence results for the restarted version of \algname{BTARD-SGD} (\algname{Restarted-BTARD-SGD}, Alg.~\ref{alg:restarted_BTARD_SGD}) when the objective function $f$ is smooth and strongly convex.
\begin{algorithm}[H]
   \caption{\algname{Restarted-BTARD-SGD}}
   \label{alg:restarted_BTARD_SGD}
\begin{algorithmic}[1]
   \Require $x^0$ -- starting point, $r$ -- number of restarts, $\{\gamma_t\}_{t=1}^r$ -- stepsizes for \algname{BTARD-SGD}, $\{K_t\}_{t=1}^r$ -- number of iterations for \algname{BTARD-SGD}, $\{s_{i,k,t}\}_{i,k,t=0,0,0}^{n,K-1,r}$ -- seeds for batches computations
   \State $\widehat{x}^{0} = x^0$
   \For{$t = 1,2,\ldots,r$}
    \State Run \algname{BTARD-SGD} (Alg.~\ref{alg:BTARD_SGD}) for $K_t$ iterations with stepsize $\gamma_t$, starting point $\widehat{x}^{t-1}$, and seeds for batches computations $\{s_{i,k,t}\}_{i,k=0,0}^{n,K-1}$. Define $\widehat{x}^{t}$ as $\widehat{x}^{t} = \frac{1}{K_t}\sum\limits_{k=0}^{K_t}x^{k,t}$, where $x^{0,t}, x^{1,t}, \ldots, x^{K_t,t}$ are the iterates produced by \algname{BTARD-SGD}.
   \EndFor
   \Ensure $\widehat x^r$
\end{algorithmic}
\end{algorithm}
 We start with the case when the number of attacking Byzantine workers is known at each iteration.

\begin{theorem}\label{thm:BTARD_SGD_known_byz_str_cvx}
    Let As.~\ref{as:bounded_var}~and~\ref{as:quadratically_bounded_tails} hold, $Q = \R^d$, $f$ be $L$-smooth (see Def.~\ref{def:L_smoothness}), $\mu$-strongly convex (see Def.~\ref{def:mu_strong_convexity}), and $x^*$ be some optimum of $f$. Moreover, assume that $b \le 0.1(n-m)$, $m \le \nicefrac{(n-2b)}{2}$, and the exact number of attacking Byzantine peers is known to all good peers at each iteration. Next, assume that 
    \begin{equation}
        \gamma_t = \min\left\{\frac{1}{4L}, \sqrt{\frac{7nR_0^2}{120\cdot 2^{t}\sigma^2 K_t}}, \sqrt{\frac{m^2 R_0^2}{1440\cdot 2^{t} C\sigma^2 n^2\delta}}\right\},\quad \Delta_{\max}^{k,t} = \frac{(1 + \sqrt{3})\sqrt{2}\sigma}{\sqrt{n_k^t-m}}, \label{eq:choice_of_parameters_BTARD_SGD_str_cvx}
    \end{equation}
    \begin{equation}
        K_t = \left\lceil\max\left\{\frac{16 L}{\mu}, \frac{32\sigma^2 2^t}{\mu^2 R_0^2}, \frac{48\sqrt{10C}n\sqrt{\delta}\sigma 2^{\frac{t}{2}}}{m\mu R_0}\right\} \right\rceil,\quad r = \left\lceil\log_2\frac{\mu R_0^2}{\varepsilon} \right\rceil - 1 \label{eq:choice_of_parameters_BTARD_SGD_str_cvx_2}
    \end{equation}
    where $R_0 \ge \|x^0 - x^*\|$, $\Delta_{\max}^{k,t}$ is the parameter for verification $3$ at iteration $k$ of \algname{BTARD-SGD} during the $t$-th restart, $n_k^t$ is the total number of workers at iteration $k$ of $t$-th restart. Then, we have $\EE[f(\widehat{x}^{r}) -f(x^*)]\le \varepsilon$ after $r$ restarts of \algname{BTARD-SGD} and the total number of executed iterations of \algname{BTARD-SGD} is
    \begin{equation}
        \sum\limits_{t=1}^rK_t = \cO\left(\frac{L}{\mu}\log\frac{\mu R_0^2}{\varepsilon} + \frac{\sigma^2 }{n\mu\varepsilon} + \frac{n\sqrt{\delta}\sigma }{m\sqrt{\mu\varepsilon}}\right). \label{eq:BTARD_SGD_known_byz_str_cvx}
    \end{equation}
\end{theorem}
\begin{proof}
    Theorem~\ref{thm:BTARD_SGD_known_byz_cvx} implies that \algname{BTARD-SGD} with 
    \begin{equation*}
        \gamma = \min\left\{\frac{1}{4L}, \sqrt{\frac{7nR_0^2}{120\sigma^2 K}}, \sqrt{\frac{m^2 R_0^2}{1440 C\sigma^2 n^2\delta}}\right\}
    \end{equation*}
    guarantees
    \begin{equation*}
        \EE\left[f(\overline{x}^K) - f(x^*)\right] \le \frac{2R_0^2}{\gamma K}
    \end{equation*}
    after $K$ iterations. Therefore, after the first restart we have
    \begin{equation}
        \EE[f(\widehat{x}^1) - f(x^*)] \le \frac{2R_0^2}{\gamma_1 K_1} \le \frac{\mu R_0^2}{4}. \notag
    \end{equation}
    From $\mu$-strong convexity of $f$ and $\nabla f(x^*) = 0$ we have
    \begin{equation*}
        \frac{\mu}{2}\|\widehat x^1 - x^*\|^2 \le f(\widehat x^1) - f(x^*) \Longrightarrow \EE[\|\widehat x^1 - x^*\|^2] \le \frac{R_0^2}{2}.
    \end{equation*}
    Next, assume that we have $\EE[f(\widehat{x}^{t}) - f(x^*)] \le \frac{\mu R_0^2}{2^{t+1}}$, $\EE[\|\widehat x^t - x^*\|^2] \le \frac{R_0^2}{2^t}$ for some $t \le r-1$. Then, Theorem~\ref{thm:BTARD_SGD_known_byz_cvx} implies that
    \begin{equation*}
        \EE[f(\widehat{x}^{t+1}) - f(x^*)\mid x^t] \le \frac{2\|\widehat{x}^t - x^*\|^2}{\gamma_t K_t}.
    \end{equation*}
    Taking the full expectation from the both sides of previous inequality we get
    \begin{equation*}
        \EE[f(\widehat{x}^{t+1}) - f(x^*)] \le \frac{2\EE[\|\widehat{x}^t - x^*\|^2]}{\gamma_t K_t} \le \frac{2R_0^2}{2^t\gamma_t K_t} \le \frac{\mu R_0^2}{2^{t+2}}.
    \end{equation*}
    From $\mu$-strong convexity of $f$ and $\nabla f(x^*) = 0$ we have
    \begin{equation*}
        \frac{\mu}{2}\|\widehat x^{t+1} - x^*\|^2 \le f(\widehat x^{t+1}) - f(x^*) \Longrightarrow \EE[\|\widehat x^{t+1} - x^*\|^2] \le \frac{R_0^2}{2^{t+1}}.
    \end{equation*}
    Therefore, by mathematical induction we have that for all $t=1,\ldots,r$
    \begin{equation*}
        \EE[f(\widehat{x}^t) - f(x^*)] \le \frac{\mu R_0^2}{2^{t+1}}, \quad \EE\left[\|\widehat x^t - x^*\|^2\right] \le \frac{R_0^2}{2^t}.
    \end{equation*}
    Then, after $r = \left\lceil\log_2\frac{\mu R_0^2}{\varepsilon} \right\rceil - 1$ restarts of \algname{BTARD-SGD} we have $\EE[f(\widehat{x}^r) - f(x^*)] \le \varepsilon$. The total number of iterations executed by \algname{BTARD-SGD} is
    \begin{eqnarray*}
        \sum\limits_{t=1}^r K_t &=& \cO\left(\sum\limits_{t=1}^r\max\left\{\frac{ L}{\mu}, \frac{\sigma^2 2^t}{\mu^2 R_0^2}, \frac{n\sqrt{\delta}\sigma 2^{\frac{t}{2}}}{m\mu R_0}\right\}\right)\\
        &=& \cO\left(\frac{L}{\mu}r + \frac{\sigma^2 2^r}{\mu^2 R_0^2} + \frac{n\sqrt{\delta}\sigma 2^{\frac{r}{2}}}{m\mu R_0}\right)\\
        &=& \cO\left(\frac{L}{\mu}\log\frac{\mu R_0^2}{\varepsilon} + \frac{\sigma^2}{\mu^2 R_0^2}\cdot \frac{\mu R_0^2}{\varepsilon} + \frac{n\sqrt{\delta}\sigma }{m\mu R_0}\cdot \sqrt{\frac{\mu R_0^2}{\varepsilon}}\right)\\
        &=& \cO\left(\frac{L}{\mu}\log\frac{\mu R_0^2}{\varepsilon} + \frac{\sigma^2 }{n\mu\varepsilon} + \frac{n\sqrt{\delta}\sigma }{m\sqrt{\mu\varepsilon}}\right).
    \end{eqnarray*}
\end{proof}

In the strongly convex case, similar observations hold as in the non-convex case. Next, we derive the result without assuming that $\widehat b^k$ is known to all peers at each iteration.

\begin{theorem}\label{thm:BTARD_SGD_unknown_byz_str_cvx}
    Let As.~\ref{as:bounded_var}~and~\ref{as:quadratically_bounded_tails} hold, $Q = \R^d$, $f$ be $L$-smooth (see Def.~\ref{def:L_smoothness}), $\mu$-strongly convex (see Def.~\ref{def:mu_strong_convexity}), and $x^*$ be some optimum of $f$. Moreover, assume that $b \le 0.1(n-m)$, $m \le \nicefrac{(n-2b)}{2}$, and $\delta = 0$ is used to compute clipping parameter $\tau_l$ for \ref{eq:CenteredClip}. Next, assume that 
    \begin{equation}
        \gamma_t = \min\left\{\frac{1}{4L}, \sqrt{\frac{7nR_0^2}{120\cdot 2^t\sigma^2 K_t}}, \sqrt{\frac{m^2 R_0^2}{72\cdot 2^t C\sigma^2 n^2b^2}}\right\},\quad \Delta_{\max}^{k,t} = \frac{(1 + \sqrt{3})\sqrt{2}\sigma}{\sqrt{n_k^t-m}}, \label{eq:choice_of_parameters_BTARD_SGD_str_cvx_unknown}
    \end{equation}
    \begin{equation}
        K_t = \left\lceil\max\left\{\frac{16 L}{\mu}, \frac{32\sigma^2 2^t}{\mu^2 R_0^2}, \frac{24\sqrt{2C}nb\sigma 2^{\frac{t}{2}}}{m\mu R_0}\right\} \right\rceil,\quad r = \left\lceil\log_2\frac{\mu R_0^2}{\varepsilon} \right\rceil - 1 \label{eq:choice_of_parameters_BTARD_SGD_str_cvx_unknown_2}
    \end{equation}
    where $R_0 \ge \|x^0 - x^*\|$, $\Delta_{\max}^{k,t}$ is the parameter for verification $3$ at iteration $k$ of \algname{BTARD-SGD} during the $t$-th restart, $n_k^t$ is the total number of workers at iteration $k$ of $t$-th restart. Then, we have $\EE[f(\widehat{x}^{r}) -f(x^*)]\le \varepsilon$ after $r$ restarts of \algname{BTARD-SGD} and the total number of executed iterations of \algname{BTARD-SGD} is
    \begin{equation}
        \sum\limits_{t=1}^rK_t = \cO\left(\frac{L}{\mu}\log\frac{\mu R_0^2}{\varepsilon} + \frac{\sigma^2 }{n\mu\varepsilon} + \frac{nb\sigma }{m\sqrt{\mu\varepsilon}}\right). \label{eq:BTARD_SGD_unknown_byz_str_cvx}
    \end{equation}
\end{theorem}
\begin{proof}
    Theorem~\ref{thm:BTARD_SGD_unknown_byz_cvx} implies that \algname{BTARD-SGD} with 
    \begin{equation*}
        \gamma = \min\left\{\frac{1}{4L}, \sqrt{\frac{7nR_0^2}{120\sigma^2 K}}, \sqrt{\frac{m^2 R_0^2}{72 C\sigma^2 n^2b^2}}\right\}
    \end{equation*}
    guarantees
    \begin{equation*}
        \EE\left[f(\overline{x}^K) - f(x^*)\right] \le \frac{2R_0^2}{\gamma K}
    \end{equation*}
    after $K$ iterations. Therefore, after the first restart we have
    \begin{equation}
        \EE[f(\widehat{x}^1) - f(x^*)] \le \frac{2R_0^2}{\gamma_1 K_1} \le \frac{\mu R_0^2}{4}. \notag
    \end{equation}
    From $\mu$-strong convexity of $f$ and $\nabla f(x^*) = 0$ we have
    \begin{equation*}
        \frac{\mu}{2}\|\widehat x^1 - x^*\|^2 \le f(\widehat x^1) - f(x^*) \Longrightarrow \EE[\|\widehat x^1 - x^*\|^2] \le \frac{R_0^2}{2}.
    \end{equation*}
    Next, assume that we have $\EE[f(\widehat{x}^{t}) - f(x^*)] \le \frac{\mu R_0^2}{2^{t+1}}$, $\EE[\|\widehat x^t - x^*\|^2] \le \frac{R_0^2}{2^t}$ for some $t \le r-1$. Then, Theorem~\ref{thm:BTARD_SGD_unknown_byz_cvx} implies that
    \begin{equation*}
        \EE[f(\widehat{x}^{t+1}) - f(x^*)\mid x^t] \le \frac{2\|\widehat{x}^t - x^*\|^2}{\gamma_t K_t}.
    \end{equation*}
    Taking the full expectation from the both sides of previous inequality we get
    \begin{equation*}
        \EE[f(\widehat{x}^{t+1}) - f(x^*)] \le \frac{2\EE[\|\widehat{x}^t - x^*\|^2]}{\gamma_t K_t} \le \frac{2R_0^2}{2^t\gamma_t K_t} \le \frac{\mu R_0^2}{2^{t+2}}.
    \end{equation*}
    From $\mu$-strong convexity of $f$ and $\nabla f(x^*) = 0$ we have
    \begin{equation*}
        \frac{\mu}{2}\|\widehat x^{t+1} - x^*\|^2 \le f(\widehat x^{t+1}) - f(x^*) \Longrightarrow \EE[\|\widehat x^{t+1} - x^*\|^2] \le \frac{R_0^2}{2^{t+1}}.
    \end{equation*}
    Therefore, by mathematical induction we have that for all $t=1,\ldots,r$
    \begin{equation*}
        \EE[f(\widehat{x}^t) - f(x^*)] \le \frac{\mu R_0^2}{2^{t+1}}, \quad \EE\left[\|\widehat x^t - x^*\|^2\right] \le \frac{R_0^2}{2^t}.
    \end{equation*}
    Then, after $r = \left\lceil\log_2\frac{\mu R_0^2}{\varepsilon} \right\rceil - 1$ restarts of \algname{BTARD-SGD} we have $\EE[f(\widehat{x}^r) - f(x^*)] \le \varepsilon$. The total number of iterations executed by \algname{BTARD-SGD} is
    \begin{eqnarray*}
        \sum\limits_{t=1}^r K_t &=& \cO\left(\sum\limits_{t=1}^r\max\left\{\frac{ L}{\mu}, \frac{\sigma^2 2^t}{\mu^2 R_0^2}, \frac{nb\sigma 2^{\frac{t}{2}}}{m\mu R_0}\right\}\right)\\
        &=& \cO\left(\frac{L}{\mu}r + \frac{\sigma^2 2^r}{\mu^2 R_0^2} + \frac{nb\sigma 2^{\frac{r}{2}}}{m\mu R_0}\right)\\
        &=& \cO\left(\frac{L}{\mu}\log\frac{\mu R_0^2}{\varepsilon} + \frac{\sigma^2}{\mu^2 R_0^2}\cdot \frac{\mu R_0^2}{\varepsilon} + \frac{nb\sigma }{m\mu R_0}\cdot \sqrt{\frac{\mu R_0^2}{\varepsilon}}\right)\\
        &=& \cO\left(\frac{L}{\mu}\log\frac{\mu R_0^2}{\varepsilon} + \frac{\sigma^2 }{n\mu\varepsilon} + \frac{nb\sigma }{m\sqrt{\mu\varepsilon}}\right).
    \end{eqnarray*}
\end{proof}

\subsection{Convergence guarantees for BTARD-Clipped-SGD}\label{appendix:BTARD_Clipped_SGD_appendix}

The results for \algname{BTARD-SGD} and \algname{Restarted-BTARD-SGD} rely on As.~\ref{as:quadratically_bounded_tails} that the stochastic gradients have not too heavy tails, i.e., sub-quadratically decreasing tails. The main reason why it is needed in the analysis is to prevent too often extra computations because of \textbf{Verification 3} from \algname{BTARD} when all workers honestly follow the protocol. However, in many important NLP tasks such as BERT training \citep{zhang2020why}, the noise in the stochastic gradient has such a heavy noise that As.~\ref{as:quadratically_bounded_tails} becomes unnatural.

\begin{algorithm}[H]
   \caption{\algname{BTARD-Clipped-SGD}}
   \label{alg:BTARD_Clipped_SGD}
\begin{algorithmic}[1]
   \Require $x^0$ -- starting point, $\gamma$ -- stepsize, $K$ -- number of iterations, $\{s_{i,k}\}_{i,k=0,0}^{n,K-1}$ -- seeds for batches computations, $\{\lambda_k\}_{k=0}^{K-1}$ -- gradient clipping parameter
   \State $C_{0} = \text{Banned}_{-1} = \varnothing$
   \For{$k = 0,1,\ldots,K-1$}
    \State Worker $i$ computes $\widetilde{g}_i^k = \begin{cases}\min\left\{1, \frac{\lambda_k}{\|\nabla f(x^k,\xi_{i,k})\|}\right\}\nabla f(x^k,\xi_{i,k}),& \text{if } i\in \cG_k\setminus \cC_k,\\ *,& \text{if } i\in \cB_k\setminus \cC_k,\end{cases}$, where $\xi_{i,k}$ is generated via seed $s_{i,k}$ available to every worker
    \State
    \State $ \left( \widehat{g}^k, \text{public{\_}info}_k \right) = \text{\algname{BTARD}}(\widetilde{g}_{i_1^k}^k,g_{i_1^k}^k,\ldots,\widetilde{g}_{i_{a_k}^k}^k)$, where $\{i_1^k,\ldots, i_{a_k}^k\} = (\cG_k\cup \cB_k)\setminus \cC_k$
    \State
    \State Choose $2m$ workers $c_1^{k+1},\ldots,c_m^{k+1}, u_1^{k+1},\ldots, u_m^{k+1}$ uniformly at random without replacement, $\cC_{k+1} = \{c_1^{k+1},\ldots,c_m^{k+1}\}$, $\cU_{k+1} = \{u_1^{k+1},\ldots, u_m^{k+1}\}$
    \State $\text{Banned}_k = \text{\algname{CheckComputations}}(\cC_{k+1}, \cU_{k+1}, \text{public{\_}info}_k)$
    \State $x^{k+1} = \text{proj}_Q(x^k - \gamma \widehat g^k) := \argmin_{x\in Q}\|x - (x^k - \gamma \widehat g^k)\|$
    \State $\cG_{k+1} = \cG_k\setminus \text{Banned}_{k-1}$
    \State $\cB_{k+1} = \cB_k\setminus \text{Banned}_{k-1}$
   \EndFor
\end{algorithmic}
\end{algorithm}

To handle the problems with heavy-tailed noise distributions we consider \algname{BTARD-Clipped-SGD} (see Algorithm~\ref{alg:BTARD_Clipped_SGD}) applied to solve \eqref{eq:main_problem} such that $Q$ is bounded. Essentially, this algorithm coincides with \algname{BTARD-SGD} up to the following change: all good peers $i\in G_k\setminus C_k$ use clipped stochastic gradients $\widetilde{g}_i^k = (\widetilde{g}_i^k(1)^\top,\ldots,\widetilde{g}_i^k(n_k-m)^\top)^\top$, where $\widetilde{g}_i^k(l) = \min\left\{1, \frac{\lambda_k}{\|g_i^k(l)\|}\right\}g_i^k(l)$, $l=1,\ldots,n_k-m$, and $g_i^k$ is the stochastic gradient. Next, we introduce the following assumption.

\begin{assumption}\label{as:bounded_alpha_moment}
    There exist such constant $G > 0$, $s_0 \in [d]$, and $\alpha \in (1,2]$ that for any set of indices $S = (i_1,\ldots,i_d)$, $1\le i_1< i_2 <\ldots < i_s \le d$, $s \ge s_0$ and arbitrary $x\in Q$ stochastic gradient $\nabla f(x,\xi)$ satisfy
    \begin{equation}
        \EE[\nabla f(x,\xi)] = \nabla f(x),\quad \EE\left[\left\|\nabla_{[S]} f(x,\xi)\right\|^\alpha\right] \le \left(\frac{\sqrt{s}G}{\sqrt{d}}\right)^\alpha, \label{eq:uniformly_bounded_alpha_moment}
    \end{equation}
where $\nabla_{[S]} f(x,\xi)$ is defined in As.~\ref{as:bounded_var}.
\end{assumption}
This is a modified version of the assumption used in \citet{zhang2020why}. When $\alpha < 2$ the variance of the stochastic gradient can be unbounded. One can show that in such a regime vanilla \algname{SGD} can diverge \citep{zhang2020why}. 

Under As.~\ref{as:bounded_alpha_moment} we derive the convergence results for convex and strongly convex problems.

\subsubsection{Quality of the aggregation}
Since now we have As.~\ref{as:bounded_alpha_moment} instead of As.~\ref{as:bounded_var}~and~\ref{as:quadratically_bounded_tails} it is needed to derive new guarantees for the quality of the aggregation. We start with the following useful lemma about the properties of clipped stochastic gradeints.

\begin{lemma}[See also Lemma~9 from \citet{zhang2020why}]
    Let As.~\ref{as:bounded_alpha_moment} holds and $i,j \in\cG_k\setminus \cC_k$. Then, for all $l = 1,2,\ldots, n_k-m$ we have
    \begin{eqnarray}
        \sqrt{\EE\left[\|\widetilde{g}_i^k(l) - \widetilde{g}_j^k(l)\|^4\mid x^k\right]} &\le& 4\lambda_{k}^{\frac{4-\alpha}{2}}\left(\frac{G}{\sqrt{n_k - m}}\right)^{\frac{\alpha}{2}}, \label{eq:clipping_distortion}\\
        \EE\left[\|\overline{g}^k(l)\|^2\mid x^k\right] &\le& \frac{G^\alpha \lambda_k^{2-\alpha}}{(n_k - m)^{\frac{\alpha}{2}}}, \label{eq:clipping_second_moment}\\
        \left\|\EE[\overline{g}^k(l)\mid x^k] - \nabla_{(l)} f(x^k)\right\|^2 &\le& \frac{G^{2\alpha}}{(n_k-m)^{\alpha}\lambda_k^{2(\alpha-1)}}, \label{eq:clipping_bias}
    \end{eqnarray}
    where $\overline{g}^k(l) = \frac{1}{|\cG_k\setminus \cC_k|}\sum\limits_{i\in\cG_k\setminus \cC_k}\widetilde{g}_i^k(l)$ for all $l=1,\ldots,n_k-m$.
\end{lemma}
\begin{proof}
    First of all, we derive
    \begin{eqnarray*}
        \EE\left[\|\widetilde{g}_i^k(l) - \widetilde{g}_j^k(l)\|^4\mid x^k\right] &=& \EE\left[\|\widetilde{g}_i^k(l) - \widetilde{g}_j^k(l)\|^\alpha \|\widetilde{g}_i^k(l) - \widetilde{g}_j^k(l)\|^{4-\alpha}\mid x^k\right]\\
        &\le& 8\lambda_k^{4-\alpha}\EE\left[\|\nabla_{(l)}f(x^k,\xi_{i,k})\|^\alpha + \|\nabla_{(l)}f(x^k,\xi_{j,k})\|^{\alpha}\mid x^k\right]\\
        &\overset{\eqref{eq:uniformly_bounded_alpha_moment}}{\le}& 16\lambda_k^{4-\alpha}\left(\frac{G}{\sqrt{n_k-m}}\right)^\alpha
    \end{eqnarray*}
    implying \eqref{eq:clipping_distortion}. Next, for all $i \in \cG_k\setminus \cC_k$ we have
    \begin{eqnarray*}
        \EE\left[\|\widetilde{g}_i^k(l)\|^2\mid x^k\right] &=& \EE\left[\|\widetilde{g}_i^k(l)\|^\alpha \|\widetilde{g}_i^k(l)\|^{2-\alpha}\mid x^k\right] \le \lambda_k^{2-\alpha} \EE\left[\|\nabla_{(l)}f(x^k,\xi_{i,k})\|^\alpha \mid x^k\right]\\
        &\overset{\eqref{eq:uniformly_bounded_alpha_moment}}{\le}& \frac{G^\alpha \lambda_k^{2-\alpha}}{(n_k - m)^{\frac{\alpha}{2}}}
    \end{eqnarray*}
    implying
    \begin{equation*}
        \EE\left[\|\overline{g}^k(l)\|^2\mid x^k\right] \le \frac{1}{|\cG_k\setminus\cC_k|}\sum\limits_{i\in \cG_k\setminus\cC_k}\EE\left[\|\widetilde{g}_{i}^k(l)\|^2\mid x^k\right] \le \frac{G^\alpha \lambda_k^{2-\alpha}}{(n_k - m)^{\frac{\alpha}{2}}}.
    \end{equation*}
    Finally, for all $i \in \cG_k\setminus \cC_k$ we derive
    \begin{eqnarray*}
        \left\|\EE[\widetilde{g}_i^k(l)\mid x^k] - \nabla_{(l)}f(x^k)\right\| &=& \left\|\EE[\widetilde{g}_i^k(l) - \nabla_{(l)} f(x^k,\xi_{i,k})\mid x^k]\right\|\\
        &\le& \EE\left[\left\|\widetilde{g}_i^k(l) - \nabla_{(l)} f(x^k,\xi_{i,k})\right\| \mid x^k\right] \\
        &=& \EE\left[\left\|\widetilde{g}_i^k(l) - \nabla_{(l)} f(x^k,\xi_{i,k})\right\|\indicator_{\{\|\nabla_{(l)} f(x^k,\xi_{i,k})\| \ge \lambda_k\}} \mid x^k\right]\\
        &\le& \EE\left[\left\|\nabla_{(l)} f(x^k,\xi_{i,k})\right\|\indicator_{\{\|\nabla_{(l)} f(x^k,\xi_{i,k})\| \ge \lambda_k\}} \mid x^k\right]\\
        &\le& \frac{\EE\left[\left\|\nabla_{(l)} f(x^k,\xi_{i,k})\right\|^\alpha\indicator_{\{\|\nabla_{(l)} f(x^k,\xi_{i,k})\| \ge \lambda_k\}} \mid x^k\right]}{\lambda_k^{\alpha-1}}\\
        &\overset{\eqref{eq:uniformly_bounded_alpha_moment}}{\le}& \frac{G^\alpha}{(n_k-m)^{\frac{\alpha}{2}}\lambda_k^{\alpha-1}}
    \end{eqnarray*}
    implying
    \begin{eqnarray*}
        \left\|\EE[\overline{g}^k(l)\mid x^k] - \nabla_{(l)} f(x^k)\right\|^2 &\le& \frac{1}{|\cG_k\setminus\cC_k|}\sum\limits_{i\in \cG_k\setminus\cC_k}\EE\left[\|\widetilde{g}_{i}^k(l)- \nabla_{(l)} f(x^k)\|^2\mid x^k\right]\\
        &\le& \frac{G^{2\alpha}}{(n_k-m)^{\alpha}\lambda_k^{2(\alpha-1)}}.
    \end{eqnarray*}
\end{proof}

Next, we derive the guarantees for the quality of the aggregation in the case when the number of Byzantine peers violating the protocol $\widehat b_k$ is known at each iteration.
\begin{lemma}\label{lem:quality_of_agg_known_num_byz_clipped_sgd}
    Let As.~\ref{as:bounded_alpha_moment} hold and $b \le 0.15(n-m)$. Assume that $\widehat{b}_k$ is known for each worker at iteration $k$, $\Delta_{\max}^k = 2\lambda_k = \frac{2\lambda}{\sqrt{n_k-m}}$ and $\delta = \widehat{\delta}_k$ is used to compute clipping parameter $\tau_l$ for \ref{eq:CenteredClip}. If the total number of iterations $T$ of \ref{eq:CenteredClip} satisfies $T \ge \log_{0.94}\frac{2\delta\sigma^2}{\EE[\|v^0 - \overline{g}^k\|^2]}$ and $\text{\algname{CheckAveraging}}$ is not triggered for any worker, then
     \begin{equation}
        \EE\left[\|\widehat{g}^k - \overline{g}^k\|^2\mid x^k\right] \le \widehat{\delta}_k(C_1\lambda^{\frac{4-\alpha}{2}}G^{\frac{\alpha}{2}} + C_2\lambda^2), \label{eq:BTARD_overall_peer_known_num_byz_clipped_SGD}
    \end{equation}
    \begin{equation}
        \EE\left[\|\widehat{g}^k\|^2\mid x^k\right] \le  2\widehat{\delta}_k(C_1\lambda^{\frac{4-\alpha}{2}}G^{\frac{\alpha}{2}} + C_2\lambda^2) + 2G^\alpha \lambda^{2-\alpha}, \label{eq:BTARD_second_moment_known_num_byz_clipped_SGD}
    \end{equation}
    where $\overline{g}^k = \frac{1}{|\cG_k\setminus \cC_k|}\sum\limits_{j\in \cG_k\setminus \cC_k} g_j^k$, $C_1 = 384$, and $C_2 = 4$. 
\end{lemma}
\begin{proof}
    Consider the $i$-th part of $\widehat{g}^k$, i.e., consider $\widehat{g}^k(i)$. If $i \in\cG_{k}\setminus \cC_k$, then, in view of \eqref{eq:clipping_distortion}, we can directly apply Lemma~\ref{lem:centered_clip_guarantee_fixed} and get
    \begin{equation*}
        \EE\left[\|\widehat{g}^k(i) - \overline{g}^k(i)\|^2\mid x^k\right] \le 384\widehat{\delta}_k\lambda_{k}^{\frac{4-\alpha}{2}}\frac{G^{\frac{\alpha}{2}}}{(n_k-m)^{\frac{\alpha}{4}}} = \frac{384\widehat{\delta}_k\lambda^{\frac{4-\alpha}{2}}G^{\frac{\alpha}{2}}}{n_k-m}.
    \end{equation*}
    Next, if $i \in\cB_{k}\setminus \cC_k$, then 
    \begin{equation*}
        \EE\left[\|\widehat{g}^k(i) - \overline{g}^k(i)\|^2\mid x^k\right] \le (\Delta_{\max}^k)^2 = 4\lambda_k^2 = \frac{4\lambda^2}{n_k-m}.
    \end{equation*}
    Putting all together, we derive
    \begin{eqnarray*}
        \EE\left[\|\widehat{g}^k - \overline{g}^k\|^2\mid x^k\right] &=& \sum\limits_{i\in \cG_k\setminus \cC_k} \EE\left[\|\widehat{g}^k(i) - \overline{g}^k(i)\|^2\mid x^k\right] + \sum\limits_{i\in \cB_k\setminus \cC_k} \EE\left[\|\widehat{g}^k(i) - \overline{g}^k(i)\|^2\mid x^k\right]\\
        &\le& (1-\widehat{\delta}_k)(n_k-m)\cdot \frac{384\widehat{\delta}_k\lambda^{\frac{4-\alpha}{2}}G^{\frac{\alpha}{2}}}{n_k-m} + \widehat{\delta}_k(n_k - m)\cdot \frac{4\lambda^2}{n_k-m}\\
        &\le& \widehat{\delta}_k(C_1\lambda^{\frac{4-\alpha}{2}}G^{\frac{\alpha}{2}} + C_2\lambda^2).
    \end{eqnarray*}
    Using \eqref{eq:clipping_second_moment} we obtain
    \begin{eqnarray*}
        \EE\left[\|\widehat{g}^k\|^2\mid x^k\right] &\le& 2\EE\left[\|\widehat{g}^k - \overline{g}^k\|^2\mid x^k\right] + 2\EE\left[\|\overline{g}^k\|^2\mid x^k\right]\\
        &\overset{\eqref{eq:BTARD_overall_peer_known_num_byz_clipped_SGD}}{\le}& 2\widehat{\delta}_k(C_1\lambda^{\frac{4-\alpha}{2}}G^{\frac{\alpha}{2}} + C_2\lambda^2) + 2\sum\limits_{i\in (\cG_k\cup\cB_k)\setminus \cC_k}\frac{G^\alpha \lambda_k^{2-\alpha}}{(n_k - m)^{\frac{\alpha}{2}}}\\
        &=& 2\widehat{\delta}_k(C_1\lambda^{\frac{4-\alpha}{2}}G^{\frac{\alpha}{2}} + C_2\lambda^2) + 2G^\alpha \lambda^{2-\alpha}.
    \end{eqnarray*}
\end{proof}

We notice that \textbf{Verification 3} can be simplified in the following way: if at least on good peer $i$ notices that $\|\widetilde{g}_i^k(j) - \widehat{g}^k(j)\| > \Delta_{\max}^k = 2\lambda_k$, then peer $i$ should accuse $j$-th peer and both are removed from the training process. In this scenario, there is no sense for Byzantine workers in triggering to deviate significantly from the clipped stochastic gradients of the good peers.

As for \algname{BTARD-SGD}, when $\widehat{b}_k$ is unknown we always use \algname{CenteredClip} with $\tau_l = \infty$ for all $l\ge 0$, i.e., good peers compute an exact average. In this settings, even $1$ Byzantine worker can significantly shift the average in all parts of the vector. The next lemma quantifies the negative effect of Byzantine workers in this case.

\begin{lemma}\label{lem:quality_of_agg_unknown_num_byz_clipped_sgd}
    Let As.~\ref{as:bounded_alpha_moment} hold and $b \le 0.15(n-m)$. Assume that $\widehat{b}_k$ is known for each worker at iteration $k$, $\Delta_{\max}^k = 2\lambda_k = \frac{2\lambda}{\sqrt{n_k-m}}$ and $\delta = \widehat{\delta}_k$ is used to compute clipping parameter $\tau_l$ for \ref{eq:CenteredClip}. If the total number of iterations $T$ of \ref{eq:CenteredClip} satisfies $T \ge \log_{0.94}\frac{2\delta\sigma^2}{\EE[\|v^0 - \overline{g}^k\|^2]}$ and $\text{\algname{CheckAveraging}}$ is not triggered for any worker, then
     \begin{equation}
        \EE\left[\|\widehat{g}^k - \overline{g}^k\|^2\mid x^k\right] \le C_2\lambda^2\indicator_{k,v}, \label{eq:BTARD_overall_peer_unknown_num_byz_clipped_SGD}
    \end{equation}
    \begin{equation}
        \EE\left[\|\widehat{g}^k\|^2\mid x^k\right] \le  2C_2\lambda^2\indicator_{k, v} + 2G^\alpha \lambda^{2-\alpha}, \label{eq:BTARD_second_moment_unknown_num_byz_clipped_SGD}
    \end{equation}
    where $\overline{g}^k = \frac{1}{|\cG_k\setminus \cC_k|}\sum\limits_{j\in \cG_k\setminus \cC_k} g_j^k$, $C_2 = 4$, and $\indicator_{k, v}$ is an indicator function of the event that at least $1$ Byzantine peer violates the protocol at iteration $k$. 
\end{lemma}
\begin{proof}
    For all $i\in (\cG_k\cup\cB_k)\setminus \cC_k$ we have
    \begin{equation*}
        \EE\left[\|\widehat{g}^k(i) - \overline{g}^k(i)\|^2\mid x^k\right] \le (\Delta_{\max}^k)^2\indicator_{k, v} = 4\lambda_k^2\indicator_{k, v} = \frac{4\lambda^2}{n_k-m}\indicator_{k, v}
    \end{equation*}
    implying
    \begin{eqnarray*}
        \EE\left[\|\widehat{g}^k - \overline{g}^k\|^2\mid x^k\right] &=& \sum\limits_{i\in (\cG_k\cup\cB_k)\setminus \cC_k} \EE\left[\|\widehat{g}^k(i) - \overline{g}^k(i)\|^2\mid x^k\right] \\
        &\le& (n_k - m)\cdot \frac{4\lambda^2}{n_k-m}\indicator_{k, v} = C_2\lambda^2\indicator_{k, v}.
    \end{eqnarray*}
    Using \eqref{eq:clipping_second_moment} we obtain
    \begin{eqnarray*}
        \EE\left[\|\widehat{g}^k\|^2\mid x^k\right] &\le& 2\EE\left[\|\widehat{g}^k - \overline{g}^k\|^2\mid x^k\right] + 2\EE\left[\|\overline{g}^k\|^2\mid x^k\right]\\
        &\overset{\eqref{eq:BTARD_overall_peer_known_num_byz_clipped_SGD}}{\le}& 2C_2\lambda^2\indicator_{k, v} + 2\sum\limits_{i\in (\cG_k\cup\cB_k)\setminus \cC_k}\frac{G^\alpha \lambda_k^{2-\alpha}}{(n_k - m)^{\frac{\alpha}{2}}}= 2C_2\lambda^2\indicator_{k, v} + 2G^\alpha \lambda^{2-\alpha}.
    \end{eqnarray*}
\end{proof}

\subsubsection{Convex case}
In this section, we provide the complete statements and the full proofs of the convergence results for \algname{BTARD-Clipped-SGD} when the objective function $f$ is smooth and convex. We start with the case when the number of Byzantine peers violating the protocol $\widehat b_k$ is known at each iteration.

\begin{theorem}\label{thm:BTARD_Clipped_SGD_known_byz_cvx}
    Let As.~\ref{as:bounded_alpha_moment} hold, $Q$ is bounded, $f$ be convex, $x^*$ be some optimum of $f$, and $\nabla f(x^*) = 0$. Moreover, assume that $b \le 0.15(n-m)$, $m \le \nicefrac{(n-2b)}{2}$, and the exact number of attacking Byzantine peers is known to all good peers at each iteration. Next, assume that 
    \begin{equation}
        \gamma = \min\left\{\frac{R_0}{\sqrt{6}GK^{\frac{1}{\alpha}}}, \frac{m R_0}{12Gn\sqrt{10\delta(C_1 K^{\frac{4-\alpha}{2\alpha}} + C_2 K^{\frac{2}{\alpha}})}}\right\},\quad \Delta_{\max}^k = 2\lambda_k = \frac{2\lambda}{\sqrt{n_k-m}}, \label{eq:choice_of_parameters_BTARD_Clipped_SGD_cvx}
    \end{equation}
    \begin{equation}
        \lambda = GK^{\frac{1}{\alpha}}, \label{eq:choice_of_parameters_BTARD_Clipped_SGD_cvx_2}
    \end{equation}
    where $R_0 \ge \|x^0 - x^*\|$ and $\Delta_{\max}^k$ is the parameter for verification $3$ at iteration $k$ of \algname{BTARD-Clipped-SGD}. Then, we have $\EE[f(\overline{x}^{K}) -f(x^*)]\le \varepsilon$ after $K$ iterations of \algname{BTARD-Clipped-SGD}, where
    \begin{equation}
        K = \cO\left(\left(\frac{GR_0}{\varepsilon}\right)^{\frac{\alpha}{\alpha-1}} + \left(\frac{n\sqrt{\delta}GR_0}{m\varepsilon}\right)^{\frac{\alpha}{\alpha-1}}\right) \label{eq:BTARD_Clipped_SGD_known_byz_cvx}
    \end{equation}
    and $\overline{x}^K = \frac{1}{K}\sum_{k=0}^{K-1}$.
\end{theorem}
\begin{proof}
    Non-expansiveness of the projection operator and convexity of $f$ imply
    \begin{eqnarray*}
        \|x^{k+1} - x^*\|^2 &=&\left\|\text{proj}_{Q}(x^k - \gamma \widehat{g}^k) - \text{proj}_{Q}(x^*)\right\|^2\\
        &\le& \|x^k - x^* - \gamma \widehat{g}^k\|^2\\
        &=& \|x^k - x^*\|^2 -2\gamma \langle x^k - x^*, \widehat{g}^k \rangle + \gamma^2\|\widehat{g}^k\|^2\\
        &=& \|x^k - x^*\|^2 -2\gamma \langle x^k - x^*, \nabla f(x^k) \rangle -2\gamma \langle x^k - x^*, \widehat{g}^k -\nabla f(x^k) \rangle + \gamma^2\|\widehat{g}^k\|^2\\
        &\le& \|x^k - x^*\|^2 -2\gamma\left(f(x^k) - f(x^*)\right) -2\gamma \langle x^k - x^*, \widehat{g}^k -\nabla f(x^k) \rangle + \gamma^2\|\widehat{g}^k\|^2.
    \end{eqnarray*}
    Taking conditional expectation $\EE[\cdot\mid x^k]$ from the both sides of previous inequality we derive
    \begin{eqnarray*}
        \EE\left[\|x^{k+1} - x^*\|^2\mid x^k\right] &\le& \|x^k - x^*\|^2 -2\gamma\left(f(x^k) - f(x^*)\right) \\
        &&\quad -2\gamma \EE\left[\langle x^k - x^*, \widehat{g}^k -\nabla f(x^k) \rangle\mid x^k\right] + \gamma^2\EE\left[\|\widehat{g}^k\|^2\mid x^k\right]\\
        &\overset{\eqref{eq:BTARD_second_moment_known_num_byz_clipped_SGD}}{\le}& \|x^k - x^*\|^2 -2\gamma\left(f(x^k) - f(x^*)\right) + 2\gamma^2G^\alpha \lambda^{2-\alpha}\\
        &&\quad -2\gamma\left\langle x^k - x^*, \EE\left[\widehat{g}^k - \overline{g}^k\mid x^k\right] \right\rangle + 2\gamma^2\widehat{\delta}_k(C_1\lambda^{\frac{4-\alpha}{2}}G^{\frac{\alpha}{2}} + C_2\lambda^2)\\
        &=& \|x^k - x^*\|^2 -2\gamma\left(f(x^k) - f(x^*)\right) + 2\gamma^2G^2 K^{\frac{2-\alpha}{\alpha}}\\
        &&\quad -2\gamma\left\langle x^k - x^*, \EE\left[\widehat{g}^k - \overline{g}^k\mid x^k\right] \right\rangle + \frac{2\gamma^2 G^2(C_1K^{\frac{4-\alpha}{2\alpha}} + C_2K^{\frac{2}{\alpha}})}{n_k-m}\widehat{b}_k.
    \end{eqnarray*}
    To estimate the inner product in the right-hand side we apply Cauchy-Schwarz inequality:
    \begin{eqnarray*}
        -2\gamma\left\langle x^k - x^*, \EE\left[\widehat{g}^k - \overline{g}^k\mid x^k\right] \right\rangle &\le& 2\gamma \|x^k-x^*\|\cdot\left\|\EE\left[\widehat{g}^k - \overline{g}^k\mid x^k\right]\right\|\\
        &\le& 2\gamma \|x^k-x^*\|\EE\left[\|\widehat{g}^k - \overline{g}^k\|\mid x^k\right]\\
        &\le& 2\gamma \|x^k-x^*\|\sqrt{\EE\left[\|\widehat{g}^k - \overline{g}^k\|^2\mid x^k\right]}\\
        &\overset{\eqref{eq:BTARD_overall_peer_known_num_byz_clipped_SGD}}{\le}& 2\gamma\|x^k-x^*\|\sqrt{\widehat{\delta}_k(C_1\lambda^{\frac{4-\alpha}{2}}G^{\frac{\alpha}{2}} + C_2\lambda^2)}\\
        &=& \frac{2\gamma G\|x^k-x^*\|\sqrt{C_1 K^{\frac{4-\alpha}{2\alpha}}+C_2K^{\frac{2}{\alpha}}}}{\sqrt{n_k-m}}\sqrt{\widehat{b}_k}\\
        &\le& \frac{2\gamma G\|x^k-x^*\|\sqrt{20(C_1 K^{\frac{4-\alpha}{2\alpha}}+C_2K^{\frac{2}{\alpha}})}}{\sqrt{7n}}\sqrt{\widehat{b}_k},
    \end{eqnarray*}
    where in the last inequality we use $b \le 0.15(n-m)$, $m \le \nicefrac{(n-2b)}{2}$, $\gamma \le \nicefrac{1}{4L}$, $n_k-m \ge n-2b-m \ge \frac{7}{20}n$. Putting all together we obtain
    \begin{eqnarray*}
        \EE\left[\|x^{k+1} - x^*\|^2\mid x^k\right] &\le& \|x^k - x^*\|^2 -2\gamma\left(f(x^k) - f(x^*)\right) + 2\gamma^2G^2 K^{\frac{2-\alpha}{\alpha}}\\
        &&\quad +\frac{2\gamma G\|x^k-x^*\|\sqrt{20(C_1 K^{\frac{4-\alpha}{2\alpha}}+C_2K^{\frac{2}{\alpha}})}}{\sqrt{7n}}\sqrt{\widehat{b}_k}\\
        &&\quad + \frac{40\gamma^2 G^2(C_1K^{\frac{4-\alpha}{2\alpha}} + C_2K^{\frac{2}{\alpha}})}{7n}\widehat{b}_k.
    \end{eqnarray*}
    Taking the full expectation from the both sides of the above inequality and summing up the results for $k = 0,1,\ldots,T-1$ we derive
    \begin{eqnarray*}
        \frac{2\gamma}{T}\sum\limits_{k=0}^{T-1}\EE[f(x^k)-f(x^*)] &\le& \frac{1}{ T}\sum\limits_{k=0}^{T-1}\left(\EE\left[\|x^k - x^*\|^2\right] - \EE\left[\|x^{k+1} - x^*\|^2\right]\right) + 2\gamma^2G^2 K^{\frac{2-\alpha}{\alpha}}\\
        &&\quad +\frac{4\gamma G\sqrt{5(C_1 K^{\frac{4-\alpha}{2\alpha}}+C_2K^{\frac{2}{\alpha}})}}{\sqrt{n}T} \sum\limits_{k=0}^{T-1}\EE\left[\|x^k - x^*\|\sqrt{\widehat{b}_k}\right]\\
        &&\quad + \frac{40\gamma^2 G^2(C_1K^{\frac{4-\alpha}{2\alpha}} + C_2K^{\frac{2}{\alpha}})}{7nT}\sum\limits_{k=0}^{T-1}\EE[\widehat b_k]\\
        &\le& \frac{\|x^0-x^*\|^2 - \EE[\|x^{K}-x^*\|^2]}{K} + 2\gamma^2G^2 K^{\frac{2-\alpha}{\alpha}}\\
        &&\quad +\frac{4\gamma G\sqrt{5(C_1 K^{\frac{4-\alpha}{2\alpha}}+C_2K^{\frac{2}{\alpha}})}}{\sqrt{n}T} \sum\limits_{k=0}^{T-1}\sqrt{\EE\left[\|x^k - x^*\|^2\right]\EE\left[\widehat{b}_k\right]}\\
        &&\quad + \frac{40\gamma^2 G^2(C_1K^{\frac{4-\alpha}{2\alpha}} + C_2K^{\frac{2}{\alpha}})}{7nT}\sum\limits_{k=0}^{T-1}\EE[\widehat b_k].
    \end{eqnarray*}
    From Jensen's inequality we have $f(\overline{x}^T) \le \frac{1}{T}\sum_{k=0}^{T-1}f(x^k)$, where $\overline{x}^T = \frac{1}{T}\sum_{k=0}^{T-1} x^k$. Using this and new notation $R_k = \|x^k - x^*\|$, $k> 0$, $R_0 \ge \|x^0 - x^*\|$ we get
    \begin{eqnarray}
        0\le 2\gamma \EE\left[f(\overline{x}^T) - f(x^*)\right] &\le& \frac{R_0^2 - \EE[R_T^2]}{T} + 2\gamma^2G^2 K^{\frac{2-\alpha}{\alpha}}\notag\\
        &&\quad +\frac{4\gamma G\sqrt{5(C_1 K^{\frac{4-\alpha}{2\alpha}}+C_2K^{\frac{2}{\alpha}})}}{\sqrt{n}T} \sum\limits_{k=0}^{T-1}\sqrt{\EE\left[R_k^2\right]\EE\left[\widehat{b}_k\right]}\notag\\
        &&\quad + \frac{40\gamma^2 G^2(C_1K^{\frac{4-\alpha}{2\alpha}} + C_2K^{\frac{2}{\alpha}})}{7nT}\sum\limits_{k=0}^{T-1}\EE[\widehat b_k] \label{eq:technical_bound_cvx_known_BTARD_Clipped_SGD_0}
    \end{eqnarray}
    implying (after changing the indices) that
    \begin{eqnarray}
        \EE[R_k^2] &\le& R_0^2 + 2\gamma^2G^2 kK^{\frac{2-\alpha}{\alpha}} +\frac{4\gamma G\sqrt{5(C_1 K^{\frac{4-\alpha}{2\alpha}}+C_2K^{\frac{2}{\alpha}})}}{\sqrt{n}} \sum\limits_{l=0}^{k-1}\sqrt{\EE\left[R_l^2\right]\EE\left[\widehat{b}_l\right]}\notag\\
        &&\quad + \frac{40\gamma^2 G^2(C_1K^{\frac{4-\alpha}{2\alpha}} + C_2K^{\frac{2}{\alpha}})}{7n}\sum\limits_{l=0}^{k-1}\EE[\widehat b_l] \label{eq:technical_bound_cvx_known_BTARD_Clipped_SGD_1}
    \end{eqnarray}
    holds for all $k\ge 0$. In the remaining part of the proof we derive by induction that
    \begin{eqnarray}
        R_0^2 + 2\gamma^2G^2 kK^{\frac{2-\alpha}{\alpha}} +\frac{4\gamma G\sqrt{5(C_1 K^{\frac{4-\alpha}{2\alpha}}+C_2K^{\frac{2}{\alpha}})}}{\sqrt{n}} \sum\limits_{l=0}^{k-1}\sqrt{\EE\left[R_l^2\right]\EE\left[\widehat{b}_l\right]} &&\notag\\
        + \frac{40\gamma^2 G^2(C_1K^{\frac{4-\alpha}{2\alpha}} + C_2K^{\frac{2}{\alpha}})}{7n}\sum\limits_{l=0}^{k-1}\EE[\widehat b_l] &\le& 2R_0^2 \label{eq:technical_bound_cvx_known_BTARD_Clipped_SGD_2}
    \end{eqnarray}
    for all $k=0,\ldots,K$. For $k=0$ this inequality trivially holds. Next, assume that it holds for all $k=0,1,\ldots, T-1$, $T \le K-1$. Let us show that it holds for $k = T$ as well. From \eqref{eq:technical_bound_cvx_known_BTARD_SGD_1} and \eqref{eq:technical_bound_cvx_known_BTARD_SGD_2} we have that $\EE[R_k^2]\le 2R_0^2$ for all $k=0,1,\ldots, T-1$. Therefore,
    \begin{eqnarray*}
        \EE[R_T^2] &\le& R_0^2 + 2\gamma^2G^2 TK^{\frac{2-\alpha}{\alpha}} +\frac{4\gamma G\sqrt{5(C_1 K^{\frac{4-\alpha}{2\alpha}}+C_2K^{\frac{2}{\alpha}})}}{\sqrt{n}} \sum\limits_{l=0}^{T-1}\sqrt{\EE\left[R_l^2\right]\EE\left[\widehat{b}_l\right]}\notag\\
        &&\quad + \frac{40\gamma^2 G^2(C_1K^{\frac{4-\alpha}{2\alpha}} + C_2K^{\frac{2}{\alpha}})}{7n}\sum\limits_{l=0}^{T-1}\EE[\widehat b_l]\\
        &\le& R_0^2 + 2\gamma^2G^2 TK^{\frac{2-\alpha}{\alpha}} +\frac{4\gamma GR_0\sqrt{10(C_1 K^{\frac{4-\alpha}{2\alpha}}+C_2K^{\frac{2}{\alpha}})}}{\sqrt{n}} \sum\limits_{l=0}^{T-1}\sqrt{\EE\left[\widehat{b}_l\right]}\notag\\
        &&\quad + \frac{40\gamma^2 G^2(C_1K^{\frac{4-\alpha}{2\alpha}} + C_2K^{\frac{2}{\alpha}})}{7n}\sum\limits_{l=0}^{T-1}\EE[\widehat b_l]
    \end{eqnarray*}
    If a Byzantine peer deviates from the protocol at iteration $k$, it will be detected with some probability $p_k$ during the next iteration. One can lower bound this probability as
    \begin{equation*}
        p_k \ge m\cdot \frac{|\cG_k|}{n_k}\cdot \frac{1}{n_k} = \frac{m(1-\delta_k)}{n_k} \ge \frac{m}{n}.
    \end{equation*}
    Therefore, each individual Byzantine worker can violate the protocol no more than $\nicefrac{1}{p}$ times on average implying that
    \begin{eqnarray*}
        \EE[R_T^2]  &\le& R_0^2 + 2\gamma^2G^2 TK^{\frac{2-\alpha}{\alpha}} +\frac{4\gamma GR_0n\sqrt{10(C_1 K^{\frac{4-\alpha}{2\alpha}}+C_2K^{\frac{2}{\alpha}})b}}{m\sqrt{n}} \notag\\
        &&\quad + \frac{40\gamma^2 G^2(C_1K^{\frac{4-\alpha}{2\alpha}} + C_2K^{\frac{2}{\alpha}})nb}{7nm}\\
        &\overset{T\le K}{\le}& R_0^2 + 2\gamma^2G^2 K^{\frac{2}{\alpha}} +\frac{4\gamma GR_0n\sqrt{10(C_1 K^{\frac{4-\alpha}{2\alpha}}+C_2K^{\frac{2}{\alpha}})\delta}}{m} \notag\\
        &&\quad + \frac{40\gamma^2 G^2(C_1K^{\frac{4-\alpha}{2\alpha}} + C_2K^{\frac{2}{\alpha}})n\delta}{7m}.
    \end{eqnarray*}
    Taking
    \begin{equation*}
        \gamma = \min\left\{\frac{R_0}{\sqrt{6}GK^{\frac{1}{\alpha}}}, \frac{m R_0}{12Gn\sqrt{10\delta(C_1 K^{\frac{4-\alpha}{2\alpha}} + C_2 K^{\frac{2}{\alpha}})}}\right\}
    \end{equation*}
    we ensure that 
    \begin{eqnarray*}
        2\gamma^2G^2 K^{\frac{2}{\alpha}} +\frac{4\gamma GR_0n\sqrt{10(C_1 K^{\frac{4-\alpha}{2\alpha}}+C_2K^{\frac{2}{\alpha}})\delta}}{m}&&\\
        + \frac{40\gamma^2 G^2(C_1K^{\frac{4-\alpha}{2\alpha}} + C_2K^{\frac{2}{\alpha}})n\delta}{7m} &\le& \frac{R_0^2}{3} + \frac{R_0^2}{3} + \frac{R_0^2}{3} = R_0^2
    \end{eqnarray*}
    and, as a result, we get $\EE[R_T^2] \le 2R_0^2$. Therefore, \eqref{eq:technical_bound_cvx_known_BTARD_Clipped_SGD_2} holds for all $k=0,1,\ldots,K$. Together with \eqref{eq:technical_bound_cvx_known_BTARD_Clipped_SGD_0} it implies
    \begin{equation*}
        \EE\left[f(\overline{x}^K) - f(x^*)\right] \le \frac{R_0^2}{\gamma K}.
    \end{equation*}
    Next, from our stepsize rule \eqref{eq:choice_of_parameters_BTARD_Clipped_SGD_cvx} it follows that
    \begin{equation*}
        \EE\left[f(\overline{x}^K) - f(x^*)\right] = \cO\left(\frac{GR_0}{K^{\frac{1-\alpha}{\alpha}}} + \frac{n\sqrt{\delta}GR_0}{mK^{\frac{1-\alpha}{\alpha}}}\right)
    \end{equation*}
    meaning that after
    \begin{equation*}
        K = \cO\left(\left(\frac{GR_0}{\varepsilon}\right)^{\frac{\alpha}{\alpha-1}} + \left(\frac{n\sqrt{\delta}GR_0}{m\varepsilon}\right)^{\frac{\alpha}{\alpha-1}}\right)
    \end{equation*}
    iterations \algname{BTARD-Clipped-SGD} guarantees $\EE[f(\overline{x}^{K}) -f(x^*)]\le \varepsilon$.
\end{proof}

If there are no Byzantine peers ($\delta = 0$), the theorem establishes new result for the convergence of \algname{Clipped-SGD} for convex objectives. In the strongly convex case, the theorem recovers the rates that are optimal in this setting as shown in \citet{zhang2020why}. Next, when the number of attacking Byzantines is known at each iteration and $\nicefrac{n\sqrt{\delta}}{m} = \cO(1)$, the complexity bound is the same as in the case when $\delta = 0$. This means that the negative impact of Byzantine workers is negligible. Finally, the derived theoretical guarantees do not benefit from the increase of the total number of peers $n$. However, the result holds even for non-smooth problems and it is known that parallelization does not help to improve the complexity bounds in such generality. Nevertheless, our results show that \algname{BTARD-Clipped-SGD} provably converges to any predefined accuracy $\varepsilon > 0$. This is a property that the majority of previous methods does not have \citep{karimireddy2020learning}.

Next, we derive the result without assuming that $\widehat b^k$ is known to all peers at each iteration.

\begin{theorem}\label{thm:BTARD_Clipped_SGD_unknown_byz_cvx}
    Let As.~\ref{as:bounded_alpha_moment} hold, $Q$ is bounded, $f$ be convex, $x^*$ be some optimum of $f$, and $\nabla f(x^*) = 0$. Moreover, assume that $b \le 0.15(n-m)$, $m \le \nicefrac{(n-2b)}{2}$, and $\delta = 0$ is used to compute clipping parameter $\tau_l$ for \ref{eq:CenteredClip}. Next, assume that 
    \begin{equation}
        \gamma = \min\left\{\frac{R_0}{\sqrt{6}GK^{\frac{1}{\alpha}}}, \frac{m R_0}{12\sqrt{2C_2}G n b K^{\frac{1}{\alpha}}}\right\},\quad \Delta_{\max}^k = 2\lambda_k = \frac{2\lambda}{\sqrt{n_k-m}}, \label{eq:choice_of_parameters_BTARD_Clipped_SGD_cvx_unknown}
    \end{equation}
    \begin{equation}
        \lambda = GK^{\frac{1}{\alpha}}, \label{eq:choice_of_parameters_BTARD_Clipped_SGD_cvx_2_unknown}
    \end{equation}
    where $R_0 \ge \|x^0 - x^*\|$ and $\Delta_{\max}^k$ is the parameter for verification $3$ at iteration $k$ of \algname{BTARD-Clipped-SGD}. Then, we have $\EE[f(\overline{x}^{K}) -f(x^*)]\le \varepsilon$ after $K$ iterations of \algname{BTARD-Clipped-SGD}, where
    \begin{equation}
        K = \cO\left(\left(\frac{GR_0}{\varepsilon}\right)^{\frac{\alpha}{\alpha-1}} + \left(\frac{nbGR_0}{m\varepsilon}\right)^{\frac{\alpha}{\alpha-1}}\right) \label{eq:BTARD_Clipped_SGD_unknown_byz_cvx}
    \end{equation}
    and $\overline{x}^K = \frac{1}{K}\sum_{k=0}^{K-1}$.
\end{theorem}
\begin{proof}
    The proof is almost identical to the proof of Theorem~\ref{thm:BTARD_Clipped_SGD_known_byz_cvx}. Following the same steps and using \eqref{eq:BTARD_overall_peer_unknown_num_byz_clipped_SGD} and \eqref{eq:BTARD_second_moment_unknown_num_byz_clipped_SGD} instead of \eqref{eq:BTARD_overall_peer_known_num_byz_clipped_SGD} and \eqref{eq:BTARD_second_moment_known_num_byz_clipped_SGD} respectively we obtain the same sequence of inequalities up to the following change: instead of $\widehat{\delta}_k$ we should use $\indicator_{k,v}$. Therefore, we have
    \begin{eqnarray*}
        \EE\left[\|x^{k+1} - x^*\|^2\mid x^k\right] &\le& \|x^k - x^*\|^2 -2\gamma\left(f(x^k) - f(x^*)\right) + 2\gamma^2G^2 K^{\frac{2-\alpha}{\alpha}}\\
        &&\quad -2\gamma\left\langle x^k - x^*, \EE\left[\widehat{g}^k - \overline{g}^k\mid x^k\right] \right\rangle + 2\gamma^2C_2G^2K^{\frac{2}{\alpha}}\indicator_{k,v}.
    \end{eqnarray*}
    \begin{eqnarray*}
        -2\gamma\left\langle x^k - x^*, \EE\left[\widehat{g}^k - \overline{g}^k\mid x^k\right] \right\rangle &\le& 2\gamma G\|x^k-x^*\|\sqrt{C_2}K^{\frac{1}{\alpha}}\indicator_{k,v},
    \end{eqnarray*}
    and
    \begin{eqnarray*}
        \EE\left[\|x^{k+1} - x^*\|^2\mid x^k\right] &\le& \|x^k - x^*\|^2 -2\gamma\left(f(x^k) - f(x^*)\right) + 2\gamma^2G^2 K^{\frac{2-\alpha}{\alpha}}\\
        &&\quad + 2\gamma G\sqrt{C_2}K^{\frac{1}{\alpha}}\|x^k-x^*\|\indicator_{k,v} + 2\gamma^2C_2G^2K^{\frac{2}{\alpha}}\indicator_{k,v}.
    \end{eqnarray*}
    Taking the full expectation from the both sides of the above inequality and summing up the results for $k = 0,1,\ldots,T-1$ we derive
    \begin{eqnarray*}
        \frac{2\gamma}{T}\sum\limits_{k=0}^{T-1}\EE[f(x^k)-f(x^*)] &\le& \frac{1}{ T}\sum\limits_{k=0}^{T-1}\left(\EE\left[\|x^k - x^*\|^2\right] - \EE\left[\|x^{k+1} - x^*\|^2\right]\right) + 2\gamma^2G^2 K^{\frac{2-\alpha}{\alpha}}\\
        &&\quad +\frac{2\gamma G\sqrt{C_2}K^{\frac{1}{\alpha}}}{T} \sum\limits_{k=0}^{T-1}\EE\left[\|x^k - x^*\|\indicator_{k,v}\right] + \frac{2\gamma^2C_2G^2K^{\frac{2}{\alpha}}}{T}\sum\limits_{k=0}^{T-1}\EE[\indicator_{k,v}]\\
        &\le& \frac{\|x^0-x^*\|^2 - \EE[\|x^{K}-x^*\|^2]}{K} + 2\gamma^2G^2 K^{\frac{2-\alpha}{\alpha}}\\
        &&\quad +\frac{2\gamma G\sqrt{C_2}K^{\frac{1}{\alpha}}}{T} \sum\limits_{k=0}^{T-1}\sqrt{\EE\left[\|x^k - x^*\|^2\right]\EE\left[\indicator_{k,v}\right]}\\
        &&\quad + \frac{2\gamma^2C_2G^2K^{\frac{2}{\alpha}}}{T}\sum\limits_{k=0}^{T-1}\EE[\indicator_{k,v}].
    \end{eqnarray*}
    From Jensen's inequality we have $f(\overline{x}^T) \le \frac{1}{T}\sum_{k=0}^{T-1}f(x^k)$, where $\overline{x}^T = \frac{1}{T}\sum_{k=0}^{T-1} x^k$. Using this and new notation $R_k = \|x^k - x^*\|$, $k> 0$, $R_0 \ge \|x^0 - x^*\|$ we get
    \begin{eqnarray}
        0\le 2\gamma \EE\left[f(\overline{x}^T) - f(x^*)\right] &\le& \frac{R_0^2 - \EE[R_T^2]}{T} + 2\gamma^2G^2 K^{\frac{2-\alpha}{\alpha}}\notag\\
        &&\quad +\frac{2\gamma G\sqrt{C_2}K^{\frac{1}{\alpha}}}{T} \sum\limits_{k=0}^{T-1}\sqrt{\EE\left[R_k^2\right]\EE\left[\indicator_{k,v}\right]}\notag\\
        &&\quad + \frac{2\gamma^2C_2G^2K^{\frac{2}{\alpha}}}{T}\sum\limits_{k=0}^{T-1}\EE[\indicator_{k,v}] \label{eq:technical_bound_cvx_unknown_BTARD_Clipped_SGD_0}
    \end{eqnarray}
    implying (after changing the indices) that
    \begin{eqnarray}
        \EE[R_k^2] &\le& R_0^2 + 2\gamma^2G^2 kK^{\frac{2-\alpha}{\alpha}} + 2\gamma G\sqrt{C_2}K^{\frac{1}{\alpha}} \sum\limits_{l=0}^{k-1}\sqrt{\EE\left[R_l^2\right]\EE\left[\indicator_{l,v}\right]}\notag\\
        &&\quad + 2\gamma^2C_2G^2K^{\frac{2}{\alpha}}\sum\limits_{l=0}^{k-1}\EE[\indicator_{l,v}] \label{eq:technical_bound_cvx_unknown_BTARD_Clipped_SGD_1}
    \end{eqnarray}
    holds for all $k\ge 0$. In the remaining part of the proof we derive by induction that
    \begin{eqnarray}
       R_0^2 + 2\gamma^2G^2 kK^{\frac{2-\alpha}{\alpha}} + 2\gamma G\sqrt{C_2}K^{\frac{1}{\alpha}} \sum\limits_{l=0}^{k-1}\sqrt{\EE\left[R_l^2\right]\EE\left[\indicator_{l,v}\right]} &&\notag\\
        + 2\gamma^2C_2G^2K^{\frac{2}{\alpha}}\sum\limits_{l=0}^{k-1}\EE[\indicator_{l,v}] &\le& 2R_0^2 \label{eq:technical_bound_cvx_unknown_BTARD_Clipped_SGD_2}
    \end{eqnarray}
    for all $k=0,\ldots,K$. For $k=0$ this inequality trivially holds. Next, assume that it holds for all $k=0,1,\ldots, T-1$, $T \le K-1$. Let us show that it holds for $k = T$ as well. From \eqref{eq:technical_bound_cvx_unknown_BTARD_SGD_1} and \eqref{eq:technical_bound_cvx_unknown_BTARD_SGD_2} we have that $\EE[R_k^2]\le 2R_0^2$ for all $k=0,1,\ldots, T-1$. Therefore,
    \begin{eqnarray*}
        \EE[R_T^2] &\le& R_0^2 + 2\gamma^2G^2 TK^{\frac{2-\alpha}{\alpha}} + 2\gamma G\sqrt{C_2}K^{\frac{1}{\alpha}} \sum\limits_{l=0}^{T-1}\sqrt{\EE\left[R_l^2\right]\EE\left[\indicator_{l,v}\right]}\notag\\
        &&\quad + 2\gamma^2C_2G^2K^{\frac{2}{\alpha}}\sum\limits_{l=0}^{T-1}\EE[\indicator_{l,v}]\\
        &\le& R_0^2 + 2\gamma^2G^2 TK^{\frac{2-\alpha}{\alpha}} + 2\gamma GR_0\sqrt{2C_2}K^{\frac{1}{\alpha}} \sum\limits_{l=0}^{T-1}\sqrt{\EE\left[\indicator_{l,v}\right]}\notag\\
        &&\quad + 2\gamma^2C_2G^2K^{\frac{2}{\alpha}}\sum\limits_{l=0}^{T-1}\EE[\indicator_{l,v}]
    \end{eqnarray*}
    If a Byzantine peer deviates from the protocol at iteration $k$, it will be detected with some probability $p_k$ during the next iteration. One can lower bound this probability as
    \begin{equation*}
        p_k \ge m\cdot \frac{|\cG_k|}{n_k}\cdot \frac{1}{n_k} = \frac{m(1-\delta_k)}{n_k} \ge \frac{m}{n}.
    \end{equation*}
     That is, each individual Byzantine worker can violate the protocol no more than $\nicefrac{1}{p}$ times on average. However, even one Byzantine peer can create a shift of the order $\Delta_{\max}^k$ at each part of the resulting vector. Therefore, all Byzantine peers can violate the protocol no more than $\nicefrac{b}{p}$ times on average implying that
    \begin{eqnarray*}
        \EE[R_T^2]  &\le& R_0^2 + 2\gamma^2G^2 TK^{\frac{2-\alpha}{\alpha}} + \frac{2\gamma GR_0\sqrt{2C_2}K^{\frac{1}{\alpha}}nb}{m} + \frac{2\gamma^2C_2G^2K^{\frac{2}{\alpha}}nb}{m}.
    \end{eqnarray*}
    Taking
    \begin{equation*}
        \gamma = \min\left\{\frac{R_0}{\sqrt{6}GK^{\frac{1}{\alpha}}}, \frac{m R_0}{12\sqrt{2C_2}G n b K^{\frac{1}{\alpha}}}\right\}
    \end{equation*}
    we ensure that 
    \begin{eqnarray*}
        2\gamma^2G^2 TK^{\frac{2-\alpha}{\alpha}} + \frac{2\gamma GR_0\sqrt{2C_2}K^{\frac{1}{\alpha}}nb}{m} + \frac{2\gamma^2C_2G^2K^{\frac{2}{\alpha}}nb}{m}&\le& \frac{R_0^2}{3} + \frac{R_0^2}{3} + \frac{R_0^2}{3} = R_0^2
    \end{eqnarray*}
    and, as a result, we get $\EE[R_T^2] \le 2R_0^2$. Therefore, \eqref{eq:technical_bound_cvx_unknown_BTARD_Clipped_SGD_2} holds for all $k=0,1,\ldots,K$. Together with \eqref{eq:technical_bound_cvx_unknown_BTARD_Clipped_SGD_0} it implies
    \begin{equation*}
        \EE\left[f(\overline{x}^K) - f(x^*)\right] \le \frac{R_0^2}{\gamma K}.
    \end{equation*}
    Next, from our stepsize rule \eqref{eq:choice_of_parameters_BTARD_Clipped_SGD_cvx_unknown} it follows that
    \begin{equation*}
        \EE\left[f(\overline{x}^K) - f(x^*)\right] = \cO\left(\frac{GR_0}{K^{\frac{1-\alpha}{\alpha}}} + \frac{nbGR_0}{mK^{\frac{1-\alpha}{\alpha}}}\right)
    \end{equation*}
    meaning that after
    \begin{equation*}
        K = \cO\left(\left(\frac{GR_0}{\varepsilon}\right)^{\frac{\alpha}{\alpha-1}} + \left(\frac{nbGR_0}{m\varepsilon}\right)^{\frac{\alpha}{\alpha-1}}\right)
    \end{equation*}
    iterations \algname{BTARD-Clipped-SGD} guarantees $\EE[f(\overline{x}^{K}) -f(x^*)]\le \varepsilon$.
\end{proof}
That is, when the number of attacking Byzantines is unknown the complexity bound becomes $\left(\nicefrac{nb}{m}\right)^{\nicefrac{\alpha}{(\alpha-1)}}$ times worse in comparison to \eqref{eq:BTARD_Clipped_SGD_known_byz_cvx}.

\subsubsection{Strongly convex case: Restarted-BTARD-Clipped-SGD}
In this section, we provide the complete statements and the full proofs of the convergence results for the restarted version of \algname{BTARD-Clipped-SGD} (\algname{Restarted-BTARD-Clipped-SGD}, Alg.~\ref{alg:restarted_BTARD_SGD}) when the objective function $f$ is smooth and strongly convex.
\begin{algorithm}[H]
   \caption{\algname{Restarted-BTARD-Clipped-SGD}}
   \label{alg:restarted_BTARD_Clipped_SGD}
\begin{algorithmic}[1]
   \Require $x^0$ -- starting point, $r$ -- number of restarts, $\{\gamma_t\}_{t=1}^r$ -- stepsizes for \algname{BTARD-Clipped-SGD}, $\{K_t\}_{t=1}^r$ -- number of iterations for \algname{BTARD-Clipped-SGD}, $\{s_{i,k,t}\}_{i,k,t=0,0,1}^{n,K-1,r}$ -- seeds for batches computations, $\{\lambda_{k,t}\}_{k,t=0,1}^{K_t, r}$ -- gradient clipping parameters
   \State $\widehat{x}^{0} = x^0$
   \For{$t = 1,2,\ldots,r$}
    \State Run \algname{BTARD-Clipped-SGD} (Alg.~\ref{alg:BTARD_Clipped_SGD}) for $K_t$ iterations with stepsize $\gamma_t$, starting point $\widehat{x}^{t-1}$, gradient clipping parameters $\{\lambda_{k,t}\}_{k=0}^{K-1}$, and seeds for batches computations $\{s_{i,k,t}\}_{i,k=0,0}^{n,K-1}$. Define $\widehat{x}^{t}$ as $\widehat{x}^{t} = \frac{1}{K_t}\sum\limits_{k=0}^{K_t}x^{k,t}$, where $x^{0,t}, x^{1,t}, \ldots, x^{K_t,t}$ are the iterates produced by \algname{BTARD-Clipped-SGD}.
   \EndFor
   \Ensure $\widehat x^r$
\end{algorithmic}
\end{algorithm}

We start with the case when the number of attacking Byzantine workers is known at each iteration.
\begin{theorem}\label{thm:BTARD_Clipped_SGD_known_byz_str_cvx}
    Let As.~\ref{as:bounded_alpha_moment} hold, $Q$ is bounded, $f$ be $\mu$-strongly convex (see Def.~\ref{def:mu_strong_convexity}), $x^*$ be some optimum of $f$, and $\nabla f(x^*) = 0$. Moreover, assume that $b \le 0.15(n-m)$, $m \le \nicefrac{(n-2b)}{2}$, and the exact number of attacking Byzantine peers is known to all good peers at each iteration. Next, assume that 
    \begin{equation}
        \gamma = \min\left\{\frac{R_0}{\sqrt{6}\cdot 2^{\frac{t}{2}}GK_t^{\frac{1}{\alpha}}}, \frac{m R_0}{12\cdot 2^{\frac{t}{2}}Gn\sqrt{10\delta(C_1 K_t^{\frac{4-\alpha}{2\alpha}} + C_2 K_t^{\frac{2}{\alpha}})}}\right\},\quad \Delta_{\max}^{k,t} = 2\lambda_{k,t} = \frac{2\lambda_t}{\sqrt{n_k^t-m}}, \label{eq:choice_of_parameters_BTARD_Clipped_SGD_str_cvx}
    \end{equation}
    \begin{equation}
        K_t = \max\left\{\left(\frac{2\sqrt{6}G\cdot 2^{\frac{t}{2}}}{\mu R_0}\right)^{\frac{\alpha}{\alpha-1}}, \left(\frac{24Gn\sqrt{10\delta(C_1+C_2)}2^{\frac{t}{2}}}{m\mu R_0}\right)^{\frac{\alpha}{\alpha-1}}\right\}, \quad \lambda_t = GK_t^{\frac{1}{\alpha}}, \label{eq:choice_of_parameters_BTARD_Clipped_SGD_str_cvx_2}
    \end{equation}
    \begin{equation}
        r = \left\lceil\log_2\frac{\mu R_0^2}{\varepsilon} \right\rceil - 1, \label{eq:choice_of_parameters_BTARD_Clipped_SGD_str_cvx_3}
    \end{equation}
    where $R_0 \ge \|x^0 - x^*\|$ and $\Delta_{\max}^{k,t}$ is the parameter for verification $3$ at iteration $k$ of \algname{BTARD-Clipped-SGD}, $n_k^t$ is the total number of workers at iteration $k$ of $t$-th restart. Then, we have $\EE[f(\widehat{x}^{r}) -f(x^*)]\le \varepsilon$ after $r$ restarts of \algname{BTARD-Clipped-SGD} and the total number of executed iterations of \algname{BTARD-Clipped-SGD} is
    \begin{equation}
        \sum\limits_{t=1}^rK_t  = \cO\left(\left(\frac{G^2}{\mu\varepsilon}\right)^{\frac{\alpha}{2(\alpha-1})} + \left(\frac{n\sqrt{\delta}}{m}\right)^{\frac{\alpha}{\alpha-1}}\left(\frac{G^2}{\mu\varepsilon}\right)^{\frac{\alpha}{2(\alpha-1)}}\right) \label{eq:BTARD_Clipped_SGD_known_byz_str_cvx}
    \end{equation}
\end{theorem}
\begin{proof}
    Theorem~\ref{thm:BTARD_Clipped_SGD_known_byz_cvx} implies that \algname{BTARD-Clipped-SGD} with 
    \begin{equation*}
        \gamma = \min\left\{\frac{R_0}{\sqrt{6}GK^{\frac{1}{\alpha}}}, \frac{m R_0}{12Gn\sqrt{10\delta(C_1 K^{\frac{4-\alpha}{2\alpha}} + C_2 K^{\frac{2}{\alpha}})}}\right\}
    \end{equation*}
    guarantees
    \begin{equation*}
        \EE\left[f(\overline{x}^K) - f(x^*)\right] \le \frac{R_0^2}{\gamma K}
    \end{equation*}
    after $K$ iterations. Therefore, after the first restart we have
    \begin{equation}
        \EE[f(\widehat{x}^1) - f(x^*)] \le \frac{R_0^2}{\gamma_1 K_1} \le \frac{\mu R_0^2}{4}. \notag
    \end{equation}
    From $\mu$-strong convexity of $f$ and $\nabla f(x^*) = 0$ we have
    \begin{equation*}
        \frac{\mu}{2}\|\widehat x^1 - x^*\|^2 \le f(\widehat x^1) - f(x^*) \Longrightarrow \EE[\|\widehat x^1 - x^*\|^2] \le \frac{R_0^2}{2}.
    \end{equation*}
    Next, assume that we have $\EE[f(\widehat{x}^{t}) - f(x^*)] \le \frac{\mu R_0^2}{2^{t+1}}$, $\EE[\|\widehat x^t - x^*\|^2] \le \frac{R_0^2}{2^t}$ for some $t \le r-1$. Then, Theorem~\ref{thm:BTARD_Clipped_SGD_known_byz_cvx} implies that
    \begin{equation*}
        \EE[f(\widehat{x}^{t+1}) - f(x^*)\mid x^t] \le \frac{\|\widehat{x}^t - x^*\|^2}{\gamma_t K_t}.
    \end{equation*}
    Taking the full expectation from the both sides of previous inequality we get
    \begin{equation*}
        \EE[f(\widehat{x}^{t+1}) - f(x^*)] \le \frac{\EE[\|\widehat{x}^t - x^*\|^2]}{\gamma_t K_t} \le \frac{R_0^2}{2^t\gamma_t K_t} \le \frac{\mu R_0^2}{2^{t+2}}.
    \end{equation*}
    From $\mu$-strong convexity of $f$ and $\nabla f(x^*) = 0$ we have
    \begin{equation*}
        \frac{\mu}{2}\|\widehat x^{t+1} - x^*\|^2 \le f(\widehat x^{t+1}) - f(x^*) \Longrightarrow \EE[\|\widehat x^{t+1} - x^*\|^2] \le \frac{R_0^2}{2^{t+1}}.
    \end{equation*}
    Therefore, by mathematical induction we have that for all $t=1,\ldots,r$
    \begin{equation*}
        \EE[f(\widehat{x}^t) - f(x^*)] \le \frac{\mu R_0^2}{2^{t+1}}, \quad \EE\left[\|\widehat x^t - x^*\|^2\right] \le \frac{R_0^2}{2^t}.
    \end{equation*}
    Then, after $r = \left\lceil\log_2\frac{\mu R_0^2}{\varepsilon} \right\rceil - 1$ restarts of \algname{BTARD-Clipped-SGD} we have $\EE[f(\widehat{x}^r) - f(x^*)] \le \varepsilon$. The total number of iterations executed by \algname{BTARD-Clipped-SGD} is
    \begin{eqnarray*}
        \sum\limits_{t=1}^r K_t &=& \cO\left(\sum\limits_{t=1}^r\max\left\{\left(\frac{G\cdot 2^{\frac{t}{2}}}{\mu R_0}\right)^{\frac{\alpha}{\alpha-1}}, \left(\frac{Gn\sqrt{\delta}2^{\frac{t}{2}}}{m\mu R_0}\right)^{\frac{\alpha}{\alpha-1}}\right\}\right)\\
        &=& \cO\left(\max\left\{\left(\frac{G}{\mu R_0}\right)^{\frac{\alpha}{\alpha-1}}\cdot 2^{\frac{r\alpha}{2(\alpha-1)}}, \left(\frac{Gn\sqrt{\delta}}{m\mu R_0}\right)^{\frac{\alpha}{\alpha-1}}\cdot 2^{\frac{r\alpha}{2(\alpha-1)}}\right\}\right)\\
        &=& \cO\left(\max\left\{\left(\frac{G}{\mu R_0}\right)^{\frac{\alpha}{\alpha-1}}\cdot \left(\frac{\mu R_0^2}{\varepsilon}\right)^{\frac{\alpha}{2(\alpha-1)}}, \left(\frac{Gn\sqrt{\delta}}{m\mu R_0}\right)^{\frac{\alpha}{\alpha-1}}\cdot \left(\frac{\mu R_0^2}{\varepsilon}\right)^{\frac{\alpha}{2(\alpha-1)}}\right\}\right)\\
        &=& \cO\left(\left(\frac{G^2}{\mu\varepsilon}\right)^{\frac{\alpha}{2(\alpha-1})} + \left(\frac{n\sqrt{\delta}}{m}\right)^{\frac{\alpha}{\alpha-1}}\left(\frac{G^2}{\mu\varepsilon}\right)^{\frac{\alpha}{2(\alpha-1)}}\right).
    \end{eqnarray*}
\end{proof}

In the strongly convex case, similar observations hold as in the convex case. Next, we derive the result without assuming that $\widehat b^k$ is known to all peers at each iteration.

\begin{theorem}\label{thm:BTARD_Clipped_SGD_unknown_byz_str_cvx}
    Let As.~\ref{as:bounded_alpha_moment} hold, $Q$ is bounded, $f$ be $\mu$-strongly convex (see Def.~\ref{def:mu_strong_convexity}), $x^*$ be some optimum of $f$, and $\nabla f(x^*) = 0$. Moreover, assume that $b \le 0.15(n-m)$, $m \le \nicefrac{(n-2b)}{2}$, and $\delta = 0$ is used to compute clipping parameter $\tau_l$ for \ref{eq:CenteredClip}. Next, assume that 
    \begin{equation}
        \gamma = \min\left\{\frac{R_0}{\sqrt{6}\cdot 2^{\frac{t}{2}}GK_t^{\frac{1}{\alpha}}}, \frac{m R_0}{12\cdot 2^{\frac{t}{2}}Gnb\sqrt{2C_2}K_t^{\frac{1}{\alpha}}}\right\},\quad \Delta_{\max}^{k,t} = 2\lambda_{k,t} = \frac{2\lambda_t}{\sqrt{n_k^t-m}}, \label{eq:choice_of_parameters_BTARD_Clipped_SGD_str_cvx_unknown}
    \end{equation}
    \begin{equation}
        K_t = \max\left\{\left(\frac{2\sqrt{6}G\cdot 2^{\frac{t}{2}}}{\mu R_0}\right)^{\frac{\alpha}{\alpha-1}}, \left(\frac{24Gnb\sqrt{2C_2}2^{\frac{t}{2}}}{m\mu R_0}\right)^{\frac{\alpha}{\alpha-1}}\right\}, \quad \lambda_t = GK_t^{\frac{1}{\alpha}}, \label{eq:choice_of_parameters_BTARD_Clipped_SGD_str_cvx_unknown_2}
    \end{equation}
    \begin{equation}
        r = \left\lceil\log_2\frac{\mu R_0^2}{\varepsilon} \right\rceil - 1, \label{eq:choice_of_parameters_BTARD_Clipped_SGD_str_cvx_unknown_3}
    \end{equation}
    where $R_0 \ge \|x^0 - x^*\|$ and $\Delta_{\max}^{k,t}$ is the parameter for verification $3$ at iteration $k$ of \algname{BTARD-Clipped-SGD}, $n_k^t$ is the total number of workers at iteration $k$ of $t$-th restart. Then, we have $\EE[f(\widehat{x}^{r}) -f(x^*)]\le \varepsilon$ after $r$ restarts of \algname{BTARD-Clipped-SGD} and the total number of executed iterations of \algname{BTARD-Clipped-SGD} is
    \begin{equation}
        \sum\limits_{t=1}^rK_t  = \cO\left(\left(\frac{G^2}{\mu\varepsilon}\right)^{\frac{\alpha}{2(\alpha-1})} + \left(\frac{nb}{m}\right)^{\frac{\alpha}{\alpha-1}}\left(\frac{G^2}{\mu\varepsilon}\right)^{\frac{\alpha}{2(\alpha-1)}}\right) \label{eq:BTARD_Clipped_SGD_unknown_byz_str_cvx}
    \end{equation}
\end{theorem}
\begin{proof}
    Theorem~\ref{thm:BTARD_Clipped_SGD_unknown_byz_cvx} implies that \algname{BTARD-Clipped-SGD} with 
    \begin{equation*}
        \gamma = \min\left\{\frac{R_0}{\sqrt{6}GK^{\frac{1}{\alpha}}}, \frac{m R_0}{12\sqrt{2C_2}G n b K^{\frac{1}{\alpha}}}\right\}
    \end{equation*}
    guarantees
    \begin{equation*}
        \EE\left[f(\overline{x}^K) - f(x^*)\right] \le \frac{R_0^2}{\gamma K}
    \end{equation*}
    after $K$ iterations. Therefore, after the first restart we have
    \begin{equation}
        \EE[f(\widehat{x}^1) - f(x^*)] \le \frac{R_0^2}{\gamma_1 K_1} \le \frac{\mu R_0^2}{4}. \notag
    \end{equation}
    From $\mu$-strong convexity of $f$ and $\nabla f(x^*) = 0$ we have
    \begin{equation*}
        \frac{\mu}{2}\|\widehat x^1 - x^*\|^2 \le f(\widehat x^1) - f(x^*) \Longrightarrow \EE[\|\widehat x^1 - x^*\|^2] \le \frac{R_0^2}{2}.
    \end{equation*}
    Next, assume that we have $\EE[f(\widehat{x}^{t}) - f(x^*)] \le \frac{\mu R_0^2}{2^{t+1}}$, $\EE[\|\widehat x^t - x^*\|^2] \le \frac{R_0^2}{2^t}$ for some $t \le r-1$. Then, Theorem~\ref{thm:BTARD_Clipped_SGD_unknown_byz_cvx} implies that
    \begin{equation*}
        \EE[f(\widehat{x}^{t+1}) - f(x^*)\mid x^t] \le \frac{\|\widehat{x}^t - x^*\|^2}{\gamma_t K_t}.
    \end{equation*}
    Taking the full expectation from the both sides of previous inequality we get
    \begin{equation*}
        \EE[f(\widehat{x}^{t+1}) - f(x^*)] \le \frac{\EE[\|\widehat{x}^t - x^*\|^2]}{\gamma_t K_t} \le \frac{R_0^2}{2^t\gamma_t K_t} \le \frac{\mu R_0^2}{2^{t+2}}.
    \end{equation*}
    From $\mu$-strong convexity of $f$ and $\nabla f(x^*) = 0$ we have
    \begin{equation*}
        \frac{\mu}{2}\|\widehat x^{t+1} - x^*\|^2 \le f(\widehat x^{t+1}) - f(x^*) \Longrightarrow \EE[\|\widehat x^{t+1} - x^*\|^2] \le \frac{R_0^2}{2^{t+1}}.
    \end{equation*}
    Therefore, by mathematical induction we have that for all $t=1,\ldots,r$
    \begin{equation*}
        \EE[f(\widehat{x}^t) - f(x^*)] \le \frac{\mu R_0^2}{2^{t+1}}, \quad \EE\left[\|\widehat x^t - x^*\|^2\right] \le \frac{R_0^2}{2^t}.
    \end{equation*}
    Then, after $r = \left\lceil\log_2\frac{\mu R_0^2}{\varepsilon} \right\rceil - 1$ restarts of \algname{BTARD-Clipped-SGD} we have $\EE[f(\widehat{x}^r) - f(x^*)] \le \varepsilon$. The total number of iterations executed by \algname{BTARD-Clipped-SGD} is
    \begin{eqnarray*}
        \sum\limits_{t=1}^r K_t &=& \cO\left(\sum\limits_{t=1}^r\max\left\{\left(\frac{G\cdot 2^{\frac{t}{2}}}{\mu R_0}\right)^{\frac{\alpha}{\alpha-1}}, \left(\frac{Gn\sqrt{\delta}2^{\frac{t}{2}}}{m\mu R_0}\right)^{\frac{\alpha}{\alpha-1}}\right\}\right)\\
        &=& \cO\left(\max\left\{\left(\frac{G}{\mu R_0}\right)^{\frac{\alpha}{\alpha-1}}\cdot 2^{\frac{r\alpha}{2(\alpha-1)}}, \left(\frac{Gnb}{m\mu R_0}\right)^{\frac{\alpha}{\alpha-1}}\cdot 2^{\frac{r\alpha}{2(\alpha-1)}}\right\}\right)\\
        &=& \cO\left(\max\left\{\left(\frac{G}{\mu R_0}\right)^{\frac{\alpha}{\alpha-1}}\cdot \left(\frac{\mu R_0^2}{\varepsilon}\right)^{\frac{\alpha}{2(\alpha-1)}}, \left(\frac{Gnb}{m\mu R_0}\right)^{\frac{\alpha}{\alpha-1}}\cdot \left(\frac{\mu R_0^2}{\varepsilon}\right)^{\frac{\alpha}{2(\alpha-1)}}\right\}\right)\\
        &=& \cO\left(\left(\frac{G^2}{\mu\varepsilon}\right)^{\frac{\alpha}{2(\alpha-1})} + \left(\frac{nb}{m}\right)^{\frac{\alpha}{\alpha-1}}\left(\frac{G^2}{\mu\varepsilon}\right)^{\frac{\alpha}{2(\alpha-1)}}\right).
    \end{eqnarray*}
\end{proof}

\section{Resisting Sybil attacks}\label{appendix:reputation}

In this section, we address Byzantine-tolerant training in a setup where new participants can join or leave collaboration midway through training. This requirement arises naturally if a training run relies on volunteers or an open pool of paid participants~\citep{volunteer_dl_async,hivemind_dmoe,atre2021distributed,dedloc}. In addition to all existing concerns from Section~\ref{sect:method}, this setup allows Byzantine attackers to assume new identity each time they are blocked. Further yet, Byzantine participants can simultaneously use multiple identities in order to obtain majority in the voting procedure, which is known as Sybil attacks~\citep{sybil, sybil_nodes,sybil_attacks_dht}.

In this analysis\footnote{Note that we only provide rigorous convergence guarantees for the case of the Byzantine attacks. However, a heuristic described in this section helps with resisting the Sybil attacks in practice.}, we consider a training run where Byzantine peers collectively possess $\delta < \delta_{max}$ of all compute resources (we explore the role of $\delta_{max} < 1/2$ later in this section). Intuitively, one can think of this setting as distributed training with $n$ identical computers, $\lfloor\delta\cdot n\rfloor$ of which are controlled by Byzantines. The ``Byzantine GPUs'' can be allocated between an arbitrary number of identities. For instance, one accelerator can run full BTARD-SGD protocol for one peer or drop some of the computation and use the freed ``compute cycles'' to run computation for another participant. Theoretically, a device can run computation for an arbitrarily large number of peers, as long as it actually computes as many gradients as one benign participant does in the same time-frame.

To protect against this new attack type, we augment \algname{BTARD-SGD} with a reputation system designed to limit the impact of pseudonymous identities with the actual underlying compute. We base this system on the following assumptions:
\begin{enumerate}[leftmargin=*]
    \vspace{-1px}
    \item \textbf{Unique and optimal computations:} the gradients computed by peer $i$ at step $k$ cannot be circumvented or reused from other peers and/or previous steps.
    \vspace{-1px}
    \item \textbf{Sybil-resistance of the message propagation protocol:} the underlying message propagation protocol should be resistant to various forms of Sybil attacks: they should not harm the protocol's ability to deliver a broadcasted message to all honest peers. \citet{vyzovitis2020gossipsub} empirically show that GossipSub used in \algname{BTARD-SGD} is resistant to such attacks.
    \vspace{-1px}
    \item \textbf{Usage of digital signatures:} peers should have unique public/private key pairs, know each other's public keys, sign their messages~\citep{rivest1978method}, and ignore all received messages with invalid signatures.
    \vspace{-1px}
    \item \textbf{Existence of a cryptographic hash function:} peers should have access to a hash function such that finding a vector $x$ satisfying $\text{hash}(x)=y$ is infeasible for $\lfloor\delta\cdot n\rfloor$ compute over the entire training duration.
\end{enumerate}

We associate each participant with a public record that is used to verify that peer's legitimacy. These records can be securely stored in a Distributed Hash Table (see Appendix~\ref{appendix:dht}). When a new peer joins the network, it begins with an empty record and is therefore ``untrusted''. Untrusted peers compute gradients normally, but cannot aggregate vectors from others and cannot serve as validators. More importantly, other peers exclude untrusted gradients from aggregation, using them only for the purpose of validating those peers.

Each time a peer computes gradients $g_i^k$ over publicly known batch $\xi_i^k$, it must write $\text{hash}(g_i^k)$ to its own public record and sign it with its private key. As in the original \algname{BTARD-SGD}, some of those entries will be validated by other peers chosen by MPRNG. In turn, the chosen validators will either approve their entry or invoke \textsc{Accuse} to ban the peer.

In order to become trusted, a given peer must report consecutive gradients until it accumulates $T$ entries approved by (provably) random peers. Here, $T$ is a hyperparameter that should be large enough for the training to recover from any previous attacks and make some progress before previously banned malicious peers can earn trust again. In practice, $T$ may be chosen experimentally by observing the number of iterations it takes to improve the loss upon its pre-attack value in case of the most effective attacks, as reported in Section~\ref{sect:experiments}.

While $T$ may be application-dependent, we note that its minimal value is small in terms of the relative training time in all our experiments. $T$ corresponding to the 10\% of total training time is more than 3 times larger than the worst ``recovery time'' for both setups considered in Section~\ref{sect:experiments}, where almost a half of the peers are Byzantine. Moreover, Appendix~\ref{appendix:resnet18_extra} suggests that recovery from the worst-case attack may happen even faster in case of a smaller share of Byzantines. In that setup (with $\approx 20\%$ of peers being Byzantine), $T$~corresponding to the 1\% of training time is already enough.

Once a peer becomes trusted, it must continue reporting gradient hashes to maintain trust. Even a single missing or invalidated hash breaks the chain and results in the corresponding peer being banned. To maintain this invariant, peers chosen as a validators add the recalculated hashes into their own record instead of the skipped iteration.

To protect against dilution attacks, a cooperative training run can simultaneously consider at most $\frac{t}{2}$ ``untrusted'' peers, where $t$ is the number of currently trusted peers. All subsequent peers should wait in a queue until one of the untrusted peers becomes either trusted or banned.

\paragraph{Analysis.} Under this formalism, a Sybil attacker will attempt to maximize the number of trusted identities it can control with a limited amount of compute. In the simplest case, an attacker has exactly one GPU that can be used to either run all computations for identity or partial computation for multiple identities.

In the latter case, an attacker can honestly compute gradients for identity A with probability $p\in[0, 1]$ and for identity B with probability $1-p$. To breaking the chain, the identity that does \textbf{not} compute gradients at a given step can report arbitrary (e.g. random) entries instead of $\text{hash}(g_i^k)$.

Consider the expected number of ``trusted'' identities after enough steps for $T$ validations by \textit{honest} validators (on average, $T \cdot \frac{n}{k \cdot (1-\delta)}$ steps). Identity A becomes trusted with probability $p^T$, otherwise it is banned. Similarly, identitiy B survives with probability $(1 - p)^T$. Thus, the expected number of trusted identities after $T$ steps is $p^T + (1 - p)^T$.

For $T > 1$, this expectation is maximal iff $p \in \{0, 1\}$. Thus, if a peer needs more than one validation to become trusted, the ``optimal strategy'' for a Sybil attacker is to fully support one identity instead of spreading the resources between multiple ones. This observation can be generalized for distributing $\lfloor\delta\cdot n\rfloor$ over an $m \geq \lfloor\delta\cdot n\rfloor$ pseudonymous identities, where maximizing the expected number of trusted identities requires fully supporting any $\lfloor\delta\cdot n\rfloor$ identities and disregarding the rest (for $T>1$, as before).

\paragraph{Overhead computation.} When training without Byzantine participants, this modified version of \algname{BTARD-SGD} requires, on average, $T \cdot \frac{n}{k}$ additional gradient computations per participant at the very beginning. However, once all peers become trusted, the algorithm computes exactly the same number of gradients as regular \algname{BTARD-SGD}, effectively training at $\frac{n-k}{n}$ efficiency of AR-SGD, plus the same communication overhead.

\paragraph{Remark 1: Temporary majority.} Despite the fact that spreading 1 ``compute unit'' across multiple identities reduces the \textit{expected} number of trusted identities, it may still be useful to establish a temporary majority, albeit with a small probability. For instance, splitting one compute unit evenly among $m$ identities (each with $p = 1/m$) may result in both $m$ identities temporarily gaining trust with probability:
\begin{equation}
    P(\text{peer}_1\land\dots\land\text{peer}_m) = \prod\limits_{i=1}^m \frac{1}{m^T} = m^{-Tm}
    \label{eq:p_temp_majority}
\end{equation} 

A Sybil attacker can simply repeat this procedure on every step until it can establish a temporary majority and use this majority to harm training (e.g. ban non-malicious peers). A natural way to remedy this is to increase $T$ to such an extent that~\eqref{eq:p_temp_majority} becomes negligibly small.

\paragraph{Remark 2: Extra compute for Byzantine nodes.} Unlike benign peers, Byzantine attackers do not need to honestly validate each other. When a Byzantine peer is chosen as validator, it can approve its target without actually computing the gradients. In turn, the freed compute resources can be used to support additional Byzantine identities.

Thus, if a given training run has $n$ trusted peers and chooses $k$ validators on each step, Sybil attackers can control slightly more than $\lfloor\delta\cdot n\rfloor$ of all identities by using the free compute cycles from validation to support additional peers. Thus, the proposed reputation system requires that the total computational power $B_{max}$ available to Byzantines is less than $\frac{1}{2}$ by a (typically small) margin that depends on $n$, $k$, and $T$.

\paragraph{Remark 3: Perpetual attacks.} When training in open collaborations, one cannot ban the Byzantine peers entirely: a Byzantine attacker will always be able to assume a new identity at the cost of running honestly for $T \cdot \frac{n}{k \cdot (1-\delta)}$ gradient steps. Thus, unlike in Appendix~\ref{appendix:extra_analysis}, we cannot make \algname{BTARD-SGD} unbiased by increasing $\tau$. However, as we demonstrated in Section~\ref{sect:experiments}, the biased variant of \algname{BTARD-SGD} with constant $\tau$ can still train real-world deep learning models with the same or virtually the same learning curves as regular SGD.

\section{Secure distributed hash tables}\label{appendix:dht}

Distributed Hash Tables (DHT) are protocols that establish a decentralized key-value storage over decentralized unreliable participants~\citep{kademlia,chord,tapestry,pastry}.
To determine which DHT peers are responsible for a given key-value pair, each participant samples a unique binary identifier (ID) sampled uniformly from the space of $\mathrm{hash}$ function outputs. When ``storing a $(key,\ value)$'' on the DHT, one finds $k$ peers whose IDs are nearest to $\mathrm{hash}(key)$ and sends the data to each one of those peers. In turn, a peer that wants to read the value or a given key will also search for neighbors whose IDs are close to $\mathrm{hash}(key)$ and request the data from those peers. Thus, the data can be accessed as long as at least one o $k$ chosen peers remains active, with some DHT variants introducing additional replication protocols.

Our specific implementation is based on Kademlia~\citep{kademlia}, a popular DHT variant that determines nearest neighbors based on XOR distance function or their IDs: $d(x, y) = \mathrm{int}(x \oplus y)$. More importantly, Kademlia protocol organizes nodes in such a way that each individual peer only ``knows'' a small subset of $O(\log_2 n)$ direct neighbors, however, it is possible to navigate the neighborhood graph to find the globally nearest neighbors in $O(\log_2 N)$ network requests.

DHT protocols were originally designed for large-scale distributed systems such as BitTorrent, IPFS and several cryptocurrencies. To maintain integrity in these applications, modern DHT protocols also employ security measures that make them resistant to Byzantine and Sybil attacks~\citep{urdaneta2011survey}. 

In our specific scenario, the most sensitive DHT entries are personal records that determine whether or not a given peer is trusted. We protect these records by enforcing that every value stored in the DHT must be signed by their author's digital signature~\citep{rivest1978method}. Thus, if a malicious peer attempts to modify a record it was not supposed to, all other peers will be able to detect that and eliminate such peers from the collective.

However, digital signature are known to be vulnerable to replay attacks: every time a non-Byzantine peer stores an given key-value pair signed with its private key, a Byzantine eavesdropper can record the signed entry and replay it in future. For ordinary DHTs, this would allow an attacker to revert any key-value pair to its previous state by replaying such pre-recorded messages. 

Our algorithm protects against replay attacks by associating each key-value pair with a third value denoted as \textbf{expiration time}. Given two entries for the same key, DHT nodes will now prioritize the ones with the latest expiration time and consider it valid up to that time. Furthermore, in order to store a new entry to the DHT, a peer must now sign the entire key-value-expiration tuple. Thus, if a Byzantine peer replays a pre-recorded message, it will not be able to overwrite newer DHT entries that were signed for a more recent expiration time.

\section{Details of the ALBERT experiment setup}\label{appendix:config_albert}

In Section~\ref{sect:experiments_albert}, we pretrain ALBERT~\citep{albert} --- a self-supervised Transformer model for learning representations of language data. We deliberately choose ALBERT instead of other models like BERT~\citep{bert} due to its high communication efficiency, which is caused by layerwise weight sharing and embedding layer factorization. In particular, we focus on a communication-efficient model, because the connection speed between the workers can become a noticeable constraint when averaging gradients of models with hundreds of millions of parameters. We train ALBERT-large on sequences of 512 tokens from the WikiText-103~\citep{wikitext103} dataset. The training procedure starts from a random initialization, but the subword vocabulary~\citep{sennrich-etal-2016-neural} is the same as created by the authors of the original ALBERT models.  

This model is trained with two objectives: masked language modeling (given a sentence with several masked tokens, predict the tokens that were masked) and sentence order prediction (given two segments from the same document, determine if they were swapped).
We use LAMB optimizer~\citep{lamb} with batches that contain 4,096 examples, training with a peak learning rate equal to 0,00176 and a warmup of 5,000 gradient descent steps. In addition, we use gradient clipping with a maximum norm of 1 and weight decay regularization with the weight of 0,01.
We run distributed training on 16 cloud instances, each equipped with a single Tesla T4 GPU. Each training run takes 2--3 days, depending on the instance availability.

\section{Additional experiments}\label{appendix:additional_experiments}

\subsection{Extra evaluations on the CIFAR10 classification task}\label{appendix:resnet18_extra}

\begin{figure}[tb]
    \vskip 0.05in
    \centering
    \includegraphics[width=0.925\textwidth]{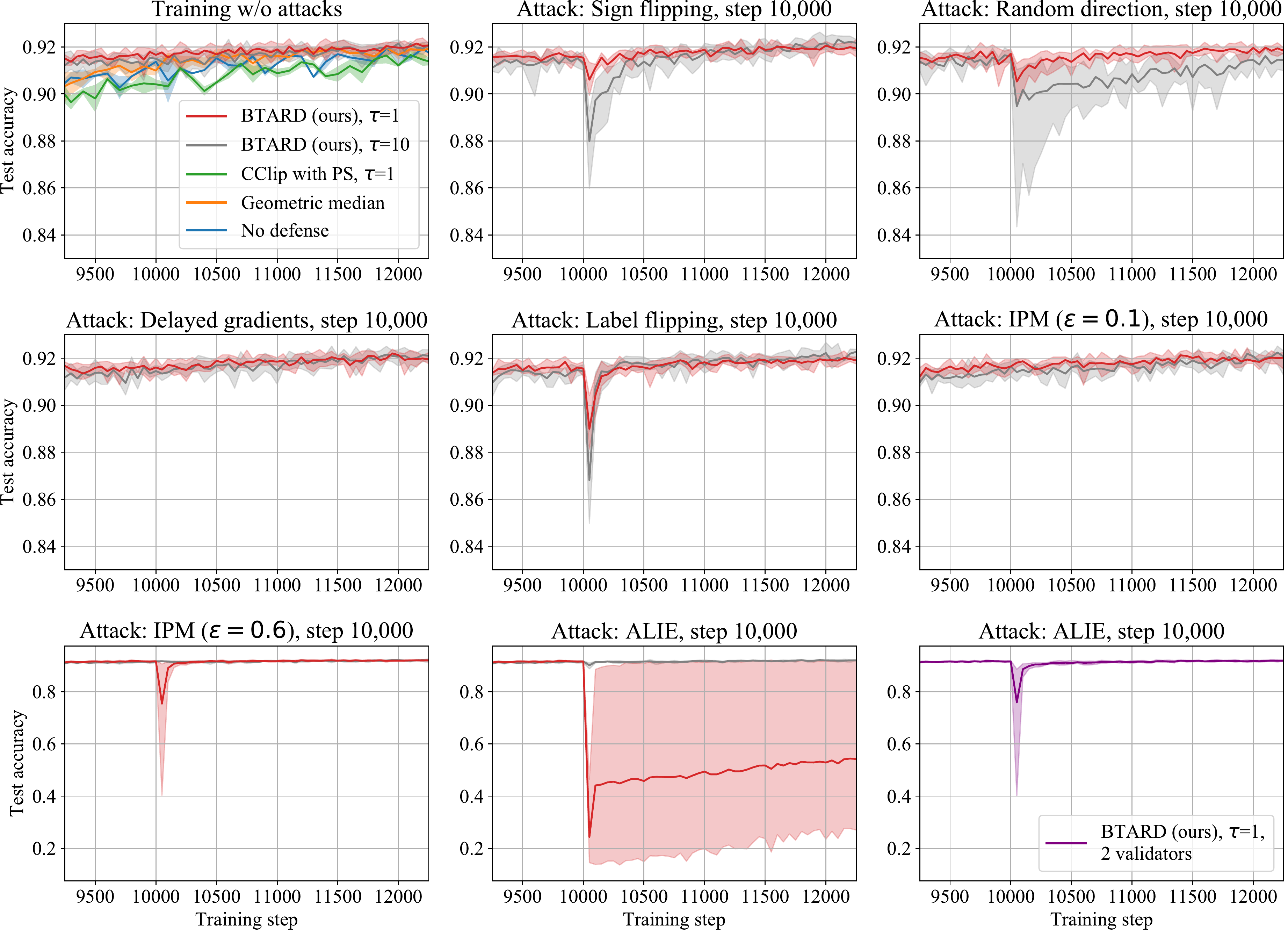}
    \vskip -0.1in
    \caption{The ResNet-18 test accuracy in the case of various attacks performed at each step starting from step $s = 10{,}000$ by 7 Byzantines.}
    \label{fig:resnet_step_10k}
    \vskip -0.15in
\end{figure}

\begin{figure}[tb]
    \vskip 0.05in
    \centering
    \includegraphics[width=0.925\textwidth]{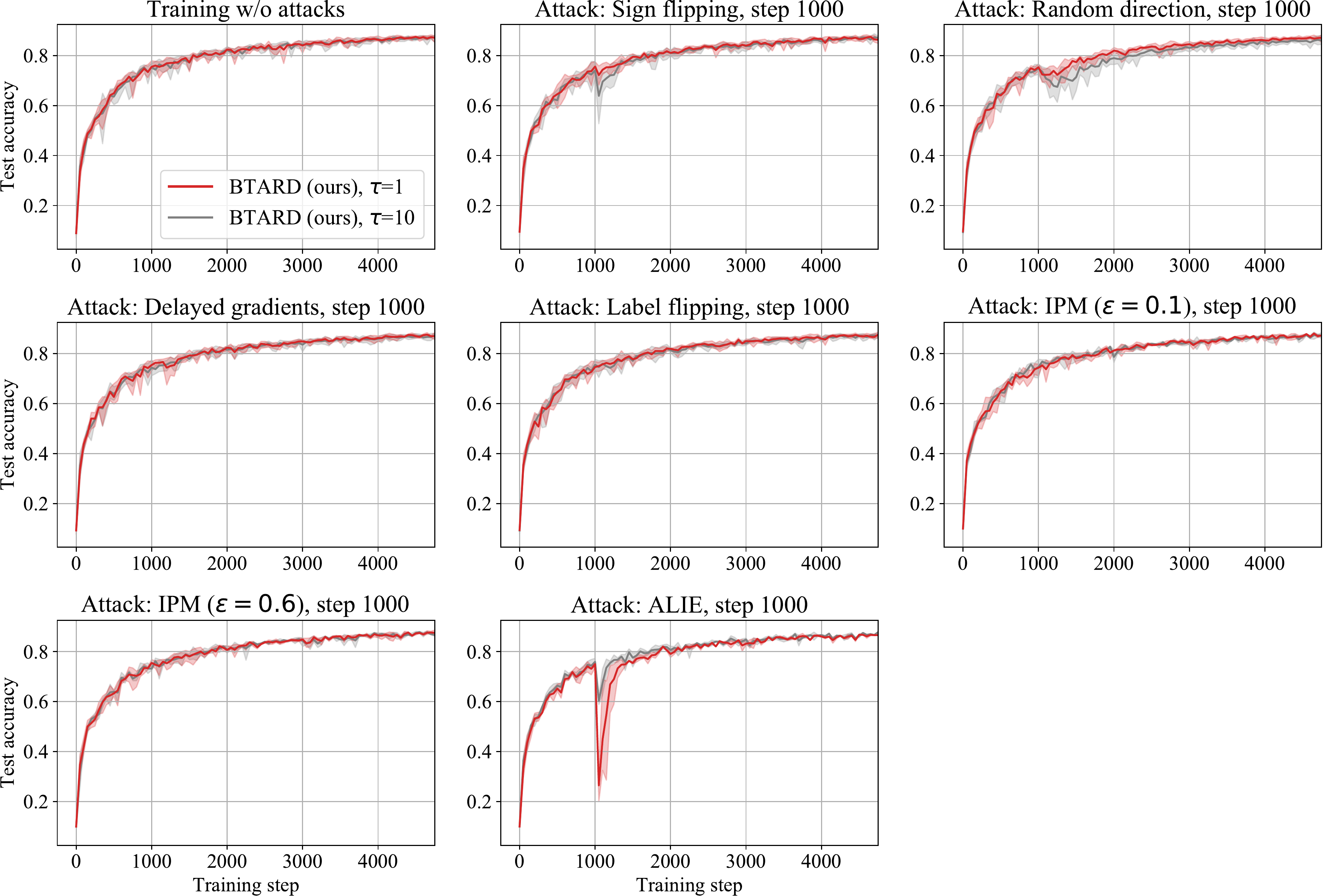}
    \vskip -0.1in
    \caption{The ResNet-18 test accuracy in the case of various attacks performed at every $T$-th step ($T = 10$) starting from step $s = 1000$ by 7~Byzantines.}
    \label{fig:resnet_every10}
    \vskip -0.15in
\end{figure}

\begin{figure}[tb]
    \vskip 0.05in
    \centering
    \includegraphics[width=0.925\textwidth]{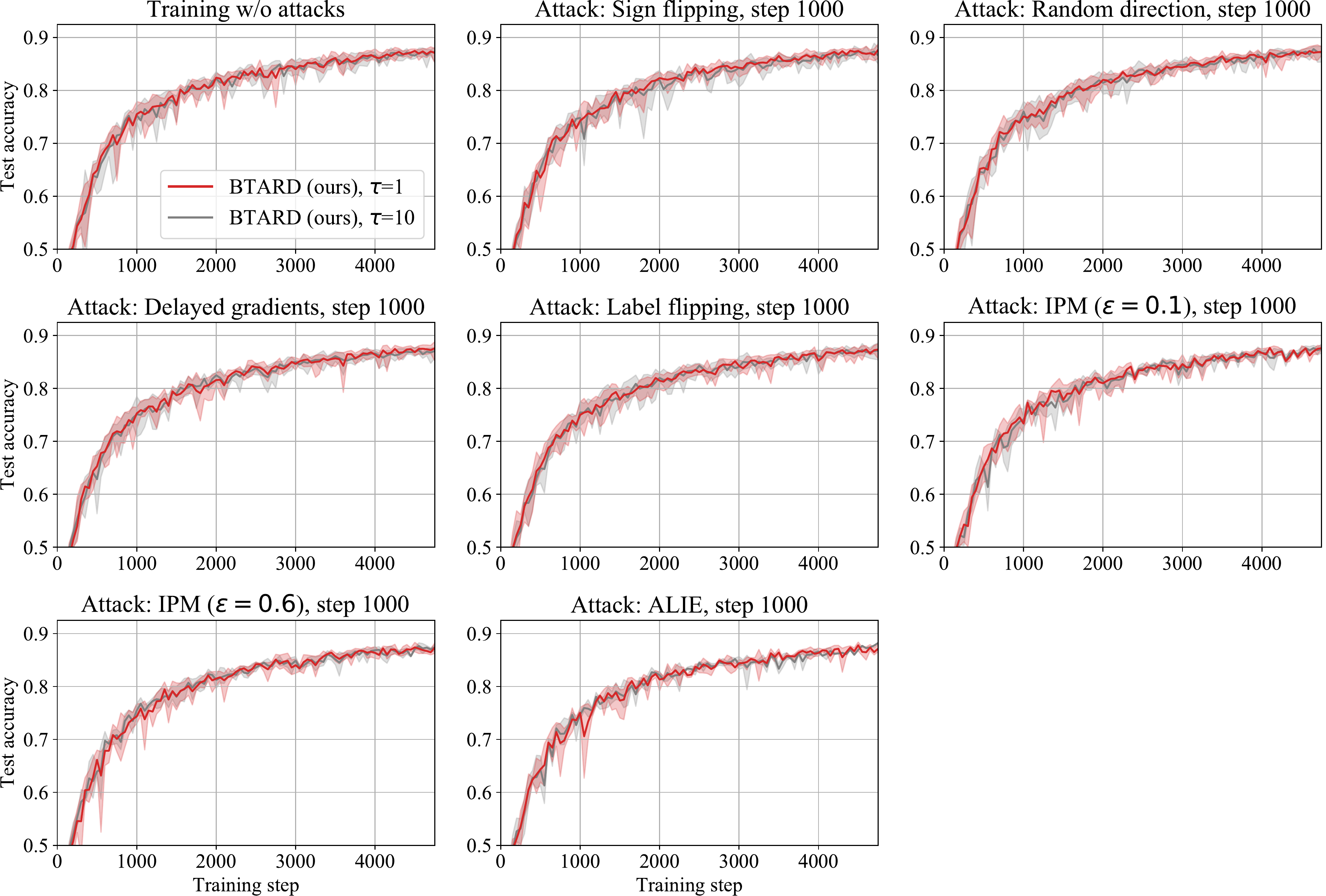}
    \vskip -0.1in
    \caption{The ResNet-18 test accuracy in the case of various attacks performed at each step starting from step $s = 1000$ by 3 Byzantines.}
    \label{fig:resnet_3byzantine}
    \vskip -0.15in
\end{figure}

\begin{figure}[tb]
    \vskip 0.05in
    \centering
    \includegraphics[height=4.5cm]{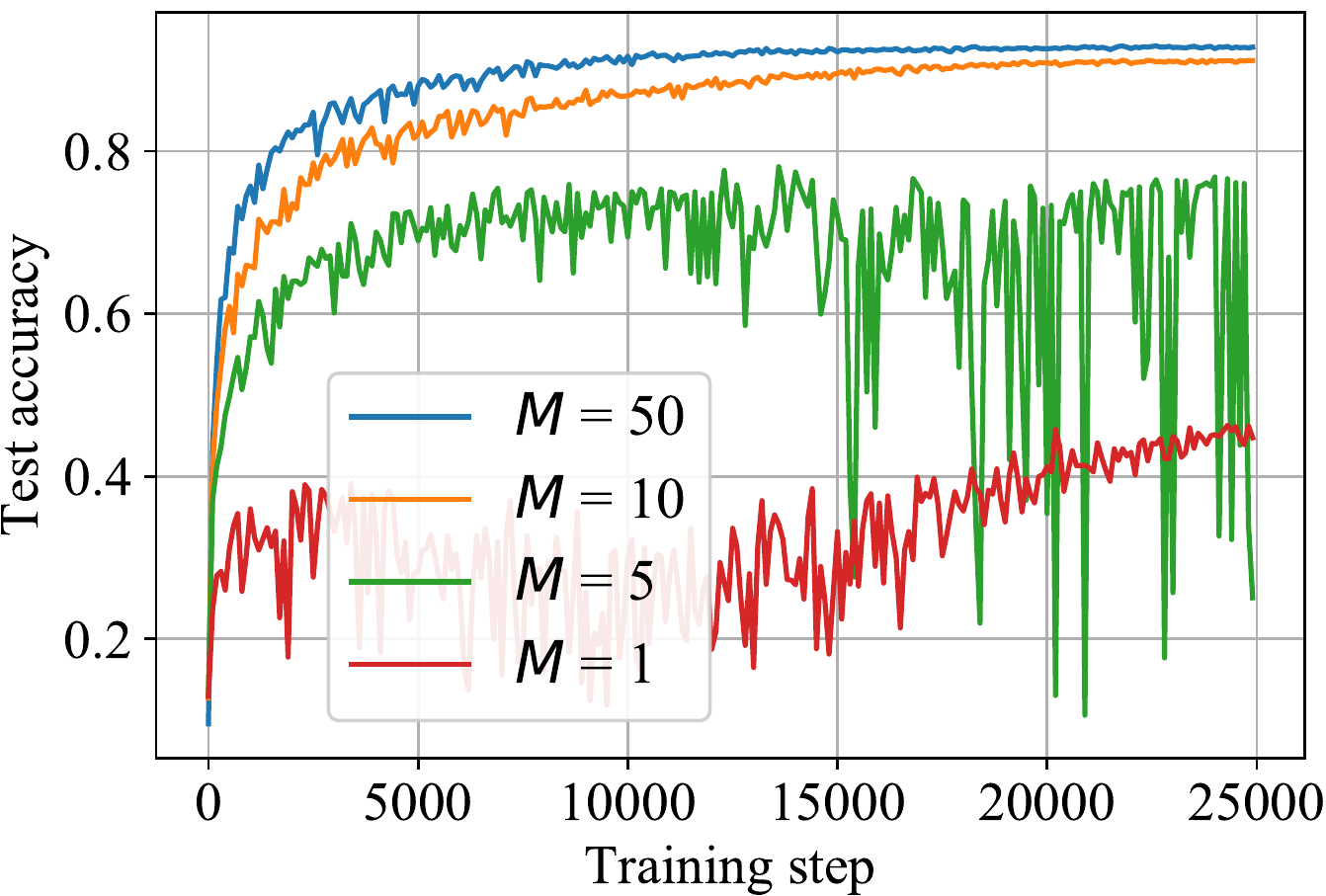}
    \vskip -0.1in
    \caption{Convergence of \algname{BTARD-SGD} with $\tau = 1$ depending on the maximal number of iterations $M$ in the CenteredClip procedure.}
    \label{fig:resnet_cclip_steps}
    \vskip -0.15in
\end{figure}

In this section, we perform several additional experiments with \algname{BTARD-SGD} used to train the ResNet-18 model to solve the CIFAR10 classification task. We start with the same configuration as used in Section~\ref{sect:experiments_resnet} and consider several changes to this setup. In this section, we omit reporting the behavior of the baselines in presence of the attacks since Figure~\ref{fig:resnet} already shows that the baselines cannot withstand most of them.

First, we evaluate our method in case of $s = 10{,}000$: we make Byzantines behave honestly prior to step 10,000, then simultaneously attack on each step until they are banned (i.e., attacks start closer to the convergence stage). The results are shown in Figure~\ref{fig:resnet_step_10k}. The behavior of BTARD turns out to be similar to the case of $s = 1000$, with the same kinds of attacks being most efficient against the stronger and weaker clipping scenarios. BTARD with the stronger clipping and 2 validators is still enough to combat all considered attacks.

Next, we explore a situation where Byzantine peers send incorrect gradients periodically, e.g. once per $T = 10$ iterations. This reduces the attack intensity but allows them to stay undetected for longer. In this setting, we consider 7 Byzantine peers and reuse all parameters from the original setup, except for the new attack period. The attacks are performed at steps $s + k \cdot T, k\in\mathbb{N}$ until the attacker is eventually banned. As expected, this setup increases the duration of each attack by a factor of $T$, but decreases the peak attack influence (see Figure~\ref{fig:resnet_every10}). In this case, even BTARD with one validator is enough to combat all kinds of attacks regardless of the clipping strength.

Next, we  consider a situation where Byzantine peers are less numerous. For this experiment, we use the same configuration as in Section~\ref{sect:experiments_resnet}, but with only 3 Byzantine peers out of 16 (just under 20\%). Figure~\ref{fig:resnet_3byzantine} demonstrates similar behavior to our original setup, but with significantly weaker magnitude across all attacks. Here, BTARD with one validator and any clipping strength is also enough to combat all kinds of attacks.

Finally, we evaluate the convergence and the final test accuracy of the less computationally intensive variants of \algname{BTARD-SGD} that limit the maximal number of iterations in the CenteredClip procedure to $M$, where $M$ varies from 1 to 50. In the setup with $\tau = 1$, we observe that $M = 50$ iterations are always enough for CenteredClip to converge with $\epsilon = 10^{-6}$ in absence of the attacks. Figure~\ref{fig:resnet_cclip_steps} demonstrates that stopping the procedure earlier harms the final test accuracy. This negative effect becomes more significant for the smaller values of $M$.

\subsection{Evaluating computation overhead in terms of wall time}\label{appendix:compute_overhead_eval}

For this analysis, we consider the ALBERT-large training setup from Section~\ref{sect:experiments_albert}. Our training ``swarm'' contains 16 peers with T4 GPUs and 1 GiB/s network bandwidth. On average over 1000 training steps, the full training step for this model takes up 28.56 seconds. Of this, approximately 23.96 seconds were used up for communication and the remaining 4.60 seconds were spent for gradient aggregation \algname{CenteredClip}.

Since MPRNG is running in the background, the only part of \textsc{BTARD} that affects the training time is Algorithm 2 (\textsc{ButterflyClip}). Thus, we measure the time complexity of this algorithm with different numbers of internal iterations. During “normal” epochs where all Byzantines remained passive, the algorithm converged in 2--3 iterations for $\tau = 0.25$ and 5--10 iterations with $\tau = 0.125$. We also noticed that this value has temporarily increased by 2--3 times while Byzantine peers were performing their attack.

{\renewcommand{\arraystretch}{1.2}
\begin{table}[tb]
    \centering
    \caption{Computation overhead of \algname{BTARD} in terms of wall time.}
    \label{tab:compute_overhead_eval}
    \vskip 0.05in
    \begin{tabular}{|c|c|c|}
    \hline
    \textbf{No. of iterations} & \textbf{Wall time (CPU), sec} & \textbf{Wall time (GPU), sec} \\
    \hline
    3  & 0.362 ± 0.003 & 0.040 ± 0.002 \\
    5  & 0.430 ± 0.002 & 0.042 ± 0.002 \\
    10 & 0.601 ± 0.003 & 0.056 ± 0.005 \\
    20 & 0.943 ± 0.002 & 0.085 ± 0.009 \\
    \hline
    \end{tabular}
    \vskip -0.1in
\end{table}

In Table~\ref{tab:compute_overhead_eval}, we report the wall time (the mean and the standard deviation over 10 runs) of our algorithm with a different number of iterations in two hardware setups: running on a 8-core VM with 3.1Ghz Intel Xeon 6148 CPU and on a single 1080 Ti GPU. Even the worst case overhead ($\tau = 0.125$, CPU) is less than the 3\% of the total step time without attacks and less than the 4\% when the attack is active. One important consideration here is that the overhead is constant with respect to the number of peers due to the scaling properties of All-Reduce. Thus, if we train with hundreds of peers, the 0.3--0.6 second overhead can eventually become significant. However, it can be easily offset by moving the \textsc{CenteredClip} execution to GPU, which at this stage is waiting for the \textsc{CenteredClip} results anyway.

\clearpage

\subsection{Experiments at a larger scale (64 machines)}\label{appendix:eval_at_scale}

\begin{figure}[tb]
    \vskip 0.2in
    \centering
    \includegraphics[height=4.5cm]{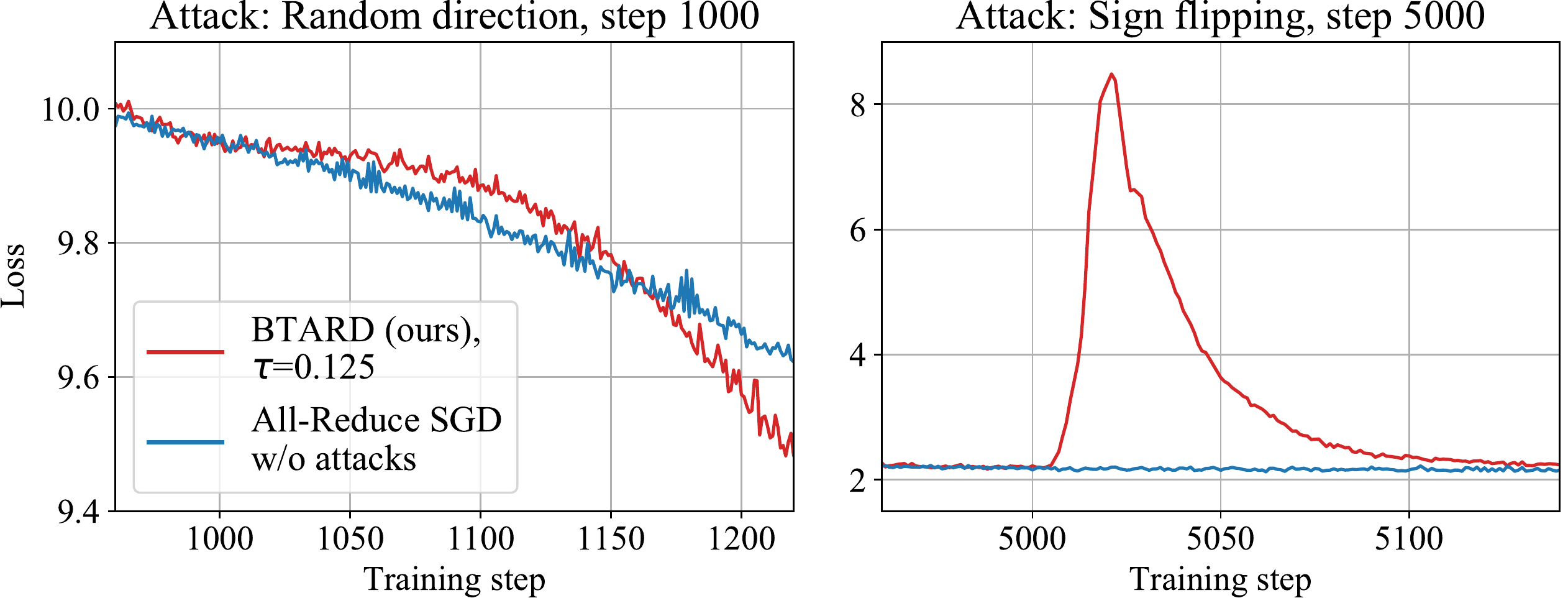}
    \vskip -0.1in
    \caption{The ALBERT-large training objective in the case of \algname{BTARD-Clipped-SGD} (with 31 out of 64 peers being Byzantine) and the standard All-Reduce SGD (without attacks).}
    \label{fig:eval_at_scale}
    \vskip -0.15in
\end{figure}

In this section, we evaluate the most effective attacks against \textsc{BTARD-SGD} in the case of a larger number of peers to ensure that our algorithm scales well. We consider the ALBERT-large training setup from Section~\ref{sect:experiments_albert} and increase the number of machines to 64 (the largest hardware setup available to us), setting up 31 of them to be Byzantine. To balance the increased number of peers, we divide the individual batch size of each peer by 4 and use 4 validators.

Due to the large computation costs, we only evaluate the two most effective strategies for Byzantines based on Figure~\ref{fig:albert}, making only one training run for each of them. We choose the random direction attack starting at step 1000 and the sign flipping attack starting at step 5000.

The results are shown in Figure~\ref{fig:eval_at_scale}. As in our previous experiments, the Byzantine peers manage to temporarily offset the training loss. As in the case with 16 peers, the sign flipping attack at step 5000 obtains the "peak" distortion approximately 20 steps into the attack, and the random direction attack at step 1000 has a longer but less intensive effect. However, \algname{BTARD-SGD} is able to quickly detect and ban the attackers, banning all 31 Byzantines in 100--150 steps and catching up with the original learning curve after approximately 150 steps (or even temporarily surpassing it). We conclude that \algname{BTARD-SGD} maintains its efficiency even at this scale.

\end{document}